\documentclass[11pt]{article}
\pdfoutput=1

\usepackage{jmlr2e}

\usepackage{times}
\usepackage{graphicx}
\usepackage{hyperref}
\usepackage{bold-extra}
\usepackage{bbm}
\usepackage{amsmath}
\usepackage{tabu}
\usepackage{colortbl}
\usepackage{amsfonts}

\usepackage{amsthm}
\usepackage{mathtools}
\usepackage{enumitem}
\setlist{parsep=0pt,listparindent=\parindent}
\usepackage{bm}
\usepackage{dsfont}
\usepackage{verbatim}
\usepackage{thmtools}
\usepackage{thm-restate}
\usepackage{stfloats}
\usepackage{algorithm}
\usepackage{tabu}
\usepackage{tabularx}
\usepackage{xfrac}
\usepackage[noend]{algpseudocode}
\algrenewcommand\textproc{}
\usepackage[font=small,labelfont=bf]{caption}
\usepackage{appendix}

\newif\ifarxiv
\arxivtrue

\usepackage{float}
\newfloat{algorithm}{t}{lop}
\floatname{algorithm}{Algorithm} 

\usepackage{color}

\definecolor{gray230}{RGB}{200,200,200}

\definecolor{amaranth}{RGB}{0,0,0}
\newcommand{\algcomment}[1]{\vspace{0.0em} \State \emph{\# #1}:}
\newcommand{\alglinebreak}{\\ \vspace{-0.4em} {\color{gray230}\hrulefill} \\ \vspace{-1.3em}}

\def\abovestrut#1{\rule[0in]{0in}{#1}\ignorespaces}
\def\belowstrut#1{\rule[-#1]{0in}{#1}\ignorespaces}
\def\abovespace{\abovestrut{0.20in}}
\def\belowspace{\belowstrut{0.10in}}
\def\liblinear{\textsc{LibLinear}}
\def\lasso{lasso}

\def\algname{BlitzWS}
\def\algnamel{BlitzMN}
\def\screenname{BlitzScreen}

\providecommand{\e}[1]{\ensuremath{\times 10^{#1}}}

\newcommand{\listm}[1]{\left( {#1} \right)_{i=1}^m}

\newcommand{\spar}[2]{\small{$s^\star \approx {#1};\ s^\star_{\mathrm{W}} \approx {#2}$}}
\newcommand{\svmspar}[1]{\small{$s^\star \approx {#1}$}}

\newcommand{\dcsvm}[1]{\small{$C = {#1} C_{\mathrm{cv}}$}}
\newcommand{\Csvm}{C_{\mathrm{cv}}}
\newcommand{\lam}[1]{\small{$\lambda = {#1} \lammax$}}
\newcommand{\m}[1]{\mathbf{#1}}
\newcommand{\X}{\mathcal{X}}
\newcommand{\Xis}{\mathcal{X}_i^{(\star)}}
\newcommand{\Xik}{\mathcal{X}_i^{(k)}}
\newcommand{\tXik}{\tilde{\mathcal{X}}_i^{(k)}}

\newcommand{\x}{\m{x}}
\renewcommand{\u}{\m{u}}
\newcommand{\z}{\m{z}}
\newcommand{\xstar}{\x^\star}

\newcommand{\omegab}{\bm{\omega}}

\newcommand{\y}{\m{y}}
\newcommand{\g}{\m{g}}

\newcommand{\A}{\m{A}}
\renewcommand{\b}{\m{b}}
\renewcommand{\c}{\m{c}}
\def\dminc{d_\mathrm{min}^{\mathrm{cap}}}

\def\nnz{\mathrm{NNZ}}
\def\Cprogress{C^{\mathrm{progress}}_t}
\def\Csetup{C^{\mathrm{setup}}_t}
\def\Csolve{C^{\mathrm{solve}}_t}
\def\hCprogress{\hat{C}^{\mathrm{progress}}}
\def\hCsetup{\hat{C}^{\mathrm{setup}}}
\def\hCsolve{\hat{C}^{\mathrm{solve}}}
\def\Tsetup{T^{\mathrm{setup}}}
\def\Tsolve{T^{\mathrm{solve}}}
\def\dmaxc{d_\mathrm{max}^{\mathrm{cap}}}
\def\rcap{r^{\mathrm{cap}}}

\def\philbtm{\phi_{i,t-1}^{\mathrm{LB}}}
\def\ftlb{f_t^{\mathrm{LB}}}
\def\f{f}
\def\fs{f_\mathrm{S}}
\def\phis{\phi_{i, \mathrm{S}}}
\def\xs{\x_\mathrm{S}}
\def\ftmlb{f_{t-1}^{\mathrm{LB}}}
\def\gilb{\g_i^{\mathrm{LB}}}
\def\tmaxc{\rcap}

\renewcommand{\a}{\m{a}}

\newcommand{\reals}{\mathbb{R}}

\newcommand{\norm}[1]{\left\|#1\right\|}
\newcommand{\norms}[1]{\|#1\|}
\newcommand{\ip}[1]{\left<#1\right>}
\newcommand{\ips}[1]{\langle #1 \rangle}
\newcommand{\abs}[1]{\left|#1\right|}

\renewcommand{\min}[1]{\ \underset{#1}{\mathrm{min}}\ }
\renewcommand{\inf}[1]{\ \underset{#1}{\mathrm{inf}}\ }
\renewcommand{\max}[1]{\ \underset{#1}{\mathrm{max}}\ }
\renewcommand{\sup}[1]{\ \underset{#1}{\mathrm{sup}}\ }
\newcommand{\minimize}[1]{\underset{#1}{\mathrm{minimize}} \ }

\newcommand{\argmin}[1]{\underset{#1}{\mathrm{argmin}} \ }
\newcommand{\argmax}[1]{\underset{#1}{\mathrm{argmax}} \ }
\renewcommand{\lim}[1]{\underset{#1}{\mathrm{lim}} \ }

\def\oh{\tfrac{1}2}
\def\toh{\sfrac{1}2}

\def\D{\reals^m}
\def\W{\mathcal{W}}
\renewcommand{\int}[1]{{\bf int} \, #1}
\newcommand{\ball}[1]{\textrm{ball}\left(#1\right)}
\newcommand{\conv}[1]{\mathrm{conv}\left(#1\right)}
\def\oh{\tfrac{1}2}
\def\Ai{{}}

\def\P{\mathcal{P}}
\def\B{\mathcal{B}}
\def\yt{{\y_t}}
\def\ytp{{\y_t'}}
\def\ytm{{\y_{t-1}}}
\def\xtm{{\x_{t-1}}}
\def\dtm{d_{t-1}}

\def\zt{{\z_{t}}}
\def\xt{{\x_{t}}}

\newcommand{\fullappref}[1]{Appendix~\ref{#1}}
\def\fld{\psi_{\mathrm{MN}}}
\def\flod{f_{\mathrm{L1D}}}
\def\flo{g_{\mathrm{L1}}}
\def\il{i_{\mathrm{limit}}}
\def\betat{{\beta_t}}
\def\thetat{\theta_t}

\newcommand{\lammax}{\lambda_{\mathrm{max}}}

\newcommand{\eqnref}[1]{(\ref{#1})}
\newcommand{\thmref}[1]{Theorem~\ref{#1}}
\newcommand{\corref}[1]{Corollary~\ref{#1}}
\newcommand{\lemref}[1]{Lemma~\ref{#1}}
\newcommand{\secref}[1]{\S\ref{#1}}
\newcommand{\appref}[1]{supplementary material}

\newcommand{\figref}[1]{Figure~\ref{#1}}

\renewcommand{\algref}[1]{Algorithm~\ref{#1}}
\newcommand{\propref}[1]{Proposition~\ref{#1}}
\newcommand{\tblref}[1]{Table~\ref{#1}}
\newcommand{\prbref}[1]{(\ref{#1})}

\theoremstyle{plain}
\newtheorem{thm}{Theorem}[section]

\newtheorem{cor}[thm]{Corollary}

\newtheorem{prop}[thm]{Proposition}

\def\D{\reals^n} 
\def\T{\mathcal{S}}

\def\Tcap{\mathcal{S}_\xi^{\mathrm{cap}}}
\def\Wcap{\mathcal{W}_\xi^{\mathrm{cap}}}
\def\Wxi{\mathcal{W}_\xi}
\def\ccap{\c^{\mathrm{cap}}}
\def\Bcap{\B^{\mathrm{cap}}}
\def\rcap{r^{\mathrm{cap}}}
\def\W{\mathcal{W}}
\renewcommand{\int}[1]{{\mathrm{int}}(#1)}
\newcommand{\cl}[1]{{\mathrm{cl}}(#1)}

\newcommand{\bd}[1]{{\mathrm{bd}}(#1)}

\def\sump{\sum_{i=1}^m} 

\def\G{\mathcal{G}}
\def\ggl{g_{\mathrm{GL}}}
\def\fgl{f_{\mathrm{GL}}}

\jmlrheading{vol}{yr}{pages}{m/yr}{m/yr}{Johnson and Guestrin}
\ShortHeadings{\algname{}: A Fast, Principled Working Set Algorithm}{Johnson and Guestrin}
\firstpageno{1}

\begin{document}

\title{A Fast, Principled Working Set Algorithm for Exploiting \\ Piecewise Linear Structure in Convex Problems}

\author{\name Tyler B. Johnson \email tbjohns@washington.edu \\
       \addr Department of Electrical Engineering 
       \\ University of Washington\\
       Seattle, WA 98195, USA
       \AND
       \name Carlos Guestrin \email guestrin@cs.washington.edu \\
       \addr Department of Computer Science \& Engineering \\
       University of Washington\\
       Seattle, WA 98195, USA}

\editor{editor}


\maketitle

\begin{abstract}
  By reducing optimization to a sequence of smaller subproblems, 
working set algorithms achieve fast convergence times for many machine learning problems.  
Despite such performance,
working set implementations often resort to heuristics to determine subproblem size, makeup, and stopping criteria.  
We propose \algname{}, a working set algorithm with useful theoretical guarantees.
Our theory 
relates subproblem size and stopping criteria to the amount of progress during each iteration.
This result motivates strategies for optimizing algorithmic parameters and discarding irrelevant components as \algname{} progresses toward a solution.
\algname{} applies to many convex problems, including training $\ell_1$-regularized models and support vector machines.
We showcase this versatility with empirical comparisons, which demonstrate \algname{} is indeed a fast algorithm.

\end{abstract}


\section{Introduction}

Many optimization problems in machine learning have useful structure at their solutions.
For sparse regression, the optimal model makes predictions using a fraction of available features.
For support vector machines, easy-to-classify examples have no influence on the optimal model.
In this work, we exploit such structure to train these models efficiently.

Working set algorithms exploit structure by reducing optimization to a sequence of simpler subproblems.  Each subproblem considers only a priority subset of the problem's components---the features likely to have nonzero weight in sparse regression, for example, or training examples near the margin in SVMs.
Likely the most prominent working set algorithms for machine learning are those of the
\liblinear{} library \citep{Fan:2008}, an efficient software package for training linear models.
By using working set and related ``shrinking'' \citep{Joachims:1999} heuristics, \liblinear{} converges very quickly.
Other successful applications of working sets include algorithms proposed by \citet{Osuna:1997}, \citet{Zanghirati:2003}, \citet{Tsochantaridis:2005}, \citet{Kim:2008}, \citet{Roth:2008}, \citet{Obozinski:2009} and \citet{Friedman:2010}.

Despite the usefulness of working set algorithms, there is limited theoretical understanding of these methods.
For \liblinear{}, except for guaranteed convergence to a solution, there are no guarantees with regard to working sets and shrinking.
As a result, critical aspects of working set algorithms typically rely on heuristics rather than principled understanding.

We propose \algname{}, a working set algorithm accompanied by useful theoretical analysis.
Our theory explains how to prioritize components of the problem in order to guarantee a specified amount of progress during each iteration.
To our knowledge, \algname{} is the first working set algorithm with this type of guarantee.
This result motivates a theoretically justified way to select each subproblem, making \algname{}'s choice of subproblem size, components, and stopping criteria more principled and robust than those of prior approaches.


\algname{} solves instances of a novel problem formulation, which formalizes our notion of ``exploiting structure'' in problems such as sparse regression.
Specifically, we define the objective function as a sum of many piecewise terms.
Each piecewise function is comprised of simpler subfunctions, some of which we assume to be linear.
Exploiting structure amounts to selectively replacing piecewise terms in the objective with linear subfunctions.  This results in a modified objective that can be much simpler to minimize.
By solving a sequence of such subproblems, \algname{} rapidly converges to the original problem's solution.

In addition to \algname{}, we propose a closely related safe screening test called \screenname{}. 
First proposed by \cite{Ghaoui:2010}, safe screening identifies problem components that are guaranteed to be irrelevant to the solution.  
Compared to prior screening tests, \screenname{} (i)~applies to a larger class of problems, and (ii)~simplifies the objective function by a greater amount.

We include empirical evaluations to showcase the usefulness of \algname{} and \screenname{}.   
We find \algname{} significantly outperforms \liblinear{} in many cases, especially for sparse logistic regression problems. 
Perhaps surprisingly, although our screening test improves on many prior tests,
we find that screening 
often
has \emph{negligible} effect on overall convergence times. In contrast, \algname{} improves convergence times significantly in nearly all cases.

This work builds upon two previous conference papers \citep{Johnson:2015,Johnson:2016}. 
New contributions include refinements to the proposed algorithm, improved theoretical results, and additional empirical results.
An open-source implementation of \algname{} is available at the web address \url{http://github.com/tbjohns/blitzml}.

We organize the remainder of this paper as follows.
In \secref{sec:blitz_lasso}, we introduce \algname{} for a simple constrained problem,  emphasizing \algname{}'s main concepts.
In \secref{sec:formulation}, we introduce a piecewise problem formulation, which encompasses a larger set of problems than we consider in \secref{sec:blitz_lasso}. In \secref{sec:blitz}, we define \algname{} for the general piecewise problem.
This section contains more detail compared to \secref{sec:blitz_lasso}, including analysis of approximate subproblem solutions and a method for selecting algorithmic parameters.
In \secref{sec:screening}, we introduce \screenname{} and explain its relation to \algname{}.
In \secref{sec:empirical}, we demonstrate the usefulness of \algname{} and \screenname{} in practice.  We discuss conclusions in \secref{sec:discussion}.

\section{\algname{} for a simple constrained problem} \label{sec:blitz_lasso}

In this section, we introduce \algname{} for computing the minimum norm vector in a polytope.
Given vectors $\a_i \in \reals^n$ and scalars $b_i \in \reals$ for $i = 1, 2, \ldots, m$, we solve 
\begin{equation}
\begin{array}{cll}
\minimize{\x \in \D} & \fld(\x) := \oh \norm{\x}^2 & \\
\mathrm{s.t.} & \ip{\a_i, \x} \leq b_i & i = 1, \ldots, m \, .
\end{array}
\label{prb:lasso_dual}
\tag{PMN}
\end{equation}
We refer to the special case of \algname{} that solves this min-norm problem as ``\algnamel{}.''
For now, we consider only this simple problem to emphasize the algorithm's concepts rather than its capability of solving a variety of problems.

\prbref{prb:lasso_dual} has important applications to machine learning.
Most notably, $\ell_1$-regularized least squares problems---the ``\lasso{}'' \citep{Tibshirani:1996}---can be transformed into an instance of \prbref{prb:lasso_dual} using duality.
We discuss using \algname{} to  solve the \lasso{} more in \secref{sec:sparsity_inducing}.

\subsection{Overview of \algnamel{}} \label{sec:overview}

During each iteration $t$, \algnamel{} selects a working set of constraints, $\W_t$.  \algnamel{} 
then computes the minimizer of $\fld$ subject only to constraints in $\W_t$, storing the result as $\xt$:
\[
\x_t \gets \mathrm{argmin} \left\{ \tfrac{1}2 \norm{\x}^2 \, \big| \, \ip{\a_i, \x} \leq b_i \ \mathrm{for\ all}\ i \in \W_t \right\} \, .
\]
We refer to the task of computing $\xt$ as ``subproblem $t$.''

Typical working set algorithms select each $\W_t$ using heuristics.
\algnamel{} improves upon this with two main novelties.
First, in addition to $\xt$, \algnamel{} introduces a second iterate, $\y_t$.
This iterate is feasible (satisfies all constraints in \prbref{prb:lasso_dual}) for all $t$.

The iterate $\yt$ is necessary for \algnamel{}'s second novelty, which is
the principled choice of each working set.
\algnamel{} selects $\W_t$ in a way that guarantees quantified progress during iteration $t$.
This amount of progress is determined by a progress parameter, $\xi_t \in (0, 1]$.
For now we assume $\xi_t$ is given, but in \secref{sec:alg_params}, we discuss a way to automatically select $\xi_t$.

If $\xi_t = 1$, then \algnamel{} is guaranteed to return \prbref{prb:lasso_dual}'s solution upon completion of iteration $t$. As $\xi_t$ decreases toward zero, \algnamel{} guarantees less progress (which we quantify more precisely in \secref{sec:teardrop}).
At a high level, \algnamel{} combines two ideas to ensure this progress:

\begin{itemize}
\item \emph{Including in the working set constraints that are active at the previous subproblem solution:} \algnamel{} includes in $\W_t$ all constraints for which $\ip{\a_i, \xtm} = b_i$.  
This ensures that $\fld(\xt) \geq \fld(\xtm)$. 
\item \emph{Enforcing an equivalence region:} For subproblem $t$, \algnamel{} defines an ``equivalence region,'' $\T_\xi$, which is a subset of $\reals^n$.  
\algnamel{} selects $\W_t$ in a way that ensures subproblem $t$ and \prbref{prb:lasso_dual} are identical within $\T_\xi$ (i.e., within $\T_\xi$, the feasible region is preserved). 

Establishing this equivalence region has two major implications.
First, if $\xt \in \T_\xi$, then $\xt$ solves not only subproblem $t$ but also \prbref{prb:lasso_dual}.
This is because subgradient values are preserved within the equivalence region.
Second, if $\xt$ does not equal the solution, then it must be the case that $\xt \notin \T_\xi$. 
We design $\T_\xi$ to ensure ``$\xi_t$ progress'' in this case.
\end{itemize}


\subsection{Making working sets tractable with suboptimality gaps} \label{sec:tractable}

We measure \algnamel{}'s progress during each iteration in terms of a \emph{suboptimality gap}.
Since $\xt$ minimizes $\fld$ subject to a subset of constraints, 
 it follows that
$\fld(\xt) \leq \fld(\xstar)$, where $\xstar$ solves \prbref{prb:lasso_dual}.
Thus, given the feasible point $\yt$, we can define the suboptimality gap
\[
\Delta_t = \fld(\yt) - \fld(\xt) \geq \fld(\yt) - \fld(\xstar) \, .
\]

Later, we analyze the improvement in $\Delta_t$ between iterations $t-1$ and $t$.
Maximizing this improvement leads to a principled method for selecting each working set.

\subsection{Converging from two directions: iterate $\yt$ and line search \label{sec:backtracking}}


\algnamel{} initializes $\y_0$ as a feasible point.\footnote{%
While computing a feasible $\y_0$ could be difficult in general, this is not an issue for the applications we consider.  When solving the \lasso{}'s dual (\secref{sec:sparsity_inducing}), for example, \algname{} defines $\y_0 = \m{0}$, which is feasible in this case.} After the algorithm computes $\xt$ during iteration $t$, \algnamel{} performs a line search update to $\yt$.  Specifically, 
$\y_t$ is is the point on the segment $[\ytm, \xt]$ that is closest to $\xt$ while remaining feasible.
Put differently, assuming that $\x_t$ violates at least one constraint, \algnamel{} updates $\yt$ so that (i)~$\yt$ lies on the segment $[\ytm, \xt]$, (ii) $\yt$ satisfies all $m$ constraints, and (iii) unless $\yt = \xt$, there exists a ``limiting constraint'' $\il$ for which $\ip{\a_{\il}, \x_t} > b_{\il}$ and $\ip{\a_{\il} \yt} = b_{\il}$.


To perform this line search update, \algnamel{} computes
\begin{equation}
\label{eqn:backtracking_update}
\alpha_t = \min{i \, : \, \ip{\a_i, \xt} > b_i} \frac{b_i - \ip{\a_i, \ytm}}{\ip{\a_i, \xt} - \ip{\a_i, \ytm}}
\end{equation}
and subsequently defines
$
\yt = \alpha_t \xt + (1 - \alpha_t) \ytm  .
$
In the special case that $\ip{\a_i, \xt} \leq b_i$ for all $i$, we define $\alpha_t = 1$ (and hence $\yt = \xt$). \algnamel{} has converged in this case, since $\Delta_t = 0$.


Because $\xt$ minimizes $\fld$ subject to a subset of constraints, $\fld(\xt) \leq \fld(\ytm)$.
By convexity of $\fld$, this implies $\fld(\yt) \leq \fld(\ytm)$.
Recall also from \secref{sec:overview} that $\fld(\xt)$ is nondecreasing with $t$.  Therefore, $\Delta_t$ is nonincreasing with $t$.

We have not yet quantified \emph{how much} $\Delta_t$ decreases between iterations.
We next derive a rule for selecting $\W_t$ that guarantees this suboptimality gap decreases by a specified amount.

\subsection{Quantifying suboptimality gap progress during iteration $t$ \label{sec:teardrop}}

In \secref{sec:overview}, we established that \algnamel{} includes in $\W_t$ all constraints that are active at $\xtm$. 
We now add additional constraints to the working set in order to guarantee a specified amount of progress toward convergence.
In particular, given a progress coefficient $\xi_t \in (0, 1]$, we design $\W_t$ such that
\begin{equation} \label{eqn:progress}
\Delta_t \leq (1 - \xi_t) \Delta_{t-1} \, .
\end{equation}


Applying properties of convexity and the definition of $\yt$, we derive the following bound:
\begin{restatable}{lem}{blitzlassolemma}
\label{lem:blitz_lasso_lemma}
Assume $\alpha_t > 0$, and define $\beta_t = \alpha_t (1 + \alpha_t)^{-1}$.
Assume that 
$\W_t$ includes all constraints that are active at $\xtm$.
Then we have
\begin{equation} \label{eqn:blitzlasso_lemma}
\Delta_t \leq \tfrac{1 - 2 \beta_t}{1 - \beta_t} \left[ \Delta_{t-1} - \tfrac{1 - \beta_t}{\beta_t^2} \oh \norm{\yt - \beta_t \xtm - (1 - \beta_t) \ytm}^2 - \beta_t \oh \norm{\xtm - \ytm}^2 \right] \, .
\end{equation}
\end{restatable} 
We prove this result in \fullappref{app:proof_blitz_lasso_lemma}.  
Since the special case that $\alpha_t = 0$ is unnecessary for understanding the algorithm's concepts,
we ignore this case except in later proofs.

The working set affects only two variables in \eqnref{eqn:blitzlasso_lemma}: $\beta_t$ and $\yt$.
The other variables---$\xtm$, $\ytm$, and $\Delta_{t-1}$---are given when selecting $\W_t$.
By understanding how $\W_t$ affects $\beta_t$ and $\yt$, we can choose $\W_t$ so that the right side of \eqnref{eqn:blitzlasso_lemma} is upper bounded by $(1 - \xi_t) \Delta_{t-1}$.

Assume for a moment that $\beta_t$ is given when \algnamel{} defines $\W_t$.
In this scenario, it is straightforward to select $\W_t$ in a way that guarantees \eqnref{eqn:progress}.
For the special case that $\beta_t = \toh$, \eqnref{eqn:blitzlasso_lemma} simplifies to $\Delta_t \leq 0$---\eqnref{eqn:progress} holds, regardless of $\W_t$.
In the case that $\beta_t < \toh$, we have $\alpha_t < 1$.
Applying the definition of $\alpha_t$ in \eqnref{eqn:backtracking_update}, there exists a limiting constraint 
$\il$
such that
\begin{equation} \label{eqn:y_boundary}
\ip{\a_{\il}, \xt} > b_{\il} \quad \quad \text{and} \quad \quad
\ip{\a_{\il}, \yt} = b_{\il} \, .
\end{equation}

Since $\xt$ violates constraint $\il$, we have $\il \notin \W_t$.
To apply this fact, note $\yt$ appears in \lemref{lem:blitz_lasso_lemma} through the quantity $\norm{\yt - \beta_t \xtm - (1 - \beta_t) \ytm}$.
\algnamel{} chooses $\W_t$ to ensure this norm equals a threshold $\tau_{\xi}(\betat)$ at minimum.
To achieve this, we define the ``equivalence ball''
\begin{equation} \label{eqn:equiball}
\B_{\xi}(\betat) = \{ \x \mid \norm{\x - \beta_t \xtm - (1 - \betat) \ytm} < \tau_{\xi}(\betat) \} \, .
\end{equation}

\begin{figure}
\begin{center}
\includegraphics[width=2.7in]{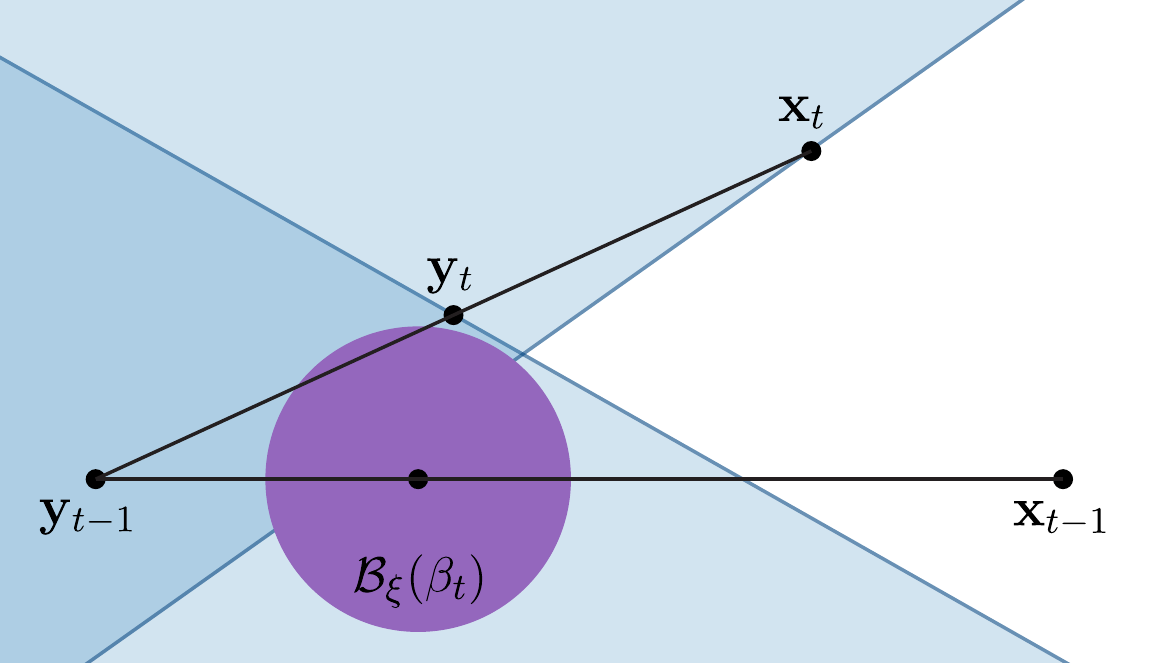} 
\end{center}
\begin{small}
\caption{\textbf{\algnamel{} geometry assuming knowledge of $\betat$.} 
Shaded areas represent the feasible regions of two linear constraints.
Assume $\betat = \alpha_t (1 + \alpha_t)^{-1}$ is known when choosing $\W_t$.
Shown in purple, $\B_\xi(\betat)$ is a ball with center $\betat \x_{t-1} + (1 - \betat) \y_{t-1}$ and radius $\tau_\xi(\betat)$.
The bound \eqnref{eqn:progress} depends on $\y_t$'s distance from this ball's center.
\algnamel{} ensures this distance is large by choosing $\W_t$ so that if $\yt \ne \xstar$, then $\y_t \notin \B_\xi(\betat)$.
In particular, \algnamel{} selects $\W_t$ in a way that preverses \prbref{prb:lasso_dual}'s feasible region within $\B_\xi(\betat)$. 
} \label{fig:beta_geometry}
\end{small}
\end{figure}

\algnamel{} includes $i$ in $\W_t$ if 
there exists an $\x \in \B_\xi(\betat)$ such that
$\ip{\a_i, \x} \geq b_i$.
This preserves \prbref{prb:lasso_dual}'s feasible region within $\B_\xi(\betat)$.
Since $\il \notin \W_t$, this implies that 
no point on the boundary of constraint $\il$---$\yt$ included, due to \eqnref{eqn:y_boundary}---lies within $\B_\xi(\betat)$.
By our definition of $\B_\xi(\betat)$ in \eqnref{eqn:equiball},  
this guarantees that $\norm{\yt - \beta_t \xtm - (1 - \beta_t) \ytm} \geq \tau_\xi(\betat)$.  
We illustrate this concept in \figref{fig:beta_geometry}.


Having linked $\tau_\xi(\betat)$ to \lemref{lem:blitz_lasso_lemma}, we can define $\tau_\xi(\betat)$ to produce our desired bound, \eqnref{eqn:progress}:
\begin{equation} \label{eqn:tau}
\tau_\xi(\betat) = \betat \sqrt{2 \Delta_{t-1}} \left[ 1 + \tfrac{\betat}{1 - \betat} \left(1 - \tfrac{\norm{\xtm - \ytm}^2}{2 \Delta_{t-1}} \right) - \tfrac{1 - \xi_t}{1 - 2 \betat} \right]_+^{1/2} .
\end{equation}
This leads to the following result:
\begin{restatable}{lem}{blitzlassoprogressbeta} \label{lem:gap_progress_betat}
Assume $\betat$ is known when selecting $\W_t$, and assume $\betat > 0$.
For all $i \in [m]$, let $\W_t$ include constraint $i$ if either $\{ \x \mid \ip{\a_i, \x} \geq b_i \} \cap \B_\xi(\betat) \ne \emptyset$ or $\ip{\a_i, \xtm} = b_i$.
Then
\[ \Delta_{t} \leq (1 - \xi_t) \Delta_{t-1} \, . \]
\end{restatable}

\begin{figure}
\begin{center}
\begin{tabular}{c@{}c}
\raisebox{2.3cm}{ \includegraphics[width=0.9in]{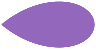} }
&
\includegraphics[width=3.6in]{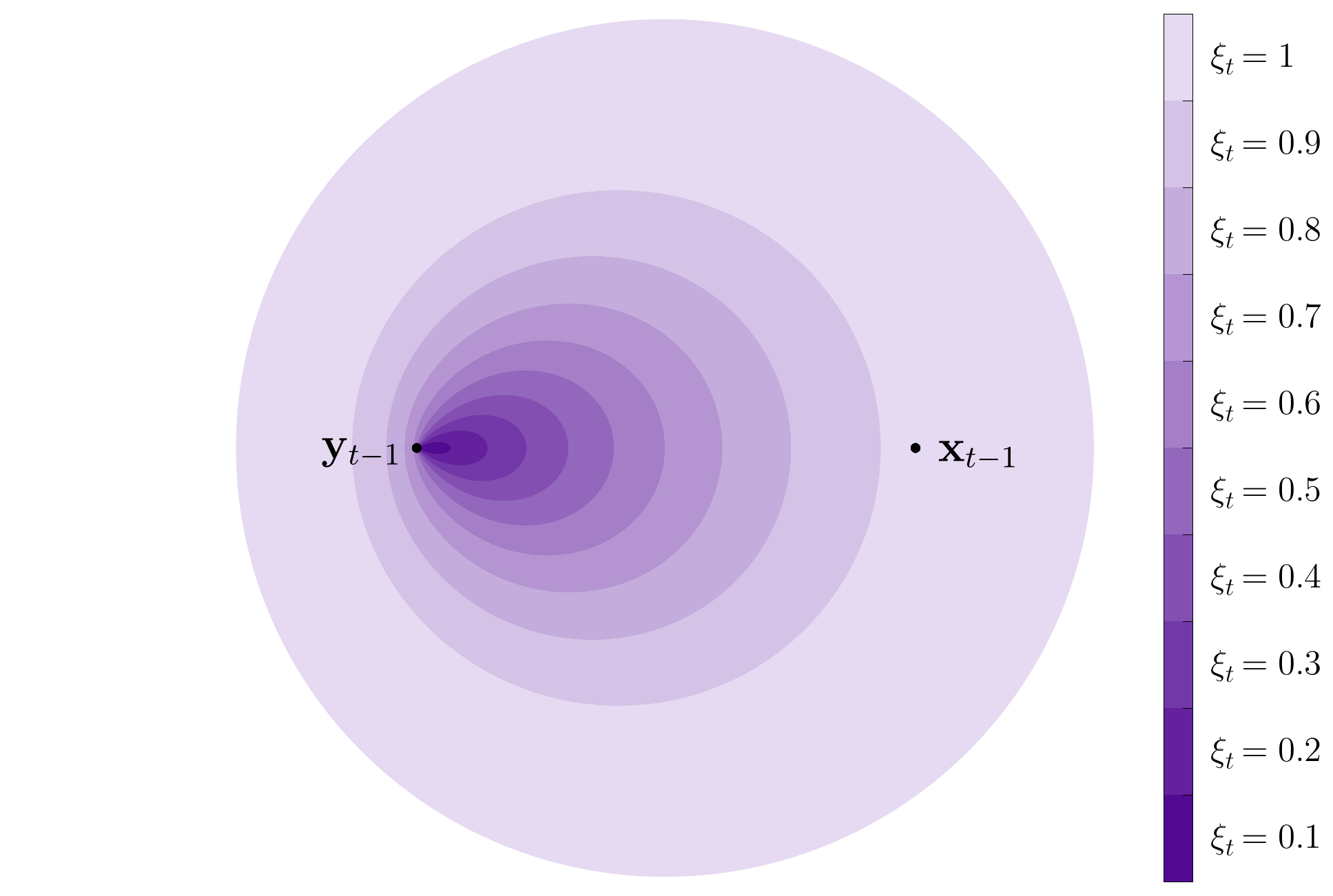} 
\\[1.1em]
 $\T_\xi$ when $\xi_t = 0.2$ (enlarged)
& 
 $\T_\xi$ for many values of $\xi_t$
\end{tabular}
\end{center}
\begin{small}
\caption{\textbf{Geometry of equivalence regions.}
As $\xi_t$ increases, the size of $\T_\xi$ increases, the number of constraints in $\W_t$ increases, and the amount of guaranteed progress increases.
For small $\xi_t$, $\T_\xi$ has a ``teardrop'' shape.  When $\xi_t = 1$, $\T_\xi$ is a ball with center $\oh (\x_{t-1} + \y_{t-1})$.
To generate the figure, we let $\norm{\x_{t-1} - \y_{t-1}}^2/ \Delta_{t-1}  = 1$.  
\label{fig:s_region}
}
\end{small}
\end{figure}

We prove \lemref{lem:gap_progress_betat} in \fullappref{app:proof_blitz_lasso_lemma2}. 
On its own, this lemma is impractical, as it assumes knowledge of $\betat$ when choosing $\W_t$.
Since $\betat$ is unknown
\algnamel{} considers \emph{all possible} $\beta$ when selecting $\W_t$.  To do so, we define the equivalence region
\[
\T_\xi = \bigcup_{\beta \in (0, \toh)} \B_\xi(\beta) \, .
\]
This new equivalence region leads to the following result, which we prove in \fullappref{app:proof_blitz_lasso_theorem}:
\begin{samepage}
\begin{restatable}[Guaranteed progress during iteration $t$ of \algnamel{}]{thm}{blitzlassoprogress} \label{thm:blitz_lasso_theorem}
During iteration $t$ of \algnamel{}, consider any progress coefficient $\xi_t \in (0, 1]$.
For all $i \in [m]$, let $\W_t$ include constraint $i$ if either $\T_\xi \cap \{ \x  \mid  \ip{\a_i, \x} \geq b_i \} \ne \emptyset$ or $\ip{\a_i, \xtm} = b_i$.
Then
\[ \Delta_{t} \leq (1 - \xi_t) \Delta_{t-1} \, . \]
\end{restatable}
\end{samepage}

\subsection{Computing $\W_t$ efficiently with capsule approximations \label{sec:capsule}}

\figref{fig:s_region} contains 
renderings of $\T_\xi$.
Note how $\T_\xi$ grows in size as $\xi_t$ increases, since $\tau_\xi(\beta)$ also increases with $\xi_t$ due to \eqnref{eqn:tau}.
Unfortunately, using $\T_\xi$ 
to select $\W_t$ is problematic in practice, since
testing if
$\T_\xi \cap \{ \x \mid  \ip{\a_i, \x} \geq b_i \} \ne \emptyset$ is not simple.
To reduce computation, \algnamel{} constructs $\W_t$ using a relaxed equivalence region, $\Tcap$. This set is the convex hull of two balls:


\[
\Tcap = \conv{\Bcap_1 \cup \Bcap_2} \, , \quad \text{where} 
\]
\[
\begin{array}{lcl}
\Bcap_1 = \{ \x \mid \norm{\x - \ccap_1} < \rcap \} \, , &&
\Bcap_2 = \{ \x \mid \norm{\x - \ccap_2} < \rcap \} \, ,  \\[0.8em]
\ccap_1 = \ytm + \frac{\xtm - \ytm}{\norm{\xtm - \ytm}} (\dminc + \rcap) \, , &&
\ccap_2 = \ytm + \frac{\xtm - \ytm}{\norm{\xtm - \ytm}} (\dmaxc - \rcap) \, , \\[0.8em]
\dminc = \inf{\beta \, : \, \tau_\xi(\beta) > 0} \beta \norm{\xtm - \ytm} - \tau_\xi(\beta)  \, , &&
\dmaxc = \sup{\beta \, : \, \tau_\xi(\beta) > 0} \beta \norm{\xtm - \ytm} + \tau_\xi(\beta)  \, , 
\end{array}
\]
\vspace{-0.75em}
\[
\text{and} \quad \tmaxc = \sup{\beta \, : \, \tau_\xi(\beta) > 0} \tau_\xi(\beta)  \, .
\]

\begin{figure}
\begin{center}
\includegraphics[width=4.25in]{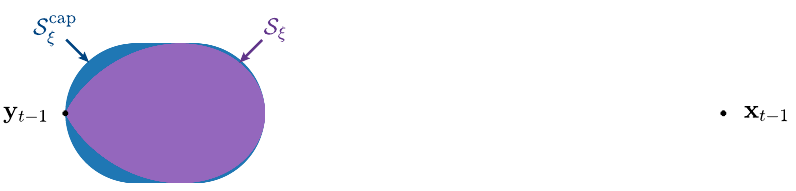} 
\end{center}
\begin{small}
\caption{\textbf{Capsule approximation.}
\label{fig:s_region_comparison}
To simplify computation, \algnamel{} constructs $\W_t$ using $\Tcap$, which is a relaxation of $\T_\xi$.  $\T_\xi$ is shaped like a teardrop when $\xi_t$ is small, while $\Tcap$ is the smallest capsule (convex hull of two balls with equal radius) that contains $\T_\xi$.
For the figure, $\xi_t = 0.4$, and $\norm{\xtm - \ytm}^2 / \Delta_{t-1} = 1$.
}
\end{small}
\end{figure}

In \fullappref{sec:proof_capsule_subset}, we prove that $\T_\xi \subseteq \Tcap$.
Illustrated in \figref{fig:s_region_comparison}, 
$\Tcap$
is the smallest ``capsule'' (set of points within a fixed distance from a line segment) for which $\T_\xi \subseteq \Tcap$.
We define the capsule using three scalars: $\dminc$, $\dmaxc$, and $\tmaxc$ (in addition to the points $\xtm$ and $\ytm$).  
The radius of the capsule is $\tmaxc$, while $\dminc$ and $\dmaxc$ parameterize the capsule's endpoints.
Determining these three parameters requires solving three 1-D optimization problems,
which we can solve efficiently:
\begin{restatable}[Computing capsule parameters is quasiconcave]{thm}{quasiconcavethm} \label{thm:quasiconcave}
For each $s \in \{-1, 0, +1\}$, the function $q_s(\beta) = s \beta \norm{\xtm - \ytm} + \tau_\xi(\beta)$ is quasiconcave over $\{ \beta \mid q_s(\beta) > 0 \}$.
Thus, $\dminc$, $\dmaxc$, and $\tmaxc$ are suprema of 1-D quasiconcave functions.
\end{restatable}
We prove \thmref{thm:quasiconcave} in \fullappref{app:proof_quasiconcave}.
Making use of this result, \algnamel{} computes $\dminc$, $\dmaxc$, and $\tmaxc$ using the bisection method.
Empirically we find the computational cost of computing $\Tcap$ is negligible compared to the cost of solving each subproblem (unless \prbref{prb:general} is very small).


After computing $\Tcap$, 
it is simple to test whether $\Tcap \cap \{ \x \mid \ip{\a_i, \x} \geq b_i \} \ne \emptyset$.
The condition is true iff $\Bcap_1$ or $\Bcap_2$ intersect $\{ \x \mid \ip{\a_i, \x} \geq b_i \}$.
It follows that $\Tcap \cap \{ \x \mid \ip{\a_i, \x} \geq b_i \} \ne \emptyset$ iff
\[
b_i - \mathrm{max}\left\{\ip{\a_i, \ccap_1}, \ip{\a_i, \ccap_2} \right\}
< {\norm{\a_i}} \rcap  \, .
\]

\algnamel{} includes in the working set any $i$ for which either the above inequality is satisfied or $\ip{\a_i, \xtm} = b_i$.
Since $\T_\xi \subseteq \Tcap$, this $\W_t$ satisfies the conditions for \thmref{thm:blitz_lasso_theorem}, ensuring the suboptimality gap decreases by at least a $1 - \xi_t$ factor during iteration $t$.

\subsection{\algnamel{} definition and convergence guarantee}

\begin{algorithm}[t]
\caption{\algnamel{} for solving \prbref{prb:lasso_dual}}
\label{alg:blitz_lasso}
\begin{algorithmic}
\State \textbf{input} feasible point $\y_0$ and method for choosing $\xi_t$ for all $t$
\State \textbf{initialize} $\x_0 \gets \m{0}$
\For{$t = 1, \ldots, T$ \textbf{until} $\x_T = \y_T$}
  \State Choose progress coefficient $\xi_t \in (0, 1]$
  \State $\Tcap \gets {\tt compute\_capsule\_region}(\xi_t, \x_{t-1}, \y_{t-1})$
  \quad \emph{\# see \secref{sec:capsule}}
  \State $\W_t \gets \{ i \in [m] \mid {\tt include\_in\_working\_set?}(i, \Tcap, \xtm)\}$
\State $\x_t \gets \mathrm{argmin} \left\{ \tfrac{1}2 \norm{\x}^2 \, \big| \, \ip{\a_i, \x} \leq b_i \ \mathrm{for\ all}\ i \in \W_t \right\}$
  \State $\y_t \gets {\tt compute\_extreme\_feasible\_point}(\x_t, \y_{t-1})$
\EndFor
\State \textbf{return} $\y_T$
\\
\vspace{-0.5em}
\hrulefill
\Function{${\tt include\_in\_working\_set?}$}{$i, \Tcap, \x_{t-1}$}
\State $\ccap_1, \ccap_2, \rcap \gets {\tt get\_capsule\_centers\_and\_radius}(\Tcap)$
  \quad \emph{\# see \secref{sec:capsule}}
\If{$\left({b_i - \mathrm{max}\left\{ \ip{\a_i, \ccap_1}, \ip{\a_i, \ccap_2} \right\} } < \norm{\a_i} \rcap \right)\ \textbf{or}\ (\ip{\a_i, \x_{t-1}} = b_i)$}
\State \textbf{return true}
\EndIf
\State \textbf{return false}
\EndFunction
\\
\vspace{-0.5em}
\hrulefill
\Function{${\tt compute\_extreme\_feasible\_point}$}{$\x_t, \y_{t-1}$}
\State $\alpha_t \gets 1$
\For{$i \in [m]$}
\If{$(\ip{\a_i, \x_t} > b_i)$}
\State $\alpha_t \gets \mathrm{min}\left\{\alpha_t, \frac{b_i - \ip{\a_i, \ytm}}{\ip{\a_i, \xt} - \ip{\a_i, \ytm}}\right\}$
\EndIf
\EndFor
\State \textbf{return} $\alpha_t \x_t + (1 - \alpha_t) \y_{t-1}$
\EndFunction
\end{algorithmic}
\end{algorithm}

We formally define \algnamel{} in \algref{alg:blitz_lasso}.
\algnamel{} assumes an initial feasible point $\y_0$ and initializes $\x_0$ as the zero vector (``subproblem 0'' is implicitly defined as minimizing $\fld$ subject to no constraints).
The suboptimality gap decreases with the following guarantee:

\begin{restatable}[Convergence bound for \algnamel{}]{thm}{blitzlassoconvergence} \label{thm:blitz_lasso_theorem2}
For any iteration $T$ of \algref{alg:blitz_lasso}, define the suboptimality gap $\Delta_T = \fld(\y_T) - \fld(\x_T)$.
For all $T > 0$, we have
\[
\Delta_T \leq \Delta_0 \prod_{t=1}^T (1 - \xi_t) \, .
\]
\end{restatable}


We have yet to address several practical considerations for \algnamel{}.
This includes analysis of approximate subproblem solutions and a procedure for selecting $\xi_t$ during each iteration.
We address these details in \secref{sec:blitz} in the context of our more general working set algorithm, \algname{}.
Before that, we define a more general problem formulation.

\section{Exploiting piecewise linear structure in convex problems  \label{sec:formulation}}

Rather than exploiting irrelevant constraints, \algname{} exploits piecewise linear structure. 
In this section, we reformulate the objective function
to accommodate
this more general concept.

\subsection{Piecewise problem formulation \label{sec:piecewise_framework}}

For the remainder of this work, we consider convex optimization problems of the form
\begin{equation}
 \minimize{\x \in \D} f(\x) := \psi(\x) + \sump \phi_i(\Ai \x)  \, .
  \label{prb:general}
  \tag{P}
 \end{equation} 
We assume each $\phi_i$ is \emph{piecewise}.
That is, for each $\phi_i$, we assume a domain-partitioning function $\pi_i \, : \, \D \rightarrow \{ 1, 2, \ldots, p_i \}$ and corresponding subfunctions $\phi_i^{(1)}, \phi_i^{(2)}, \ldots, \phi_i^{(p_i)}$ such that
\[
\phi_i(\x) = 
\left\{
\begin{array}{cl}
\phi_i^{(1)}(\x) & \text{if}\ \pi_i(\x) = 1, \\[0.0em]
\vdots & \\[0.2em]
\phi_i^{(p_i)}(\x) & \text{if}\ \pi_i(\x) = p_i \, . \\[0.1em]
\end{array}
\right. 
\]
We assume that $\psi$, $\phi_i$, and $\phi_i^{(k)}$ for all $i$ and $k$ are convex lower semicontinuous functions.  We also assume that $\psi$ is $1$-strongly convex. (We can adapt this formulation to the more general case that $\psi$ is $\gamma$-strongly convex for some $\gamma > 0$ by scaling $f$ by $\gamma^{-1}$.)

Let $\X_i^{(k)}$ denote the $k$th subdomain of $\phi_i$:
$
\X_i^{(k)} = \{ \x \mid \pi_i(\Ai \x) = k \}  .
$
Denoting \prbref{prb:general}'s solution by $\xstar$, 
we use $\Xis$ to denote the subdomain of $\phi_i$ that contains $\xstar$.  

\algname{} efficiently solves \prbref{prb:general} by exploiting $f$'s piecewise structure.
We focus on instances of \prbref{prb:general} for which the piecewise functions are the primary obstacle to efficient optimization (generally problems for which $m$ is large).
We also focus on instances of \prbref{prb:general} for which many $\phi_i^{(k)}$ subfunctions are linear.
We base our methods on the following proposition:
\begin{prop}[Exploiting piecewise structure at $\xstar$] \label{prop:pw}
For each $i \in [m]$, assume knowledge of $\pi_i(\Ai \xstar)$ and whether \mbox{$\xstar \in \bd{\Xis}$}. Define
$\phi_i^\star = \phi_i$ if $\xstar \in \bd{\Xis}$ and 
$\phi_i^\star = \phi_i^{(\pi_i(\xstar))}$ otherwise.
Then $\xstar$ is also the solution to 
\begin{equation}
 \minimize{\x \in \D} f^\star(\x) := \psi(\x) + \sump \phi^\star_i(\Ai \x)  \, .
  \label{prb:oracle}
  \tag{P$^\star$}
 \end{equation}
\end{prop}
\propref{prop:pw} states that if $f$'s minimizer does not lie on the boundary of $\Xis$, then replacing $\phi_i$ with the subfunction $\phi_i^{(\pi_i(\xstar))}$ in $f$ does not change the objective's minimizer.
We can verify this by observing that $f^\star$ preserves the subgradient of $f$ at $\xstar$, which implies that $\m{0} \in \partial f^\star(\xstar)$.

Despite matching solutions, solving \prbref{prb:oracle} can require \emph{much less} computation than solving \prbref{prb:general}.
This is especially true when many $\phi_i^\star$ are linear subfunctions. In this case, the linear subfunctions collapse into a single linear term, making $f^\star$ simpler to minimize than $f$.
We next provide some examples to illustrate this idea.


\subsection{Piecewise linear structure in machine learning \label{sec:examples}} 
We now describe several instances of \prbref{prb:general} that are importance to machine learning.  

\subsubsection{Inactive constraints in constrained optimization \label{sec:constraind_example}}

We first consider constrained optimization (for which \prbref{prb:lasso_dual} is a special case):
\begin{equation}
\begin{array}{cll}
\minimize{\x \in \D} & \psi(\x) & \\
\mathrm{s.t.} & \sigma_i(\Ai \x) \leq 0 & i = 1, \ldots, m \, .
\end{array}
\label{prb:constrained}
\tag{PC}
\end{equation}
If each $\sigma_i$ is convex and $\psi$ is $1$-strongly convex,
this problem
can be transformed into an instance of \prbref{prb:general} using implicit constraints. 
For each $i \in [m]$, define $\phi_i$ as 
\[
\phi_i(\Ai \x) = 
\left\{
  \begin{array}{ll}
    0 & \text{if}\ \sigma_i(\Ai \x) \leq 0  , \\
    +\infty & \text{otherwise.}
  \end{array}
  \right. 
\]
There are two subdomains for each $\phi_i$---the constraint's feasible and infeasible regions. 
Since $\xstar$ must satisfy all constraints, note $\Xis$ represents constraint $i$'s feasible region.

Let us consider \propref{prop:pw} in the context of \prbref{prb:constrained}.
Define $\mathcal{W}^\star = \{ i \mid \sigma_i(\xstar) = 0 \}$, the set of constraints that are active at \prbref{prb:constrained}'s solution.
The condition \mbox{$\xstar \in \bd{\Xis}$} implies \mbox{$\sigma_i(\Ai \xstar) = 0$} and $i \in \W^*$.
To define each $\phi_i^\star$, we let $\phi_i^\star = \phi_i$ for all $i \in \W^\star$ and $\phi_i^\star = 0$ otherwise.
Applying \propref{prop:pw}, we see that $\xstar$ also solves
\begin{equation}
\begin{array}{cll}
\minimize{\x \in \D} & \psi(\x) & \\
\mathrm{s.t.} & \sigma_i(\Ai \x) \leq 0 & \quad \mathrm{for}\ i \in \mathcal{W}^\star \, .  \end{array}
\label{prb:constrained_star}
\tag{PC$^\star$}
\end{equation}
That is, we have reduced \prbref{prb:constrained} to a problem with only $ \abs{\mathcal{W}^\star}$ constraints. 
Since often $\abs{\mathcal{W}^\star} \ll m$,
solving \prbref{prb:constrained_star} can be significantly simpler than solving the original problem.

\subsubsection{Zero-valued weights in sparse optimization \label{sec:sparsity_inducing}}

\begin{table*}
\begin{center}
\begin{small}
\begin{tabular}{l@{\hskip 0.1in}c@{\hskip 0.1in}c@{\hskip 0.1in}c}
\hline
\abovespace\belowspace
\sc{Loss} & $L_j(\ip{\a_j, \omegab})$ & $L_j^*(x_j)$  \\
\hline
\\[-0.9em]
\abovespace
\belowspace
Logistic & 
$4 \log(1 + \exp(-b_j \ip{\a_j, \omegab}))$ &
$
-\tfrac{1}4 \tfrac{x_j}{b_j} \log( - \tfrac{x_j}{b_j} ) +\tfrac{1}4  (1 + \tfrac{x_j}{b_j}) \log( 1 + \tfrac{x_j}{b_j} )
$
\\[-0.8em]
 \\
\abovespace
\belowspace
\begin{tabular}{@{}l@{}}
Squared \\ hinge 
\end{tabular}
 &
$
\left\{
  \begin{array}{ll}
    \oh (1 - b_j \ip{\a_j, \omegab})^2  & \text{if}\ b_j \ip{\a_j, \omegab} \leq 1  , \\
    0 & \text{otherwise}
  \end{array}
  \right.
$
&
$
\left\{
  \begin{array}{ll}
    \oh (\tfrac{x_j}{b_j} + 1)^2 - \oh  & \text{if}\ \tfrac{x_j}{b_j} \leq 0  , \\
    +\infty & \text{otherwise}
  \end{array}
  \right.
$
\\[-0.8em]
 \\
\abovespace
\belowspace
Squared & 
$\oh (\ip{\a_j, \omegab} - b_j)^2$ &
$\oh (x_j + b_j)^2 - \oh b_j^2$ 
\\[-0.8em]
 \\
\abovespace
\belowspace
Huber
 &
$
\left\{
  \begin{array}{ll}
    \oh ( \ip{\a_j, \omegab} - b_j )^2 & \text{if}\ \abs{\ip{\a_j, \omegab} - b_j} \leq s , \\
    s \abs{\ip{\a_j, \omegab} - b_j} - \oh s^2 & \text{otherwise}
  \end{array}
  \right.
$
&
$
\left\{
  \begin{array}{ll}
    \oh ( x_j + b_j)^2 - \oh b_j^2  & \text{if}\ \abs{x_j} \leq s  , \\
    +\infty & \text{otherwise}
  \end{array}
  \right.
$
\\[-0.8em]
 \\
\hline
\end{tabular}
\end{small}
\end{center}
\caption{\textbf{Smooth loss examples.}
The dual of $\ell_1$-regularized smooth loss minimization is a strongly convex constrained problem.
Each feature corresponds to a dual constraint. We can use working sets to make convergence times less dependent on the number of features.
In the table, $L_j$ is the loss for training example $j$, $(\a_j, b_j)$ is the $j$th example, and $s$ is a design parameter. 
}
\label{tbl:smooth_losses}
\end{table*}

Optimization with sparsity-inducing penalties is popular in machine learning---see \cite{Bach:2012} for a survey.
 Here we consider 
learning $\ell_1$-regularized linear models.

Let $( (\a_j, b_j) )_{j=1}^n$ be a collection of $n$ training examples where $\a_j \in \reals^m$ is a feature vector and $b_j \in \mathcal{B}$ is a corresponding label.  Typically $\mathcal{B} = \{-1, +1\}$ for classification problems, while $\mathcal{B} = \reals$ for regression.
We can fit parameters of a linear model to this data by solving
\begin{equation}
\minimize{\omegab \in \reals^m} \sum_{j=1}^n L_j(\ip{\a_j, \omegab})  + \lambda \norm{\omegab}_1 \, .
\label{prb:l1_loss}
\tag{PL1}
\end{equation}
Above $\lambda > 0$ is a tuning parameter, and $L_j$ is a loss function (parameterized by $b_j$).
  When $\lambda$ is sufficiently large, a solution $\omegab^\star$ is \emph{sparse}, meaning most entries of $\omegab^\star$ equal 0.

\prbref{prb:l1_loss} is not directly an instance of \prbref{prb:general}, since this problem is not $1$-strongly convex in general.
 Assuming each $L_j$ is $1$-smooth, however, we can transform \prbref{prb:l1_loss} into an instance of \prbref{prb:constrained} by considering the problem's dual \citep[Chapter 4]{Borwein:2005}: 
\begin{equation}
\begin{array}{cll}
\minimize{\x \in \D} & \sum_{j=1}^n L_j^*(x_j) &  \\
\mathrm{s.t.} & \abs{\ip{\A_i, \x}} \leq \lambda & i = 1, \ldots, m \, .
\end{array}
\label{prb:l1_dual}
\tag{PL1D}
\end{equation}
By solving \prbref{prb:l1_loss}, we can efficiently recover \prbref{prb:l1_dual}'s solution and vice versa.
In the dual problem, $\A_i \in \D$ refers to the $i$th column (feature) of the $n \times m$ design matrix $[\a_1 , \ldots , \a_n]^T$.  $L_j^*$ denotes the convex conjugate of $L_j$. 
Since each $L_j$ is 1-smooth, the $L_j^*$ terms are 1-strongly convex \citep[Chapter 12]{Rockafellar:1997}.
We include several examples of smooth loss functions and their convex conjugates in \tblref{tbl:smooth_losses}.

This dual transformation allows \algname{} to exploit sparsity to solve \prbref{prb:l1_loss} efficiently.
Due to the correspondence between features and dual constraints, zero entries in $\omegab^\star$ correspond to constraints that are unnecessary for computing \prbref{prb:l1_dual}'s solution.
That is, if we define $\W^\star = \{ i \mid \omega_i^\star \ne 0 \}$, then we can also compute $\xstar$ by minimizing the dual objective subject only to constraints in $\W^\star$.
Since $\omegab^\star$ is sparse, solving the problem with only $\abs{\W^\star}$ constraints typically requires much less computation than solving \prbref{prb:l1_dual} directly.


\subsubsection{Training examples in support vector machines}
\label{sec:svm_example}

Our final example considers support vector machines and, more generally, $\ell_2$-regularized loss minimization problems with piecewise loss functions.
Given training examples $\listm{ (\a_i, b_i) }$ and a tuning parameter $C > 0$,
we can learn a linear model by solving
\[
\minimize{\x \in \D} \tfrac{1}2 \norm{\x}^2 + C \sump L_i(\x) \, .
\tag{PL2}
\label{prb:l2_loss}
\]
Each $L_i$ is a loss function (parameterized by $\a_i$ and $b_i$). 
Often each $L_i$ has piecewise linear components; we include some examples of  such losses in \tblref{tbl:regularized_piecewise_losses}.  

\begin{table}
\begin{center}
\begin{small}
\begin{tabular}{@{\hskip 0.10in}l@{\hskip 0.35in}c@{\hskip 0.35in}l@{}l@{\hskip 0.1in}}
\hline
\abovespace\belowspace
\sc{Loss} & \sc{Use} & \multicolumn{2}{c}{$L_i(\x)\ \ \ \ \ $} \\
\hline
\\[-0.8em]
\abovespace
\belowspace
Hinge & 
Classification & 
$\left\{
  \begin{array}{ll}
    1 - b_i \ip{\a_i, \x} \\
    0 
  \end{array}
  \right.
$
& 
$
  \begin{array}{l}
    \text{if}\ b_i \ip{\a_i, \x} \leq 1  , \\
    \text{otherwise}
  \end{array}
$
\\[-0.8em]
\\
\abovespace
\belowspace
\begin{tabular}{@{}l@{}}
Squared \\ hinge 
\end{tabular} &
Classification & 
$\left\{
  \begin{array}{l}
    \oh (1 - b_i \ip{\a_i, \x})^2 \\
    0 
  \end{array}
  \right.
$
&
$
\begin{array}{l}
  \text{if}\ b_i \ip{\a_i, \x} \leq 1  , \\
  \text{otherwise}
\end{array}
$
\\[-0.8em]
\\
\abovespace
\belowspace
Quantile & Regression & 
$\left\{
\begin{array}{l}
  (1 - s) (b_i - \ip{\a_i, \x})  \\
  s (\ip{\a_i, \x} - b_i) 
\end{array}
\right.
$
& 
$
\begin{array}{l}
  \text{if}\ \ip{\a_i, \x} \leq b_i  , \\
  \text{otherwise}
\end{array}
$
\\
\\[-0.8em]
\hline
\end{tabular}
\end{small}
\end{center}
\caption{
\textbf{Losses with piecewise linear subfunctions.}
For $\ell_2$-regularized learning, we can leverage piecewise losses to reduce the training time's dependence on the number of observations.  
Above, $L_i$ is the loss for example $i$, $(\a_i, b_i)$ is the $i$th training example, and $s$ is a design parameter. 
}
\label{tbl:regularized_piecewise_losses}
\end{table}

When each $L_i$ has piecewise linear components, we can solve the problem quickly by exploiting piecewise structure.
Consider \prbref{prb:l2_loss} instantiated with hinge loss.
Given knowledge of $\pi_i(\xstar)$ and whether \mbox{$\xstar \in \bd{\Xis}$} for each $i$ (for this problem, this implies knowledge of the sign of \mbox{$1 - b_i \ip{\a_i, \xstar}$}), 
we can define
\[
  L_i^\star(\x) = 
\left\{
  \begin{array}{ll}
    0 & \text{if}\ 1 - b_i \ip{\a_i, \xstar} < 0 \, , \\
    1 - b_i \ip{\a_i, \x} & \text{if}\ 1 - b_i \ip{\a_i, \xstar} > 0 \, , \\
    L_i(\x) & \text{if}\ 1 - b_i \ip{\a_i, \xstar} = 0 \, .
  \end{array}
  \right.
\]
Applying \propref{prop:pw}, we can compute $\xstar$ by solving
\begin{equation}
\minimize{\x \in \D} \tfrac{1}2 \norm{\x}^2 + C \sum_{i=1}^n L_i^\star(\x)  \, .
\tag{PSVM$^\star$}
\label{prb:l2_loss_oracle_prime}
\end{equation}
If we define $\mathcal{W}^\star = \{ i \, : \,  1 - b_i \ip{\a_i, \xstar} \ne 0 \}$, the benefit becomes clear.  \prbref{prb:l2_loss_oracle_prime} has the same solution as
\begin{equation}
\minimize{\x \in \D} \tfrac{1}2 \norm{\x}^2 + \ip{\a^\star, \x} +  C \sum_{i \in \mathcal{W}^\star} L_i(\x)  \, , \tag{PSVM$^{\star\star}$} \label{prb:l2_loss_oracle}
\end{equation}
where the vector $\a^\star$ is easily computable.
We have reduced \prbref{prb:l2_loss} from a problem with $m$ training examples to a problem with $\abs{\mathcal{W}^\star}$ training examples.    Since often $\abs{\mathcal{W}^\star} \ll m$, we can solve \prbref{prb:l2_loss_oracle_prime} much more efficiently.

\section{\algname{} working set algorithm \label{sec:blitz}}


\algname{} extends \algnamel{} to problems of the form \prbref{prb:general}.
In this section,
we adapt main concepts from \secref{sec:blitz_lasso} to this piecewise formulation.
We also address some practical considerations for \algname{}. 

\subsection{Working set algorithms for piecewise objectives}

To generalize working set algorithms to our piecewise problem formulation,
we generalize the form of each subproblem.
During each iteration $t$, \algname{}
minimizes a ``relaxed objective,''
\begin{equation} \label{eqn:blitz_relaxed_9aq}
f_t(\x) = \psi(\x) + \sum_{i=1}^m \phi_{i,t}(\x) \, , 
\end{equation}
where for $i \in [m]$, we have $\phi_{i,t} \in \{ \phi_i \} \cup \{ \phi_i^{(1)}, \phi_i^{(2)}, \ldots, \phi_i^{(p_i)} \}$.
That is, \algname{} defines each $\phi_{i,t}$ as either (i)~the original piecewise function, $\phi_i$, or (ii)~one of $\phi_i$'s simpler subfunctions.
The ``working set'' is the set of indices corresponding to piecewise functions in $f_t$.


Analogous to \algnamel{}, \algname{} computes $\x_t \gets \mathrm{argmin}\ f_t(\x)$ during iteration $t$.
As long as $\phi_{i,t}$ is linear for most $i$, solving this subproblem generally requires much less time than solving \prbref{prb:general}.

\subsection{Line search and $\y_t$ in \algname{} \label{sec:blitzws_line}}

Like \algnamel{}, \algname{} maintains two iterates, $\xt$ and $\yt$.
In \algnamel{}, $\y_t$ is the feasible point on the segment $[\y_{t-1}, \x_t]$ with smallest objective value.

Extending this concept to the piecewise problem,
\algname{} initializes $\y_0$ such that $f(\y_0)$ is finite.
After computing $\xt$, \algname{} updates $\yt$ using the rule
\[
\y_t \gets \mathrm{argmin} \left\{{f(\x) \mid \x \in [\y_{t-1}, \x_t]}  \right\} \, . \label{eqn:linesearch}
\]
\algname{} can compute this update using the bisection method; we elaborate on this in \secref{sec:backtracking_time23}.


\subsection{Equivalence regions in \algname{} \label{sec:blitzws_equivalence}}

\algname{}'s choice of each $\phi_{i,t}$ is a choice of \emph{where} in $\phi_i$'s domain the algorithm ensures that $\phi_{i,t}$ and $\phi_i$ are equivalent.
If $\phi_{i,t} = \phi_i^{(1)}$, then $\phi_{i,t}(\x) = \phi_i(\x)$ is only guaranteed for $\x \in \X_i^{(1)}$.


Like \algnamel{},
\algname{} uses equivalence regions to ensure quantifiable progress during each iteration.
Given a progress parameter $\xi_t \in (0, 1]$, \algname{} defines $\Tcap$ exactly as in \secref{sec:capsule}:
\[
\Tcap = \conv{\Bcap_1 \cup \Bcap_2} \, .
\]
Recall that $\Bcap_1$ and $\Bcap_2$ are balls with centers $\ccap_1$ and $\ccap_2$, respectively.

\algname{} selects each $\phi_{i,t}$ so that for all $\x \in \Tcap$, we have $f_t(\x) = f(\x)$.
To establish this equivalence region, \algname{} defines each $\phi_{i,t}$ so that $\phi_{i,t}(\x) = \phi_i(\x)$ for all $\x \in \Tcap$, $i \in [m]$.
This property results from \algname{}'s first sufficient condition for including a particular $i \in [m]$ in the working set (there are three conditions total). For each $i \in [m]$, $k = \pi_i(\ccap_1)$, \algname{} includes $i$ in the working set---that is, \algname{} defines $\phi_{i,t} = \phi_i$---if the following:
\begin{itemize}
\item[(C1)] There exists an $\x \in \Tcap$ such that $\x \notin \Xik$, meaning \emph{equivalence between $\phi_i$ and $\phi_i^{(k)}$ within $\Tcap$ is not guaranteed.}
\end{itemize}

We will define conditions (C2) and (C3) soon.
If any of these sufficient conditions are satisfied, \algname{} includes $i$ in the working set.
Otherwise, \algname{} defines $\phi_{i,t} = \phi_i^{(k)}$,
where $k = \pi_i(\ccap_1)$.


\subsection{Generalizing suboptimality gaps for \algname{}} \label{sec:blitz_tractable}

 Condition (C2) for constructing the working set
 allows us to generalize the suboptimality gap to our piecewise formulation.
Specifically, for each $i \in [m]$, $k = \pi_i(\ccap_1)$, \algname{} defines $\phi_{i,t} = \phi_i$ if:
\begin{itemize}
\item[(C2)] There exists an $\x \in \reals^n$ such that $\phi_i^{(k)}(\x) > \phi_i(\x)$---that is, \emph{$\phi_i^{(k)}$ does not lower bound $\phi_i$.}
\end{itemize}

From (C2) and \eqnref{eqn:blitz_relaxed_9aq}, it follows that $f_t(\x) \leq f(\x)$ for all $\x$.
 Importantly, we can use this property to define a suboptimality gap.
Since $\x_t$ minimizes $f_t$, we have
$
f_t(\x_t) \leq f_t(\xstar) \leq f(\xstar)  .
$
Thus, given any $\y_t$ such that $f(\y_t)$ is finite, we have the suboptimality gap
\[
\Delta_t = f(\y_t) - f_t(\x_t) \geq f(\y_t) - f(\xstar) \, .
\]

We note (C2) does not generally affect the computational cost of  minimizing $f_t$.
The advantage of minimizing $f_t$ instead of $f$ results from the fact that linear $\phi_{i,t}$ functions collapse into a single linear term. 
In the case that $\phi_i^{(k)}$ is linear, then $\phi_i^{(k)}$ lower bounds $\phi_i$ as a result of convexity (assuming $\Xik$ has non-empty interior).
\subsection{Ensuring $f_t(\x_t)$ is nondecreasing with $t$ \label{sec:blitz_sandwich}}


\algnamel{} defines each working set to ensure $\fld(\xt)$ is nondecreasing with $t$.
Generalizing this idea, \algname{} ensures $f_t(\xt) \geq f_{t-1}(\x_{t-1})$ for all $t$.
This property follows from the third condition for including $i$ in the working set.
For each $i \in [m]$, $k = \pi_i(\ccap_1)$, \algname{} defines $\phi_{i,t} = \phi_i$ if:
\begin{itemize}
\item[(C3)] \emph{$\phi_i^{(k)}$ does not upper bound $\phi_{i,t-1}$ in a neighborhood of $\xtm$}---there exists a \mbox{$\g_i \in \partial \phi_{i,t-1}(\xtm)$} and $\x \in \reals^n$ such that
\[
\phi_i^{(k)}(\x) < \phi_{i,t-1}(\xtm) + \ip{\x - \xtm, \g_i} \, .
\]
\end{itemize}
Because of (C3), we have $f_t(\x) \geq f_{t-1}(\x_{t-1}) + \ips{\x - \x_{t-1}, \g_{t-1}}$
for all $\x$ and $\g_{t-1} \in \partial f_{t-1}(\x_{t-1})$.
Since $\x_{t-1}$ minimizes $f_{t-1}$, we have $\m{0} \in \partial f_{t-1}(\x_{t-1})$.
It follows that $f_t(\x) \geq f_{t-1}(\x_{t-1})$ for all $\x$, which implies that
$f_t(\xt) \geq f_{t-1}(\xtm)$.

\subsection{\algname{} definition and convergence guarantee}

\begin{algorithm}[t]
\caption{\algname{} for solving \prbref{prb:general}}
\label{alg:working_set}
\begin{algorithmic}
\State \textbf{input} initial $\y_0$ such that $f(\y_0) < +\infty$, linear functions $(\phi_{i,0})_{i=1}^m$ for which $\phi_{i,0}(\x) \leq \phi_i(\x) \ \forall \x$,
\State \hspace{0.344in} and method for choosing $\xi_t$ 
\vspace{0.2em}
\State $\x_0 \gets \mathrm{argmin}\: f_0(\x) := \psi(\x) + \sump \phi_{i, 0}(\Ai \x)$
\vspace{0.2em}
\For{$t = 1, \ldots, T$ \textbf{until} $f_T(\x_T) = f(\y_T)$}
\alglinebreak{}
  \algcomment{Form subproblem}
  \State Choose progress coefficient $\xi_t \in (0, 1]$
  \State $\Tcap \gets {\tt compute\_capsule\_region}(\xi_t, \x_{t-1}, \y_{t-1})$
  \quad \emph{\# see \secref{sec:capsule}}
  \For{$i = 1, \ldots, m$}
    \If{(C1) or (C2) or (C3)}
      \algcomment{Include $i$ in working set}
      \State $\phi_{i,t} \gets \phi_i$
    \Else
      \State $\phi_{i,t} \gets \phi_i^{(k)}\ \text{where}\ k \ \text{is the subdomain for which}\ \Tcap \subseteq \X_{i}^{(k)}$
    \EndIf
  \EndFor
\alglinebreak{}
  \algcomment{Solve subproblem}
  \State $\x_t \gets \mathrm{argmin}\: f_t(\x) := \psi(\x) + \sump \phi_{i, t}(\Ai \x)$
\alglinebreak{}
  \algcomment{Perform line search update}
  \State $\y_t \gets \argmin{\x \in [\y_{t-1}, \x_t]} f(\x)$
\EndFor
\State \textbf{return} $\y_T$
\end{algorithmic}
\end{algorithm}

We define \algname{} in \algref{alg:working_set}.
\algname{} initializes $\y_0$ such that $f(\y_0)$ is finite, while $\x_0$ is the minimizer of a function $f_0$.  \algname{} defines $f_0$ so that $\x_0$ is easy to compute. For example, for \prbref{prb:constrained}, we can define $\phi_{i,0}(\x) = 0$ for all $i$, making $\x_0$ the unconstrained minimizer of $\psi$.

During each iteration $t$, \algname{} chooses a progress coefficient $\xi_t$, which parameterizes the equivalence region $\Tcap$.
\algname{} defines $f_t$ according to \secref{sec:blitzws_equivalence}, \secref{sec:blitz_tractable}, and \secref{sec:blitz_sandwich}. 
Given $f_t$, \algname{} computes $\xt \gets \mathrm{argmin}\ f_t(\x)$.  At the end of iteration $t$, the algorithm updates $\yt$ via line search. 


Together, conditions (C1), (C2), and (C3) guarantee quantified progress toward convergence during iteration $t$.
In particular, we have the following convergence result for \algname{}:

\begin{restatable}[Convergence bound for \algname{}]{thm}{convergencebound} \label{thm:convergence_rate}
For any iteration $T$ of \algref{alg:working_set},
define the suboptimality gap $\Delta_T = f(\y_T) - f_T(\x_T)$.  For all $T > 0$, we have
\[ \Delta_T \leq \Delta_0 \prod_{t=1}^T (1 - \xi_t) \, .\]
\end{restatable}



\subsection{Accommodating approximate subproblem solutions \label{sec:approximate}}

\algname{} can minimize $f_t$ using any subproblem solver.
Since solvers are usually iterative, it is important to only compute $\x_t$ approximately. Computing $\xt$ to high precision would require time that \algname{} could instead use to solve subproblem $t+1$.

To accommodate approximate solutions, we make several adjustments to \algname{}.
Most significantly, the subproblem solver returns three objects: (i)~an approximate subproblem solution, $\z_t$, where $f_t(\zt)$ is finite, (ii) a function $\ftlb$ that lower bounds $f_t$, and (iii)~$\x_t = \mathrm{argmin}\, \ftlb(\x)$, which is a ``dual'' approximate minimizer of $f_t$.
We assume $\ftlb$ takes the form
\[
f_t^{\mathrm{LB}}(\x) = 
\left[\psi(\zt) + \ip{\g_\psi^{\mathrm{LB}}, \x - \zt} + \oh \norm{\x - \zt}^2\right]
+ 
\sum_{i=1}^m \left[\phi_{i,t}(\zt) + \ip{\gilb, \x - \zt} \right] \, ,
\]
where $\g_\psi^{\mathrm{LB}} \in \partial \psi(\zt)$ and 
$\gilb \in \partial \phi_i(\zt)$
for each $i$.
Since $\psi$ is 1-strongly convex and each $\phi_i$ is convex, we have $\ftlb(\x) \leq f_t(\x)$ for all $\x$.
Also, because $\ftlb$ is a simple quadratic function, it is straightforward to compute $\xt$.

Together, $\zt$, $\xt$, and $\ftlb$ allow us to quantify the precision of the approximate subproblem solutions in terms of suboptimality gap.  Since $\xt$ minimizes $\ftlb$, it follows that
\[
f_t(\zt) - \ftlb(\xt) \geq f_t(\zt) - \mathrm{min}_\x\  f_t(\x) \, .
\]

We note that when subproblem $t$ is solved exactly, we can define $\ftlb$ such that this ``subproblem suboptimality gap'' is zero.
To do so, we define $\g_\psi^{\mathrm{LB}}$ and each $\gilb$ such that $\g_\psi^{\mathrm{LB}} +\sum_{i=1}^m \gilb = \m{0}$.
In this case, $\zt$ also minimizes $\ftlb$, which implies that $\xt = \zt$ and $f_t(\zt) - \ftlb(\xt) = 0$.

To bound the effect of approximate subproblem solutions, 
\algname{} chooses a subproblem termination threshold $\epsilon_t \in [0, 1)$.  We require that $\zt$, $\xt$, and $\ftlb$ satisfy two conditions:
\[
\begin{array}{cc}
& f_t(\zt) - \ftlb(\xt) \leq \epsilon_t \Delta_{t-1} \, , \\[0.4em]
\quad \quad \text{and} \quad \quad &\ftlb(\xt) - \ftmlb(\xtm) \geq (1 - \epsilon_t) \oh \norm{\zt - \xtm}^2 \, .
\end{array}
\]
The first condition bounds the subproblem suboptimality gap. The second condition lower bounds the algorithm's dual progress.
The parameter $\epsilon_t$ trades off the precision of subproblem $t$'s solution with the amount of time used to solve this subproblem---smaller $\epsilon_t$ values imply more precise subproblem solutions.
While not obvious in this context, we note the second condition is always satisfied once the subproblem solution is sufficiently precise. We prove this fact in \fullappref{app:condition_2_always}.

In addition to the changes already discussed, we make three final modifications to \algname{} to accommodate approximate subproblem solutions.
First,
 we redefine \algname{}'s suboptimality gap to ensure $\Delta_t \geq f(\yt) - f(\xstar)$. 
Specifically, we define
$
\Delta_t = f(\yt) - \ftlb(\xt) \, .
$

The second final change is that instead of searching along the segment $[\xt, \ytm]$, \algname{} updates $\yt$ by performing line search along $[\zt, \ytm]$:
\[
\yt \gets \argmin{\x \in [\zt, \ytm]} f(\x) \, .
\]

The last change to \algname{} adjusts condition (C3) from \secref{sec:blitz_sandwich}.
Specifically, for each $i \in [m]$ and \mbox{$k = \pi_i(\ccap_1)$}, \algname{} defines $\phi_{i,t} = \phi_i$ if:
\begin{itemize}
\item[(C3)] \emph{$\phi_i^{(k)}$ does not upper bound $\philbtm$}---for some $\x \in \reals^n$, we have $\phi_i^{(k)}(\x) < \philbtm(\x)$.
\end{itemize}
This change guarantees that $f_t$ upper bounds $\ftmlb$.
Compared to the original (C3) condition from \secref{sec:blitz_sandwich}, our new (C3) guarantees that $\phi_{i,t}$ upper bounds $\phi_{i,t-1}$ in a neighborhood of $\zt$ as opposed to a neighborhood of $\xt$.

Taking these changes into account, we have the following convergence result for \algname{} with approximate subproblem solutions (proven in \fullappref{app:proof_approximate}):
\begin{restatable}[Convergence bound for \algname{} with approximate subproblem solutions]{thm}{convergenceboundapprox} \label{thm:approx_convergence_rate}
Consider \algname{} with approximate subproblem solutions.
For any iteration $T$,
define the suboptimality gap $\Delta_T = f(\y_T) - f_T^{\mathrm{LB}}(\x_T)$.  For all $T > 0$, we have
\[ \Delta_T \leq \Delta_0 \prod_{t=1}^T (1 - (1 - \epsilon_t) \xi_t) \, .\]
\end{restatable}
This result clearly describes the effect of approximate subproblem solutions.
By solving subproblem $t$ with tolerance \mbox{$\epsilon_t \in [0, 1)$}, it is guaranteed that \algname{} makes $(1 - \epsilon_t) \xi_t$ progress during iteration $t$.
When $\epsilon_t = 0$, we recover our original convergence bound, \thmref{thm:convergence_rate}.

The parameters $\xi_t$ and $\epsilon_t$ allow \algname{} to trade off between subproblem size, time spent solving subproblems, and progress toward convergence.
We next explore these trade-offs in more detail.

\subsection{Bottlenecks of \algname{}}

Each iteration $t$ of \algname{} has three stages: select subproblem $t$, solve subproblem $t$, and update $\yt$.
Here we discuss the amount of computation that each stage requires.

\subsubsection{Time required to form each subproblem} \label{sec:time_form_f_t}
The time-consuming step for forming subproblem $t$ is testing condition (C1).
This step requires checking if $\Tcap \subseteq \Xik$ for all $i \in [m]$.

Recall $\Tcap = \conv{\Bcap_1 \cup \Bcap_2}$, where $\Bcap_1$ and $\Bcap_2$ are balls with centers $\ccap_1$ and $\ccap_2$  and radius $\rcap$.
If $\Xik$ is convex, then $\Tcap \subseteq \Xik$ iff $\Bcap_1 \subseteq \Xik$ and $\Bcap_2 \subseteq \Xik$.
Unfortunately, when $\Xik$ is convex, testing whether $\Bcap_1 \subseteq \Xik$ is not convex in general.
Even so, in the common scenarios that $\Xik$ is a half-space or ball, we can check if $\Bcap_1 \subseteq \Xik$ in $\mathcal{O}(n)$ time.
In other cases, we can often approximately check this condition efficiently.

Let us first consider the case that $\Xik$  is a half-space. For some $\a_i \in \reals^n$, $b_i \in \reals$, we can write $\Xik =  \{ \x \mid \ip{\a_i, \x} \leq b_i \}$. Then $\Bcap_1 \subseteq \Xik$ iff
\[
{\ip{\a_i, \ccap_1} - b_i} < \norm{\a_i} \rcap \, .
\]
Alternatively, suppose that $\Xik$ is a ball: $\Xik = \{ \x \mid \norm{\x - \a_i} \leq b_i \}$. Then
 $\Bcap_1 \subseteq \Xik$ iff
\[
\norm{\a_i - \ccap_1} + \rcap < b_i \, .
\]

When $\Xik$ is neither a half-space nor a ball, 
one option may be to approximate condition (C1) by defining a ball $\tXik$ such that $\tXik \subseteq \Xik$.
If $\Bcap_1 \subseteq \tXik$, then $\Bcap_1 \subseteq \Xik$.
%
If $\Bcap_1 \subseteq \tXik$ and $\Bcap_2 \subseteq \tXik$, then we must have $\Tcap \subseteq \Xik$.  
By approximating condition (C1) in this manner, $f_t$ and $f$ remain equivalent within $\T_\xi$, meaning \thmref{thm:convergence_rate} and \thmref{thm:approx_convergence_rate} still apply.

\subsubsection{Time required to solve subproblem $t$}

The time required to solve subproblem $t$ depends mainly on three factors: the progress coefficient, the subproblem termination threshold, and the subproblem solver.
Larger values of $\xi_t$ result in a larger equivalence region, increasing the size of the working set due to (C1).


\algname{} can use any solver to minimize $f_t$, 
but to be effective,
the time required to solve each subproblem must increase with the working set size. 
This is usually the case but not always.
For example, in the distributed setting, communication bottlenecks can affect convergence times greatly.  Depending on the algorithm and implementation, some distributed solvers require $\mathcal{O}(n)$ communication per iteration,
while the communication for other solvers may scale with the working set size.  The $\mathcal{O}(n)$ case is not desirable for \algname{}, since the amount of time needed to solve each subproblem depends less directly on the working set size.

\subsubsection{Time required to update $\yt$ \label{sec:backtracking_time23}}

Updating $\y_t$ requires minimizing $f$ along the segment $[\xt, \ytm]$.
\algname{} can perform this update using the bisection method, which requires evaluating $f$ a logarithmic number of times.
In this case, it is not necessary to compute $\yt$ exactly.  Our analysis requires only that $f(\yt) \leq f(\ytp)$, where $\ytp$ is the point on the segment $[\xt, \ytm]$ that is closest to $\xt$ while remaining in the closure of $\T_\xi$.

In many cases it is also straightforward to compute $\yt$ exactly.
For constrained problems like \prbref{prb:lasso_dual}, $\yt$ is the extreme feasible point on the segment $[\ytm, \xt]$.
If the constraints are linear or quadratic, \algname{} can compute $\yt$ in closed form.

\subsection{Choosing algorithmic parameters in \algname{} \label{sec:alg_params}}

Each \algname{} iteration uses a progress parameter, $\xi_t \in (0, 1]$, and termination threshold, $\epsilon_t \in [0, 1)$.
We could assign $\xi_t$ and $\epsilon_t$ constant values for all $t$.
As we will see in \secref{sec:empirical}, however, values of $\xi_t$ and $\epsilon_t$ that work well for one problem may result in slow convergence times for other problems.
For this reason, it is beneficial to choose these parameters in an adaptive manner.

To adapt the parameter choices to each problem, we model as functions of $\xi_t$ and $\epsilon_t$ both (i) the time required to complete iteration $t$ and (ii) \algname{}'s progress during this iteration.
Using these models, \algname{} selects $\xi_t$ and $\epsilon_t$ by approximately optimizing the trade-offs between subproblem size, iteration duration, and convergence progress.


To model the time required for \algname{} to complete iteration $t$, we define the function
\begin{equation}
\hat{T}_t(\xi, \epsilon) = \underbrace{\Csetup}_{\begin{array}{c} \text{\footnotesize{Estimated time to}} \\[-0.2em] \text{\footnotesize{update $\yt$ and define $f_t$}} \end{array}} + \quad \underbrace{\Csolve {\mathrm{ProblemSize}(\xi)}{\epsilon^{-1}}}_{\begin{array}{c} \text{\footnotesize{Estimated time to}} \\[-0.2em] \text{\footnotesize{solve subproblem $t$}}\end{array}} \, .
\label{eqn:time}
\end{equation}
Above, $\mathrm{ProblemSize}(\xi)$ measures the size of subproblem $t$ as a function of $\xi$.
For \prbref{prb:lasso_dual}, 
we define
\[
\textstyle \mathrm{ProblemSize}(\xi) = \sum_{i \in \mathcal{W}_t(\xi)} \nnz(\a_i) \, ,
\]
where $\W_t(\xi)$ denotes the working set when $\xi_t = \xi$, and $\nnz(\a_i) = \norm{\a_i}_0$.

\algname{} adapts the scalars $\Csetup$ and $\Csolve$ from iteration to iteration.
During each iteration $t$, \algname{} measures the time required to solve subproblem $t$, denoted $\Tsolve_t$, as well as the time taken for all other steps of iteration $t$, denoted $\Tsetup_t$.
Upon completion of iteration $t$, \algname{} estimates $\Csetup$ and $\Csolve$ by solving for the appropriate value in the model:
\[
\hCsetup_t = \Tsetup_t \,, \quad \quad \text{and} \quad \quad \hCsolve_t = \tfrac{\Tsolve_t \epsilon_t}{\mathrm{ProblemSize}(\xi_t)} \, .
\]

When selecting subproblem $t$, \algname{} defines $\Csetup$ and $\Csolve$ by taking the median of the five most recent estimates for these parameters.  For example,
\[
\Csetup = \mathrm{median}(\hCsetup_{t-1}, \hCsetup_{t-2}, \ldots, \hCsetup_{t-5}) \, .
\]
If $t \leq 5$, then the algorithm takes the median of only the past $t - 1$ parameter estimates.
Since this is not possible when $t = 1$ ($\hCsetup_0$ does not exist), \algname{} does not model the time required for iteration $1$.  Instead, during iteration 1, we define $\xi_1$ as the smallest value in $(0, 1]$ such that $f_t = f$, but we solve the subproblem crudely by terminating the subproblem solver after one iteration.

In addition to modeling the time for iteration $t$, \algname{} applies \thmref{thm:approx_convergence_rate} to model the suboptimality gap upon completion of this iteration:
\begin{equation}
 \hat{\Delta}_t(\xi, \epsilon) = \mathrm{max}\left\{(1 - (1-\epsilon) \xi \Cprogress) \Delta_{t-1}, \epsilon \Delta_{t-1} \right\}  \, .  \label{eqn:progress_blitz}
\end{equation}
The parameter $\Cprogress \geq 1$ accounts for looseness in the theorem's bound, which guarantees that $\Delta_t \leq (1 - (1 - \epsilon_t) \xi_t) \Delta_{t-1}$.
The $\mathrm{\max}\{ \cdot \}$ in \eqnref{eqn:progress_blitz} results from the fact that we should not expect that $\Delta_t \leq \epsilon_t \Delta_{t-1}$, regardless of looseness in our bound. 
This is because as a termination condition for subproblem $t$, we only assume that the subproblem suboptimality gap does not exceed $\epsilon_t \Delta_{t-1}$.

\algname{} estimates $\Cprogress$ in the same way that the algorithm estimates $\Csolve$ and $\Csetup$---by solving for the appropriate parameter after iteration $t$ and taking the median over past estimates:
\[
\hCprogress_t = \tfrac{1}{(1 - \hat{\epsilon}_t) \xi_t} \left[ 1 - \tfrac{\Delta_{t}}{\Delta_{t-1}} \right] \, , \quad 
\quad \hat{\epsilon}_t = \tfrac{f_t(\z_t) - \ftlb(\xt)}{\Delta_{t-1}} \, , \]
\vspace{-1.3em}
\[
\text{and} \quad \Cprogress =  \mathrm{max} \left\{ 1,  \mathrm{median}( \hCprogress_{t-1}, \hCprogress_{t-2} )  \right\} \, .
\]
Here the $\mathrm{max}\{ 1, \cdot\}$ guarantees that $\Cprogress \geq 1$---this ensures the value of $\hat{\Delta}(\xi, \epsilon)$ is at most the bound predicted by \thmref{thm:approx_convergence_rate}.
In this case, we take the median of only the past two estimates for $\Cprogress$, allowing $\xi_t$ to change significantly between iterations if necessary (unlike $\Csetup$, for example, it is unclear whether we should expect $\Cprogress$ to be approximately constant for all $t$).

Having modeled both the time for iteration $t$ and the progress during iteration $t$, \algname{} combines $\hat{T}_t$ and $\hat{\Delta}_t$ to approximately optimize $\xi_t$ and $\epsilon_t$.
Specifically, \algname{} defines
\begin{equation} \label{eqn:linear_assumption}
\xi_t, \epsilon_t = \argmax{\xi, \epsilon} -\frac{\log(\hat{\Delta}_t(\xi, \epsilon) / \Delta_{t-1})}{\hat{T}_t(\xi, \epsilon)} \, . 
\end{equation}
With \eqnref{eqn:linear_assumption}, \algname{} values time as if the algorithm converges linearly.  That is, a subproblem that requires an additional second to solve should result in a $\Delta_t$ that is smaller by a multiplicative factor.

\algname{} solves \eqnref{eqn:linear_assumption} approximately with grid search, considering 125 candidates for $\xi_t$ and 10 candidates for $\epsilon_t$.  The candidates for $\xi_t$ span between $10^{-6}$ and $1$, while the candidates for $\epsilon_t$ span between $0.01$ and $0.7$.  Later in \secref{sec:empirical}, we examine some of these parameter values empirically.

We also enforce a time limit when solving each subproblem.
In addition to the termination conditions described in \secref{sec:approximate},
we also terminate subproblem $t$ if the threshold $\epsilon_t$ is not reached before a specified amount of time elapses.
We define the time limit as $\Csolve \mathrm{ProblemSize}(\xi_t) {\epsilon_t}^{-1}$, which is the estimated time  for solving the subproblem in \eqnref{eqn:time}.

\subsection{Relation to prior algorithms}



Many prior algorithms also exploit piecewise structure in convex problems.
The classic simplex algorithm \citep{Dantzig:1965}, for example, exploits redundant constraints in linear programs.

In the late 1990s and early 2000s, working set algorithms became important to machine learning for training support vector machines.
Using working sets, \cite{Osuna:1997} prioritized computation on training examples with suboptimal dual value.
\cite{Joachims:1999} improved the choice of working sets based on a first-order ``steepest feasible direction'' strategy---an idea that \cite{Zoutendijk:1970} originally proposed for constrained optimization.
To further reduce computation, \citeauthor{Joachims:1999} developed a ``shrinking'' heuristic, which freezes values of specific dual variables that satisfy a condition during several consecutive iterations.

Later works refined and extended these working set ideas.  \cite{Fan:2005} as well as \citet{Glasmachers:2006} used second-order information to improve working sets for kernelized SVMs.
Unlike \algname{}, these approaches apply only to working sets of size two, which is limiting but nevertheless practical for kernelized SVMs.  \citet{Zanghirati:2003} and \citet{Zanni:2006} considered larger working sets and parallel algorithms.
\cite{Tsochantaridis:2005} extended working set ideas to structured prediction problems.
\cite{Hsieh:2008} combined shrinking with dual coordinate ascent to train linear SVMs. This resulted in a very fast algorithm, and the popular \liblinear{} library \citep{Fan:2008} uses this approach to train linear SVMs today.

Similar coordinate descent strategies work well for training $\ell_1$-regularized models. \cite{Friedman:2010} proposed a fast algorithm that combines working sets with coordinate descent and a proximal Newton strategy.
Similarly, \cite{Yuan:2010} found that combining CD with shrinking heuristics leads to a fast algorithm for sparse logistic regression.  
Today \liblinear{} uses a refined version of this approach to train such models \citep{Yuan:2012}, which applies working sets and shrinking in a two-layer prioritization scheme.

These are just two of many algorithms that incorporate working sets to speed up sparse optimization.  
For lasso-type problems, many additional studies combine working set \citep{Scheinberg:2016,Massias:2017} or active set \citep{Wen:2012,Solntsev:2015,Neskar:2016} strategies with standard algorithms.
Researchers have also applied working sets to many other sparse problems---see e.g. 
\citep{Lee:2007,Bach:2008,Kim:2008,Roth:2008,Obozinski:2009,Friedman:2010,Schmidt:2010}.

More generally, prioritizing components of the objective continues to be an important idea for scaling model training. 
Many works consider importance sampling to speed up stochastic optimization
\citep{Needell:2014,Zhao:2015,Csiba:2015,Vainsencher:2015,Perekrestenko:2017,Stich2:2017}.
\cite{Harikandeh:2015}, \cite{Stich:2017}, and \cite{Johnson:2017:stingy} use alternative strategies to
improve first-order algorithms.



To our knowledge, \algname{} is the first working set algorithm that selects each working set in order to guarantee an arbitrarily large amount of progress during each iteration.
Prior algorithms choose working sets in intuitive ways, but there is little understanding of the resulting progress.
In contrast, our theory for \algname{} provides justification for the algorithm, avoids possible pathological scenarios, and inspires new ideas, such as our approach to tuning algorithmic parameters.


We note that because \algname{} can use any subproblem solver,
\algname{} could also use importance sampling, shrinking, or another strategy when solving each subproblem.

\section{BlitzScreen safe screening test} \label{sec:screening}



In this section, we introduce \screenname{}, a safe screening test that relates closely to \algname{}.
Like \algname{}, \screenname{} involves minimizing a relaxed objective instead of the original objective, $f$.
Unlike \algname{}, \screenname{} guarantees
that the relaxed objective and $f$ have the same minimizer.


\subsection{\screenname{} definition \label{sec:main_screening}}

\screenname{} requires three ingredients: 
\begin{enumerate}
\item An approximate minimizer of $f$, denoted $\y_0$, for which $f(\y_0)$ is finite.
\item A 1-strongly convex function, denoted $f_0$, that satisfies $f_0(\x) \leq f(\x)$ for all $\x$.
\item The minimizer of $f_0$, denoted $\x_0$.  \end{enumerate}
One way to construct such a $f_0$ uses a subgradient $\g_0 \in \partial f(\y_0)$.
Given such a point, we can define
\[
f_0(\x) = f(\y_0) + \ip{\g_0, \x - \y_0} + \oh \norm{\x - \y_0}^2 \, .
\]
We can easily compute $\x_0 = \mathrm{argmin}\ f_0(\x)$,
and since $f$ is $1$-strongly convex, $f_0$ lower bounds $f$.

Regardless of how we define $f_0$, we have the following screening result. 



\begin{restatable}[\screenname{} safe screening test]{thm}{screeningtheorem} \label{thm:screening}
Let $f_0$ be any $1$-strongly convex function that 
satisfies $f_0(\x) \leq f(\x)$ for all $\x$,
and let $\x_0 = \mathrm{argmin}\, f_0(\x)$.
Given any $\y_0 \ne \xstar$, define the suboptimality gap \mbox{$\Delta_0 = f(\y_0) - f_0(\x_0)$} as well as the ``safe region''
\[
\T_1 = \left\{ \x \ \big| \ \norms{\x - \oh (\x_0 + \y_0)} < \sqrt{\Delta_0 - \tfrac{1}4 \norm{\x_0 - \y_0}^2} \right\} \, .
\]
For any $i \in [m]$, define $k$ such that the subdomain $\Xik$ contains $\oh (\x_0 + \y_0)$.  Then if $\T_1 \subseteq \Xik$, we can safely replace $\phi_i$ with $\phi_i^{(k)}$ in $f$.
That is, for all $i \in [m]$, if we let
\[
\phis = \left\{
  \begin{array}{ll}
  \phi_i^{(k)} & \text{if}\ \T_1 \subseteq \Xik \, , \\
  \phi_i & \text{otherwise,}
\end{array} \right.
\]
then the ``screened objective''
$
\fs(\x) := \psi(\x) + \sum_{i=1}^m \phis(\x)
$
has the same minimizer as $f$.
\end{restatable}

We prove \thmref{thm:screening} in \fullappref{app:proof_screening}.
The proof relies on the equivalence of $\fs$ and $f$ within $\T_1$.  As long as $\fs(\x) = \f(\x)$ for all $\x \in \T_1$, these objectives have the same minimizer.
Note the safe region size greatly depends on the approximate solutions.
When $\Delta_0$ is large, $\T_1$ is large and $\phis = \phi_i$ for many $i$. 
If $\Delta_0$ is small, 
minimizing $\fs$ can be significantly simpler than minimizing $f$.

Applying \screenname{} requires checking whether $\T_1 \subseteq \Xik$ for each $i$. 
This condition is closely related to (C1) in \algname{}, and our remarks in \secref{sec:time_form_f_t} about testing (C1) also apply to screening. 
In many scenarios, we can evaluate whether $\T_1 \subseteq \Xik$ in $\mathcal{O}(n)$ time.

\subsection{Example: \screenname{} for $\ell_1$-regularized learning \label{sec:screen_l1}}

As an example, we apply \screenname{} to $\ell_1$-regularized loss minimization:
\begin{equation}
\textstyle \minimize{\omegab \in \reals^m}\, \flo(\omegab) := \sum_{j=1}^n L_j(\ip{\a_j, \omegab})  + \lambda \norm{\omegab}_1 \, .
\label{prb:l1_loss_2}
\tag{PL1}
\end{equation}
If each $L_j$ is $1$-smooth, we can transform the problem into its $1$-strongly convex dual:
\begin{equation}
\textstyle \minimize{\x \in \D} \flod(\x) := \sum_{j=1}^n L_j^*(x_j) + \sum_{i=1}^m \phi_i(\x)   \, .
\label{prb:l1_dual_2}
\tag{PL1D}
\end{equation}
Above, each implicit constraint defines $\phi_i(\x) = 0$ if $\abs{\ip{\A_i, \x}} \leq \lambda$ and $\phi_i(\x) = +\infty$ otherwise.
Successfully screening a constraint in \prbref{prb:l1_dual_2} corresponds to eliminating a feature from \prbref{prb:l1_loss_2}. 

To apply \screenname{}, we assume an approximate solution to \prbref{prb:l1_loss_2}, which we denote by $\omegab_0$.  Letting $L_j'(\cdot)$ represent the derivative of $L_j(\cdot)$, we define
\begin{equation}
\x_0 = 
\left[
L_1'(\ip{\a_j, \omegab_0}) ,
\ldots ,
L_n'(\ip{\a_j, \omegab_0}) 
\right]^T
\, , \quad \text{and} \quad
f_0(\x) = \oh \norm{\x - \x_0}^2 - \flo(\omegab_0) 
\, . \label{eqn:sz80}
\end{equation}
Using properties of duality, we show in \fullappref{app:screen_l1} that $f_0$ indeed lower bounds $\flod$.

\screenname{} also
requires a $\y_0 \in \D$ such that $\flod(\y_0)$ is finite, 
meaning $\y_0$ must satisfy all constraints. 
We define $\y_0$ by scaling $\x_0$ toward $\m{0}$ until this requirement is satisfied:



\begin{equation} \label{eqn:bbzzz}
\y_0 = \tfrac{\lambda}{\max{i \in [m]} \abs{\ip{\A_i, \x_0}}} \x_0 \, .
\end{equation}
We note there exist more advanced strategies for defining $\y_0$ \citep{Massias:2018},
but we do not consider such ideas in this work.
With \eqnref{eqn:bbzzz}, we have the following screening test for \prbref{prb:l1_loss_2}.
\begin{cor}[\screenname{} for \prbref{prb:l1_loss_2}] \label{screen:l1}
Given any $\omegab_0$ that does not solve \prbref{prb:l1_loss_2}, define $f_0$, $\x_0$, and $\y_0$ as in \eqnref{eqn:sz80} and \eqnref{eqn:bbzzz}.  
Define $\Delta_0 = \flod(\y_0) + \flo(\omegab_0)$.
For any $i \in [m]$, if
\[
{\lambda - \abs{\ips{\A_i, \oh (\x_0 + \y_0)}}} \geq {\norm{\A_i}} \sqrt{\Delta_0 - \tfrac{1}4 \norm{\x_0 - \y_0}^2 }  \, ,
\]
we can safely remove $\phi_i$ from \prbref{prb:l1_dual_2}, which implies that $\omega^\star_i = 0$ for all $\omegab^\star$ that solve \prbref{prb:l1_loss_2}.
\end{cor}

\subsection{Relation to prior screening tests}

\screenname{} improves upon prior screening tests in a few ways, which
we summarize as follows:

\begin{itemize}
\item \emph{More broadly applicable:}  
Prior works have derived separate screening tests for different objectives, including sparse regression \citep{Ghaoui:2010,Xiang:2012,Tibshirani:2012,Liu:2014,Wang:2015},
 sparse group \lasso{} \citep{Wang:2014}, as well as SVM and least absolute deviation problems \citep{Wang:Wonka:2014}.
Extending screening to each new objective requires substantial new derivations.
In contrast, \screenname{} applies in a unified way to all instances of our piecewise problem formulation.

Recently
\cite{Raj:2016} proposed a general recipe for deriving screening tests for different problems.
Unlike this approach, \screenname{} is an explicit screening test.

\item \emph{Adaptive:}
Before recently, most safe screening tests relied on knowledge of an exact solution to a related problem.
For example, \citet{Ghaoui:2010}'s test requires the solution to an identical problem but with greater regularization.
This is disadvantageous for a few reasons, one of which is that screening only applies as a preprocessing step prior to optimization.

Recent works have proposed \emph{adaptive} (also called ``dynamic'') safe screening tests \citep{Bonnefoy:2014,Bonnefoy:2015,Fercoq:2015,Johnson:2015,Ndiaye:2015,Zimmert:2015,Shibagaki:2016,Raj:2016,Ndiaye:2016,Ndiaye:2017}. 
Adaptive screening tests increasingly simplify the objective as the quality of the approximate solution improves. 
\screenname{} is an adaptive screening test.

\item \emph{More effective:}
Prior to \screenname{}, the ``gap safe sphere'' tests proposed by \citet{Fercoq:2015} 
were state-of-the-art adaptive screening tests,
as were the closely related tests proposed by \citet{Johnson:2015}, \citet{Zimmert:2015}, \citet{Shibagaki:2016}, \citet{Raj:2016} and \citet{Ndiaye:2017}.
Each of these tests applies to a different class of objectives, but they relate to \screenname{} in the same way.
With the exception of \citeauthor{Zimmert:2015}'s result (which is a special case of \screenname{} for SVM problems), we can recover these prior screening tests as special cases of \screenname{} but only by replacing $\T_1$ with a larger set.  Specifically, if we replace $\T_1$ in \thmref{thm:screening} with the larger ball
\begin{equation} \label{eqn:def_s_gap}
\T_{\mathrm{Gap}} = \left\{ \x \ \big| \ \norm{\x - \y_0} \leq \sqrt{2 \Delta_0} \right\} \, ,
\end{equation}
then the resulting theorem is a more general version of these existing tests.
The main difference is that $\T_{\mathrm{Gap}}$ is at least a factor $\sqrt{2}$ larger than the radius of $\T_1$.
As a result, \screenname{} is more effective at simplifying the objective.


\end{itemize}

\subsection{Relation to \algname{} \label{sec:screen_relat}}
We can view safe screening as a working set algorithm that converges in one iteration.
To solve ``subproblem 1,'' we minimize the screened objective, $f_S$. The subproblem solution also solves \prbref{prb:general}.

Our next theorem shows that
in the case of \screenname{} and \algname{},
this relation goes further:
\begin{restatable}[Relation between equivalence regions in \screenname{} and \algname{}]{thm}{thmSone} \label{thm:S1}
Given points $\x_0$ and $\y_0$, function $f_0$, and suboptimality gap $\Delta_0$ that satisfy the requirements for \thmref{thm:screening}, define the ball $\T_1$ as in \thmref{thm:screening}.
In addition, consider the equivalence region $\T_\xi$ from \secref{sec:blitz_lasso} with parameter choices $\xi_t = 1$, $\xtm = \x_0$, $\ytm = \y_0$, and $\Delta_{t-1} = \Delta_0$. Then 
\[
\T_1 = \T_\xi \, .
\]
\end{restatable}

We prove \thmref{thm:S1} in \fullappref{app:proof_S1}.
When $\xi_1 = 1$, using \algname{} is nearly equivalent to applying \screenname{}.
The only minor difference is that 
\algname{} may not simplify the objective as much as \screenname{}, since \screenname{} does not consider conditions analogous to (C2) and (C3).


Importantly, it is usually \emph{not} desirable for a working set algorithm to converge in one iteration.
Since screening tests only make ``safe'' simplifications to the objective, screening tests often simplify the problem only a modest amount.  In fact, unless a good approximate solution is already known, screening can fail to simplify the objective \emph{at all}.
We find it is usually better to simplify the objective aggressively, correcting erroneous choices later as needed.  This is precisely the working set approach.
As part of the next section,
we support this observation with empirical results.

\section{Empirical evaluation \label{sec:empirical}}

This section demonstrates the performance of \algname{} and \screenname{} in practice.

\subsection{Comparing the scalability of \algname{} and \screenname{} \label{sec:emp_scale}}

We first consider a group \lasso{} task and a linear SVM task. 
In each case, we examine how \algname{} and \screenname{} affect convergence times as the problem grows larger.
To our knowledge, such scalability tests are a novel contribution to research on safe screening.

\subsubsection{Scalability tests for group \lasso{} application}

For our first experiment, we consider the group \lasso{} objective \citep{Yuan:2006}:
\[
\ggl(\omegab) := \oh \norm{\A \omegab - \b}^2 + \lambda \sump \norm{ \omegab_{\mathcal{G}_i}}  .
\]
Here $\mathcal{G}_1, \ldots, \mathcal{G}_m$ are disjoint sets of feature indices such that $\cup_{i=1}^m \mathcal{G}_i = [q]$.  
Let $\omegab^\star \in \reals^q$ denote a minimizer of $\ggl$.
If $\lambda > 0$ is sufficiently large, then $\omegab^\star_{\G_i} = \m{0}$ for many $i$. 

We transform this problem into an instance of \prbref{prb:general} by considering the dual problem:
\begin{equation}
\begin{array}{cll}
\minimize{\x \in \D} & \fgl(\x) := \oh \norm{\x + \b}^2 - \oh \norm{\b}^2 &  \\
\mathrm{s.t.} & \norm{\A^T_{\G_i} \x} \leq \lambda & i = 1, \ldots, m \, .
\end{array}
\tag{PGD} \label{prb:group_dual}
\end{equation}
Each feature group corresponds to a constraint in the dual problem.
Constraints that do not determine the dual solution correspond to zero-valued groups in the primal solution.

We apply group \lasso{} to perform feature selection for a loan default prediction task.
Using data available from Lending Club,\footnote{URL: \url{https://www.kaggle.com/wendykan/lending-club-loan-data}.} we 
train a boosted decision tree model to
predict whether a loan will default during a given month.
We apply group \lasso{} to reduce the number of trees in the model. 
Features correspond to leaves in the tree model ($q\approx 3.0\e{4}$ features); groups correspond to trees ($m = 990$).
We generate $n = 4.8\e{5}$ training instances by passing data through the tree model, using the model's prediction values (sum of appropriate leaf weights) as training labels.
Since each tree maps each instance to one leaf, the feature matrices corresponding to each group are orthogonal.

There exist many algorithms for minimizing $\ggl$ \citep{Yuan:2006,Liu:2009,Kim:2010}.
Our implementation uses the block coordinate descent approach of \citet{Qin:2013}.
During an iteration, BCD updates weights in one group, 
keeping the remaining weights unchanged.
Following \citet{Qin:2013}, our implementation computes an optimal update to $\omegab_{\G_i}$ for roughly the cost of multiplying a dual vector $\x \in \reals^n$ by $\A_{\G_i}$.
Each update requires solving a 1-D optimization problem, which we solve with the bisection method.

We implement BCD in C++.
Using the same code base, we also implement the following:
\begin{itemize}
\item \emph{\algname{}:} To solve each subproblem, we use BCD.
\item \emph{\algname{} + \screenname{}:} After solving  each \algname{} subproblem, we apply \screenname{}.
\item \emph{BCD + \screenname{}:} After every five epochs, we apply \screenname{}.
\item \emph{BCD + gap safe screening:} After every five passes over the groups, we apply gap safe screening \citep{Ndiaye:2015}.
This implementation is identical to BCD + \screenname{} except we replace $\T_1$ in \screenname{} with the $\T_{\mathrm{Gap}}$ region defined in \eqnref{eqn:def_s_gap}.
\end{itemize}

\algname{} and the screening tests require checking if 
a region $\T = \{ \x \mid \norm{\x - \c} \leq r \}$ is a subset of
a set $\X_i^{(1)} = \{ \x \mid \norm{\A_{\G_i}^T \x} \leq \lambda \}$.
Computing this is nontrivial, so we apply 
relaxation ideas from \secref{sec:time_form_f_t} and \secref{sec:main_screening}. We define a set $\tilde{\X}_i^{(1)}$ such that $\tilde{\X}_i^{(1)} \subseteq \X_i^{(1)}$, and the algorithms test if $\T \subseteq \tilde{\X}_i^{(1)}$.
For each $i \in [m]$, we let $L_i = \max{k \in \G_i} \norm{\A_k}$ and define
$ \tilde{\X}_i^{(1)} = \{ \x \mid \norm{\A_{\G_i}^T \c} + L_i \norm{\x - \c} \leq \lambda \} $.

\begin{figure}[h]
\centering
\begin{small}
\begin{tabu}{@{}c@{\hspace{0.115in}}c@{\hspace{0.115in}}c@{}}
$m = 110,\ n = 4.8\e{5}$ & $m = 330,\ n = 4.8\e{5}$ & $m = 990,\ n = 4.8\e{5}$  \\
\includegraphics[width=1.90in]{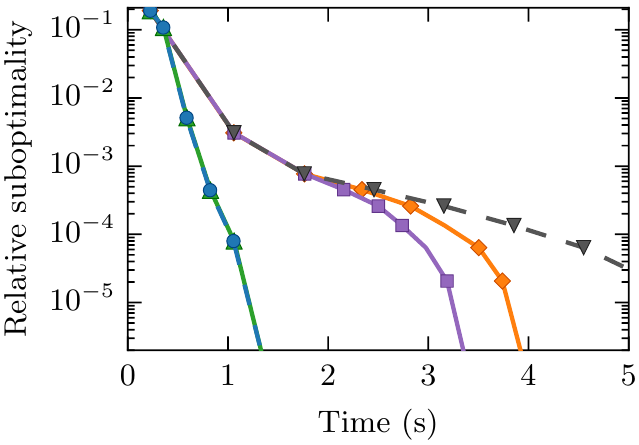} 
&
\includegraphics[width=1.90in]{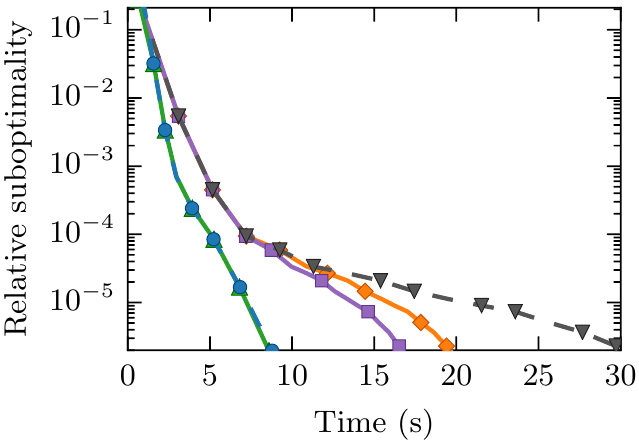} 
&
\includegraphics[width=1.90in]{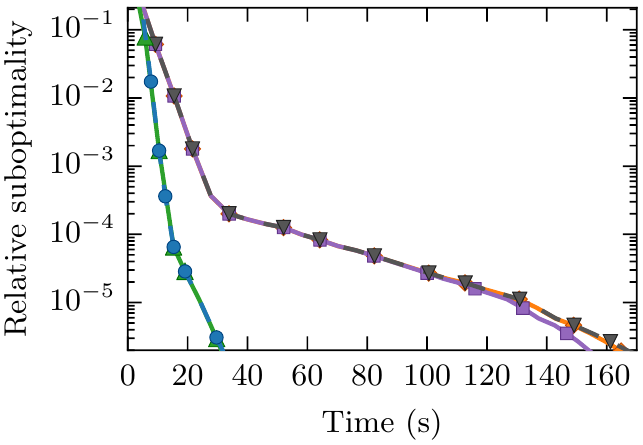} 
\\
\hline
\\[-0.95em]
$m = 110,\ n = 1.6\e{5}$ & $m = 330,\ n = 1.6\e{5}$ & $m = 990,\ n = 1.6\e{5}$  \\
\includegraphics[width=1.90in]{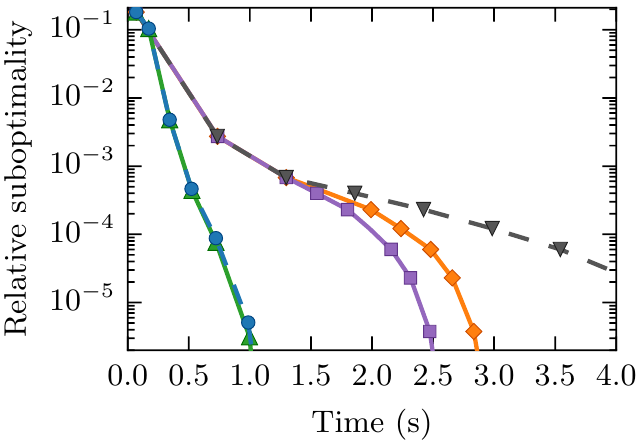} 
&
\includegraphics[width=1.90in]{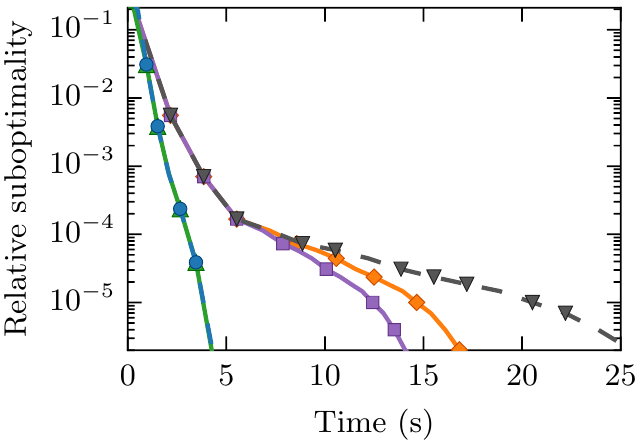} 
&
\includegraphics[width=1.90in]{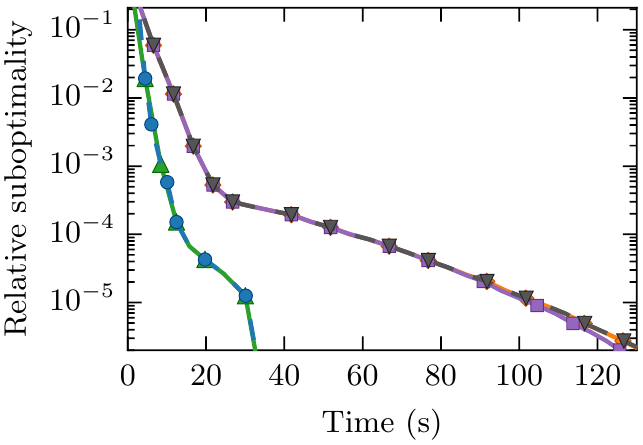} 
\\
\hline
\\[-0.95em]
$m = 110,\ n \approx 5.3\e{4}$ & $m = 330,\ n \approx 5.3\e{4}$ & $m = 990,\ n \approx 5.3\e{4}$  \\
\includegraphics[width=1.90in]{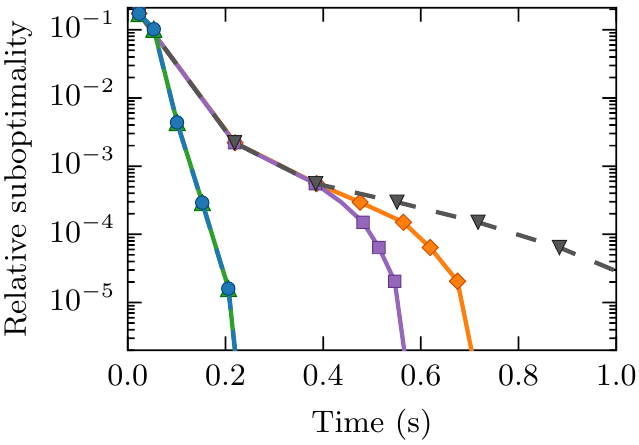} 
&
\includegraphics[width=1.90in]{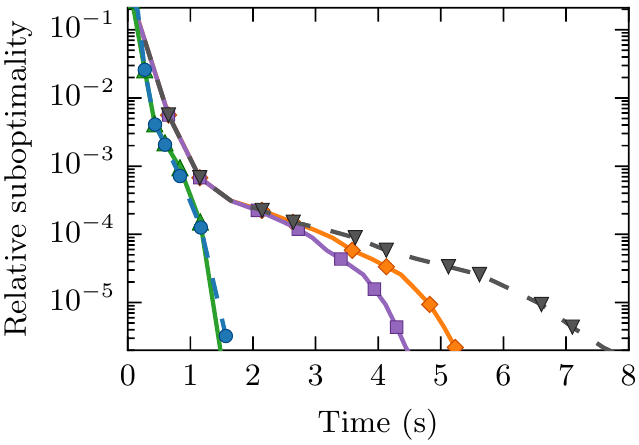} 
&
\includegraphics[width=1.90in]{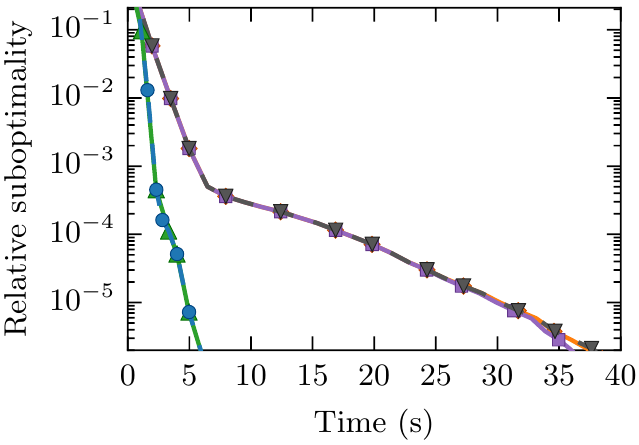} 
\\
\multicolumn{3}{c}{
\includegraphics[width=4.6in]{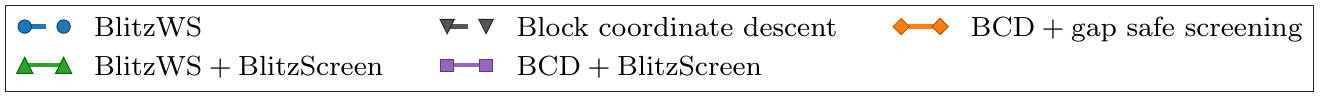} 
}
\end{tabu}
\end{small}
\caption{
\label{fig:group_lasso_comparisons}
\textbf{Scalability tests for group \lasso{}.}
From left to right, the number of groups increases from 110 to 330 and finally to 990.
From top to bottom, the number of instances decreases from 480k to 160k to 53.3k.
The impact of screening degrades as the number of groups increases, but \algname{} provides significant speed-ups in all cases.
Each \algname{} point represents 1 iteration; each BCD point represents 5 epochs.
}
\end{figure}

We perform data preprocessing to standardize groups in $\A$.  For each $i$, we scale $\A_{\G_i}$ so the variances of each column sum to one.
Our implementations include an unregularized bias variable.
We can easily accommodate this bias term by adding the constraint $\ip{\x, \m{1}} = 0$ to \prbref{prb:group_dual}.

To test the scalability of \algname{} and \screenname{}, we create nine smaller problems from the original group \lasso{} problem.  
We consider problems with $m = 990$, 330, and 110 groups by subsampling  groups uniformly without replacement.
We consider problems with $n = 480k$, $160k$, and $53.3k$ training examples by subsampling examples.
For each problem, we define $\lambda$ so that exactly 10\% of the groups have nonzero weight in the optimal model.

We evaluate performance using the relative suboptimality metric:
\[
\text{Relative suboptimality} = \frac{\ggl(\omegab_T) - \ggl(\omegab^\star)}{\ggl(\omegab^\star)}  \, .
\]
Here $\omegab_T$ is the weight vector at time $T$.
We take the optimal solution to be \algname{}'s solution after optimizing for twice the amount of time as displayed in each figure.

\figref{fig:group_lasso_comparisons} shows the results of these scalability tests.
Our first takeaway is that for this problem, the number of training examples does not greatly affect the impact of \algname{} and screening; as $n$ increases, the relative performance of each algorithm is remarkably consistent.
As the number of \emph{groups} increases, we observe a different trend.
When $m=110$, the screening tests provide some speed-up compared to BCD with no screening, particularly once the relative suboptimality reaches $6 \e{-4}$.
When $m = 990$, however, screening provides much less benefit.  In this case, despite being state-of-the-art for safe screening, \screenname{} has no impact on convergence progress until relative suboptimality reaches $10^{-5}$.

In contrast to safe screening, we find \algname{} achieves significant speed-ups compared to BCD, regardless of $m$.
We also note \screenname{} provides no benefit when combined with \algname{}.
This is because \algname{} already effectively prioritizes BCD updates.  

\subsubsection{Scalability tests for linear SVM application \label{sec:svm_scale_results}}

\begin{figure}[t]
\centering
\begin{small}
\begin{tabu}{@{}c@{\hspace{0.115in}}c@{\hspace{0.115in}}c@{}}
\includegraphics[width=1.90in]{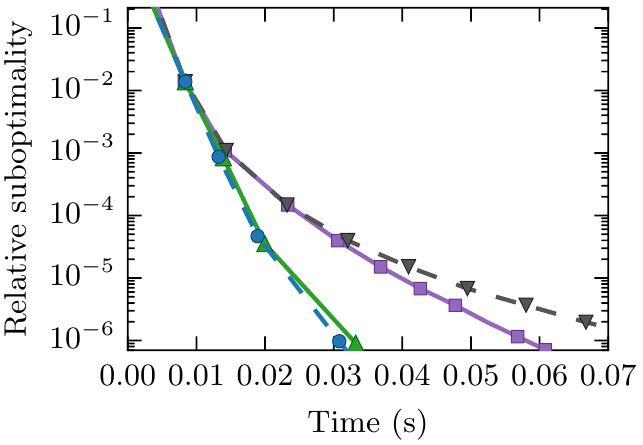} 
&
\includegraphics[width=1.90in]{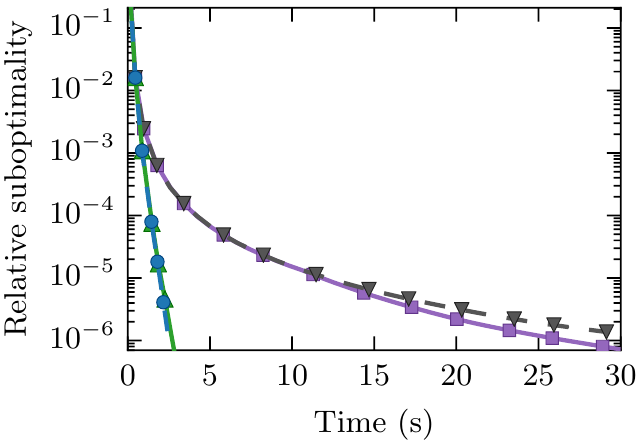} 
&
\includegraphics[width=1.90in]{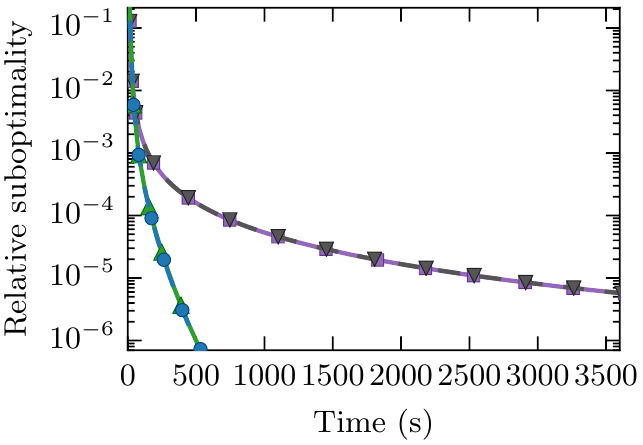} 
\\[-0.15em]
\multicolumn{3}{c}{
\includegraphics[width=4.6in]{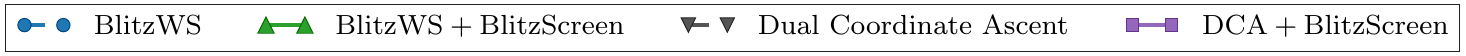} 
} 
\\[0.1em]
\includegraphics[width=1.90in]{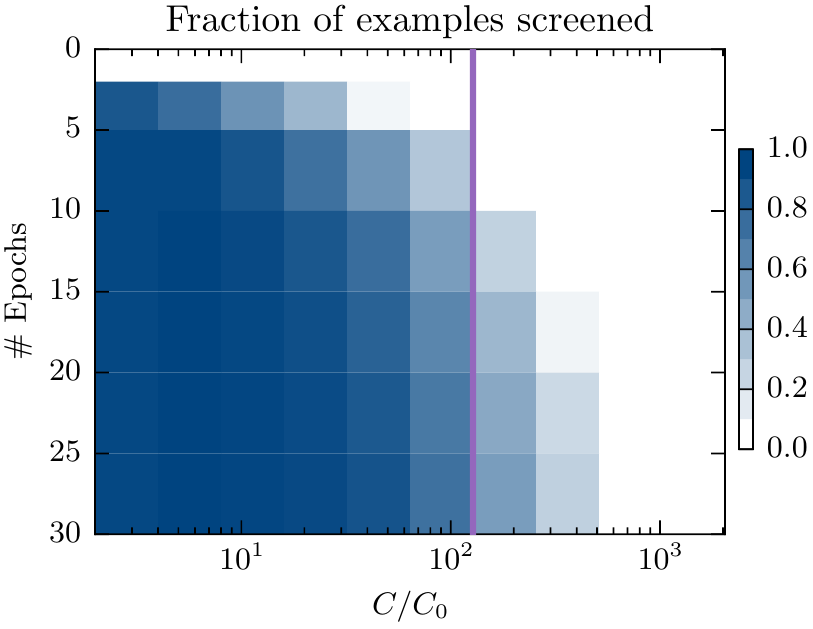} 
&
\includegraphics[width=1.90in]{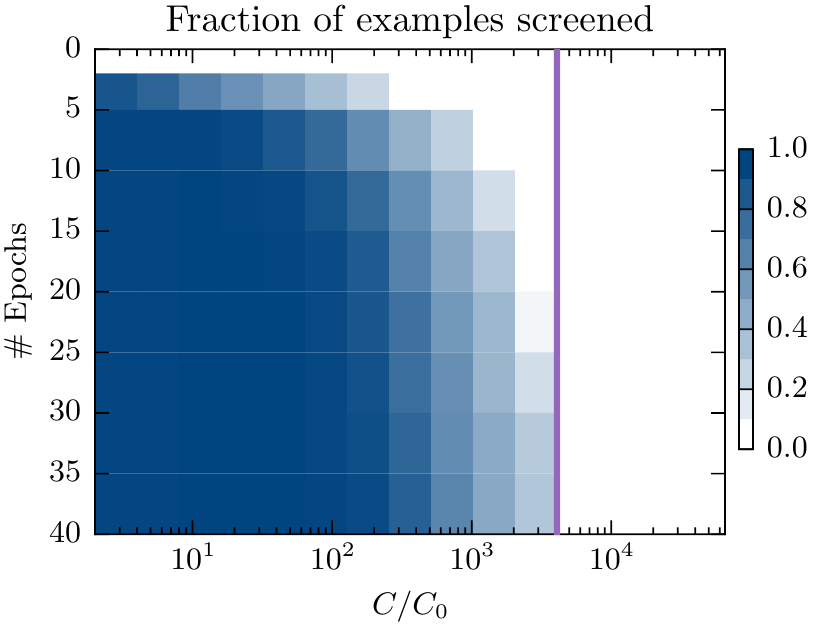} 
&
\includegraphics[width=1.90in]{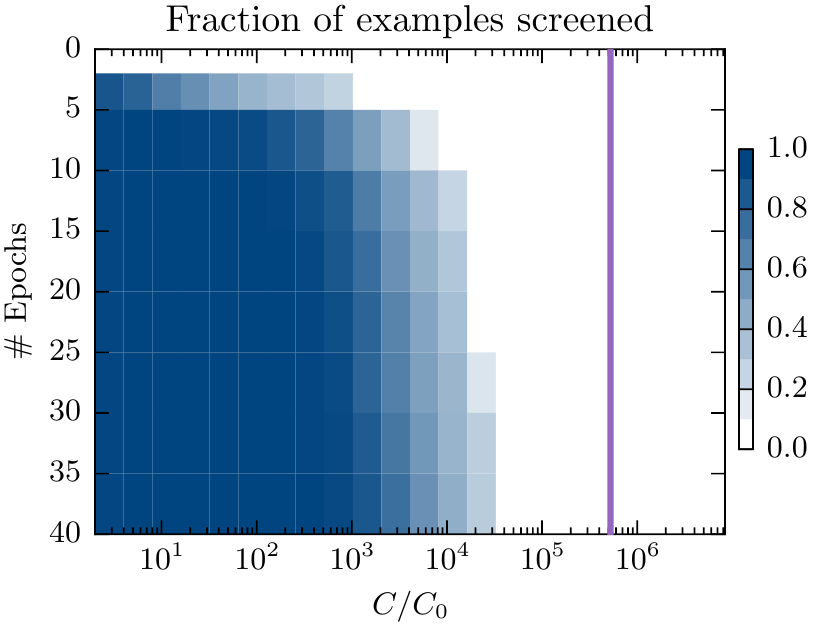} 
\\[-0.25em]
$m = 10^{4}$ & $m = 3.2 \e{5}$ & $m = 10^{7}$  
\end{tabu}
\end{small}
\caption{
\label{fig:svm_scalability}
\textbf{Scalability tests for linear SVM problem.}
From left to right, the number of training instances ($m$) increases for training a linear SVM. 
\emph{Above:} Relative suboptimality vs. time. 
\emph{Below:} Heat maps depicting the fraction of examples screened by \screenname{} when used with dual coordinate ascent. 
The purple vertical line indicates the $C$ chosen by five-fold cross validation.  We also use this value of $C$ for the above plots.
As the number of examples increases, screening becomes less useful when training with desirable values of $C$.
}
\end{figure}

We perform similar scalability tests for a linear SVM problem (\prbref{prb:l2_loss} with hinge loss).
We consider a physics prediction task involving the Higgs boson \citep{Bourdarios:2014}.
 We perform feature engineering using XGBoost \citep{Chen:2016},
which achieves good accuracy for this problem \citep{Chen:2014:HEP}. 
Using each leaf in the ensemble as a feature, the data set contains $n= 8010$ features and $m = 10^7$ examples.

There exist many algorithms for solving this problem \citep{Zhang:2004,Joachims:2006,ShalevShwartz:2007,teo:2010}.
We use dual coordinate ascent, which is simple and fast \citep{Hsieh:2008}.
We implement DCA in C++. Like the group \lasso{} comparisons, we implement \algname{} using the same code base. For each algorithm, we also implement \screenname{}.

By subsampling training instances without replacement,
we test the scalability of \algname{} and \screenname{} using
$m=10^7$, $3.2\e{5}$, and $10^4$ training instances.
For each problem and each algorithm, we plot relative suboptimality vs. time---here we measure relative suboptimality using the dual objective. 
We choose $C$ using five-fold cross validation.

We also test the performance of \screenname{} using a range of $C$ parameters.
We show these results using heatmaps, where 
the $y$-axis indicates epochs completed by DCA, and 
the $x$-axis indicates $C$.
The shading of the heat map depicts the fraction of training instances that \screenname{} screens successfully at each point in the algorithm.

\figref{fig:svm_scalability}
includes results from these comparisons.
Similar to the group \lasso{} case, we see \screenname{} provides some speed-up when $m$ is small.
But as $m$ increases, \screenname{} has no impact on convergence times until the relative suboptimality is much smaller.  In contrast, \algname{} provides improvements that, relative to the DCA solver, do not degrade as $m$ grows larger.

\subsection{Comparing \algname{} to \normalsize{L}\footnotesize{IB}\normalsize{L}\footnotesize{INEAR}}

\liblinear{} is one of the most popular and, to our knowledge, one of the fastest solvers
for sparse logistic regression and linear SVM problems.
Here we test how \algname{} compares. 

For sparse logistic regression, \liblinear{} uses working sets and shrinking to prioritize computation \citep{Yuan:2012}.
For linear SVM problems, \liblinear{} applies only shrinking \citep{Joachims:1999}. 
We can view shrinking as a working set algorithm that initializes the working set with all components (i.e., $\W_t = [m]$ and $f_t = f$); then while solving the subproblem, shrinking progressively removes elements from $\W_t$ using a heuristic. 

\subsubsection{Sparse logistic regression comparisons \label{sec:logreg_comparisons}}

Our \liblinear{} comparisons first consider sparse logistic regression (\prbref{prb:l1_loss}
with logistic loss).
There are many efficient algorithms for solving this problem \citep{Shalev-Shwartz:2009,Xiao:2010,Bradley:2011,Defazio:2014,Xiao:2014,Fercoq:2015:approx}.
To solve each subproblem, our \algname{} implementation uses an inexact proximal Newton algorithm (``ProxNewton'').  
We use coordinate descent to compute each proximal Newton step.
\liblinear{} uses the same ProxNewton strategy \citep{Yuan:2012}. 

We compare \algname{} with \liblinear{} version 2.11.
We compile \algname{} and \liblinear{} with GCC 4.8.4 and the {\tt -O3} optimization flag. We compare with two baselines: our ProxNewton subproblem solver (no working sets) and ProxNewton combined with \screenname{}.  
We perform screening as described in \secref{sec:screen_l1}
after each ProxNewton iteration.

We compare the algorithms using data from the LIBSVM data repository.\footnote{URL: \url{https://www.csie.ntu.edu.tw/~cjlin/libsvmtools/data sets/}.}
Tasks include spam detection \citep{Webspam:2006}, malicious URL identification \citep{Url:2009}, text classification \citep{Rcv1:2004}, and educational performance prediction \citep{Kdda:2010}.

We perform conventional preprocessing on each data set.
We standardize all features to have unit variance.  We remove features with fewer than ten nonzero entries.
We include an unregularized bias term in the model.  To accommodate this term, we add the constraint $\ip{\m{1}, \x} = 0$ to \prbref{prb:l1_dual}.
Since \liblinear{} implements an $\ell_1$-regularized bias term, we slightly modify \liblinear{} to 
(i)~use regularization 0 for the bias variable, and (ii)
always include the bias term in the working set.

\begin{figure}
\centering
\begin{small}
\begin{tabu}{@{}c@{\hspace{0.115in}}c@{\hspace{0.115in}}c@{}}
\multicolumn{3}{c}{{\tt webspam $\quad m \approx 4.4\e{5},\ n = 3.5\e{5},\ \nnz \approx 1.3\e{9}$}} \\
\lam{0.2} & \lam{0.02} & \lam{0.002} \\
\spar{6\e{-5}}{0.005} & \spar{9\e{-4}}{0.03} & \spar{0.006}{0.05}
\\
\includegraphics[width=1.90in]{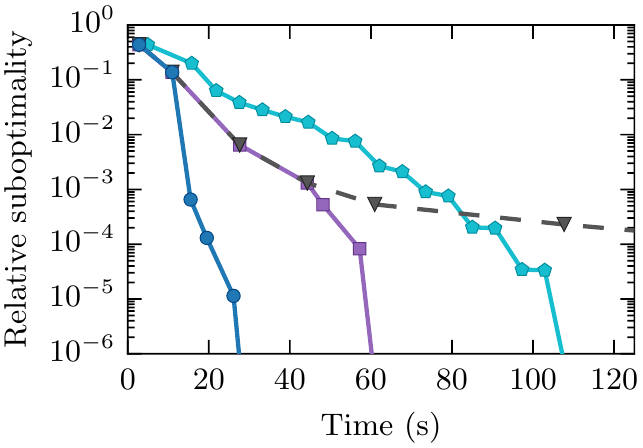} 
&
\includegraphics[width=1.90in]{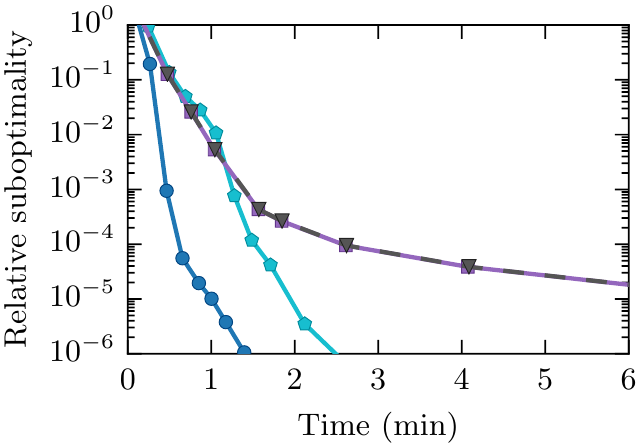} 
&
\includegraphics[width=1.90in]{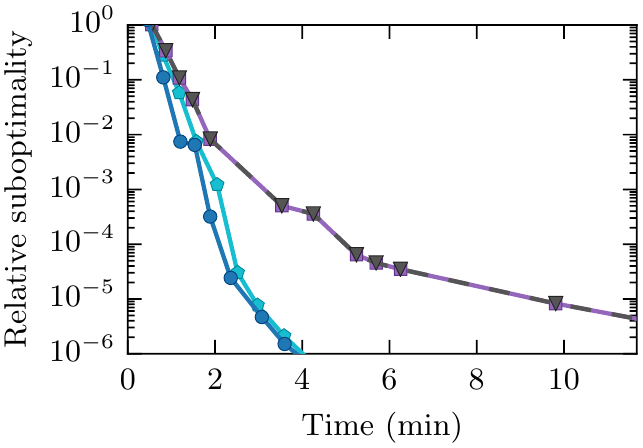} 
\\
\hline
\\[-0.99em]
\multicolumn{3}{c}{{\tt url} $\quad m \approx 1.5\e{5},\ n \approx 2.4\e{6},\ \nnz \approx 2.6\e{8}$} \\
\lam{0.2} & \lam{0.02} & \lam{0.002} \\
\spar{4\e{-5}}{0.03} & \spar{7\e{-4}}{0.09} & \spar{0.01}{0.2}
\\
\includegraphics[width=1.90in]{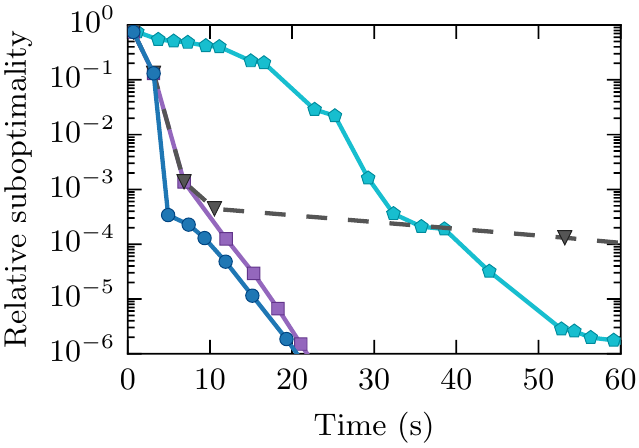} 
&
\includegraphics[width=1.90in]{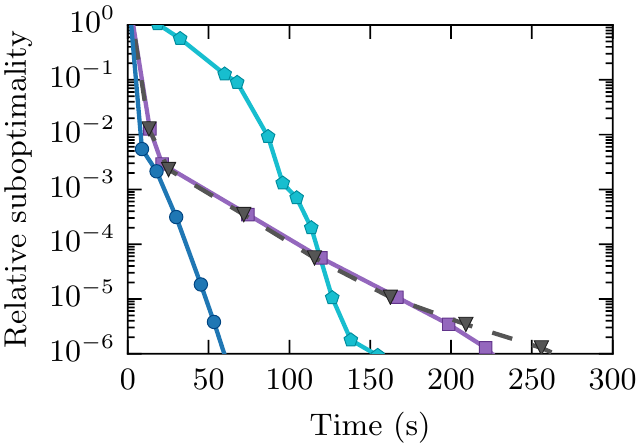} 
&
\includegraphics[width=1.90in]{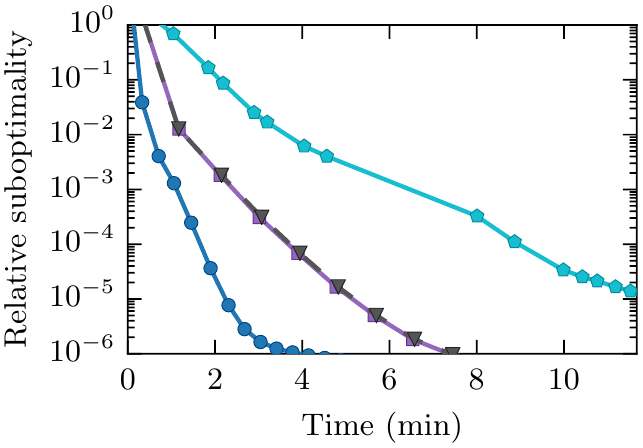} 
\\
\hline
\\[-0.99em]
\multicolumn{3}{c}{{\tt kdda $\quad m \approx 2.2\e{6},\ n \approx 8.4\e{6},\ \nnz \approx 2.8\e{8}$}} \\
\lam{0.2} & \lam{0.02} & \lam{0.002} \\
\spar{1\e{-5}}{0.01} & \spar{1\e{-3}}{0.1} & \spar{0.1}{0.3}
\\
\includegraphics[width=1.90in]{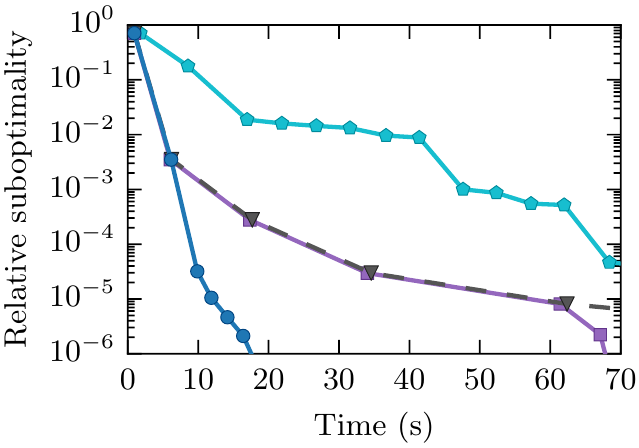} 
&
\includegraphics[width=1.90in]{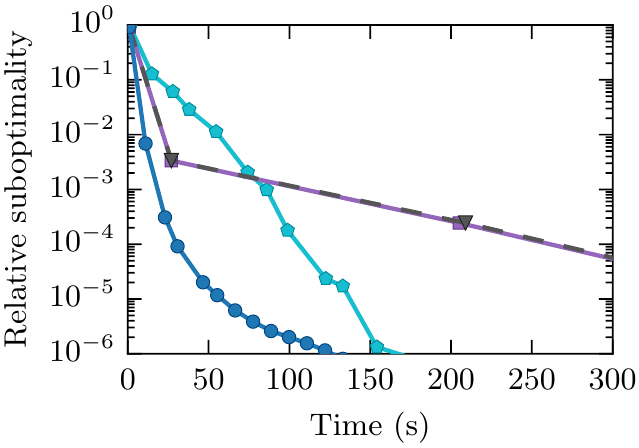} 
&
\includegraphics[width=1.90in]{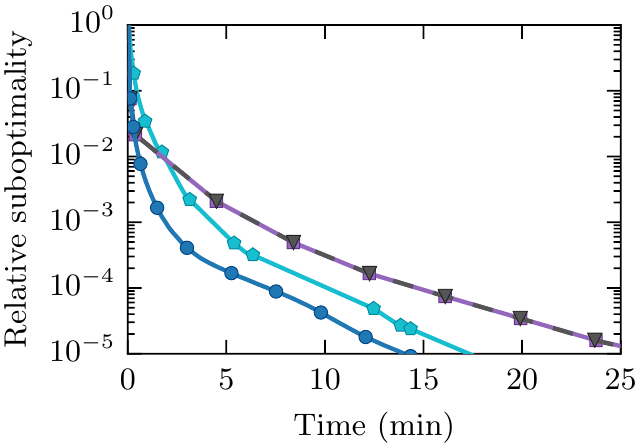} 
\\
\hline
\\[-0.99em]
\multicolumn{3}{c}{{\tt rcv1\_test} $\quad m \approx 3.4 \e{4},\ n \approx 6.8\e{5},\ \nnz \approx 5.0\e{7}$ } \\
\lam{0.2} & \lam{0.02} & \lam{0.002} \\
\spar{2\e{-3}}{0.09} & \spar{0.02}{0.3} & \spar{0.2}{0.6}
\\
\includegraphics[width=1.90in]{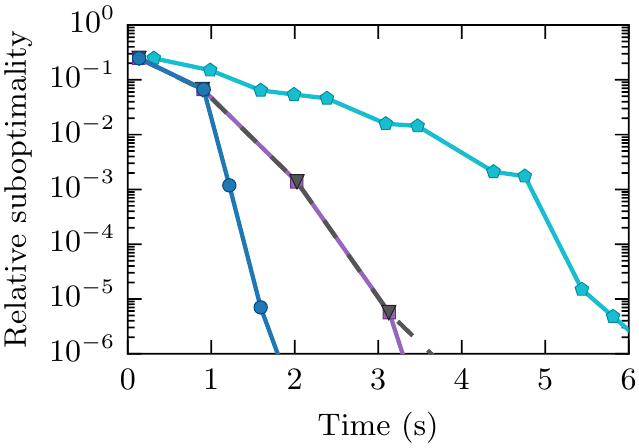} 
&
\includegraphics[width=1.90in]{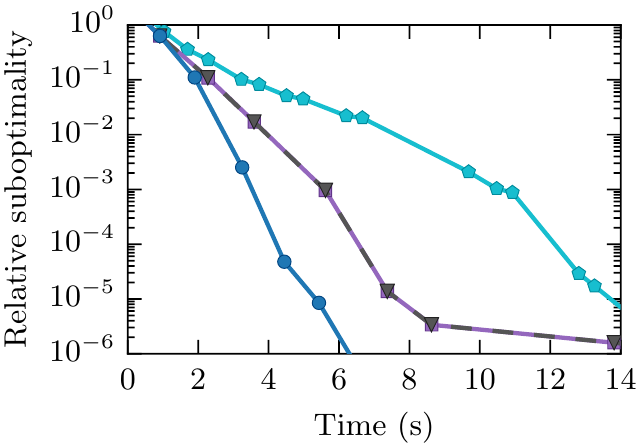} 
&
\includegraphics[width=1.90in]{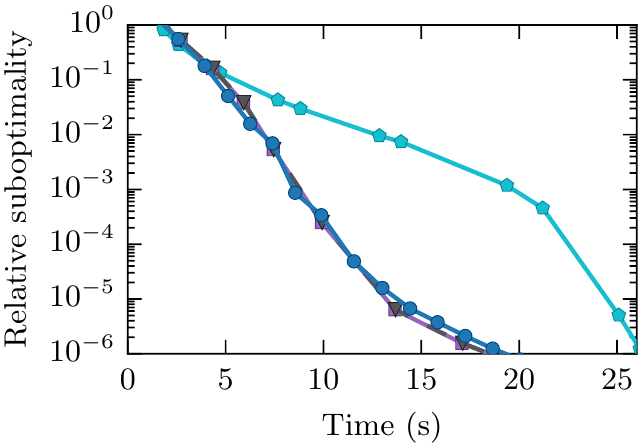} 

\\[-0.1em]
\multicolumn{3}{c}{
\includegraphics[width=4.4in]{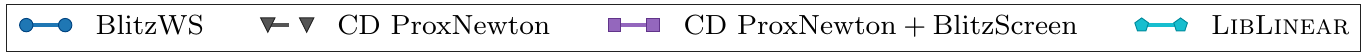} 
}
\end{tabu}
\end{small}
\caption{
\label{fig:logreg_comparisons}
\textbf{Convergence comparisons for sparse logistic regression.}
We compare \algname{} to its subproblem solver and \liblinear{}. 
\algname{} provides consistent optimization speed-ups.
}
\end{figure}

We solve each problem using three $\lambda$ values: $0.2 \lammax$, $ 0.02 \lammax$, and $ 0.002 \lammax$.
Here $\lammax$ is the smallest regularization value for which the problem's solution, $\omegab^\star$, equals $\m{0}$.
For each problem, we report the fraction of nonzero entries in $\omegab^\star$, which we denote by $s^\star$.
We also report a weighted version of this quantity, which we define as
$
s^\star_{\mathrm{W}} = \tfrac{1}{\nnz(\A)} \sum_{i \, : \, \omega_i^\star \ne 0} \nnz{(\A_i)} 
$.
Here $\nnz(\A_i)$ denotes the number of nonzero entries in column $i$ of the design matrix. 

With the exception of the spam detection problem, we solve each problem using a {\tt m4.2xlarge} Amazon EC2 instance with 2.3 GHz Intel Xeon E5-2686 processors, 46 MB cache, and 32 GB memory.
Due to memory requirements, we use a {\tt r3.2xlarge} instance with 61 GB memory and Intel Xeon E5-2670 processors for the spam detection problem.

\figref{fig:logreg_comparisons} contains the results of these comparisons.
In many cases, we see that \algname{} converges in much less time than \liblinear{}.
Considering that \liblinear{} is an efficient, established library, these results show that \algname{} is indeed a fast algorithm.

We also note that \algname{} provides significant speed-ups compared to the non-working set approach.
The amount of speed-up depends on the solution's sparsity, which is not surprising since we designed \algname{} to exploit the solution's sparsity. 


\subsubsection{Adaptation to regularization strength}


We find \algname{} outperforms \screenname{} because \algname{} adapts its $\xi_t$ progress parameter to each problem.
In contrast,
ProxNewton $+$ \screenname{} is approximately equivalent to using \algname{} with $\xi_t= 1$ for all iterations (as discussed in \secref{sec:screen_relat}).  
\figref{fig:xi_epsilon_values} contains plots of \algname{}'s chosen $\xi_t$ parameters for each logistic regression problem.
When $\lambda = 0.2 \lammax$, \algname{} uses large $\xi_t$ values, and screening (i.e., $\xi_t = 1$) also tends to perform well.
As $\lambda$ decreases,
screening becomes ineffective, while \algname{} adapts by choosing smaller values of $\xi_t$.

\begin{figure}
\centering
\begin{small}
\begin{tabu}{@{}c@{\hspace{0.075in}}c@{\hspace{0.075in}}c@{\hspace{0.075in}}c@{}}
\hspace{0.2in} {\tt webspam} & 
\hspace{0.23in} {\tt url} & 
\hspace{0.22in} {\tt kdda} & 
\hspace{0.18in} {\tt rcv1\_test} 
\\
\includegraphics[width=1.425in]{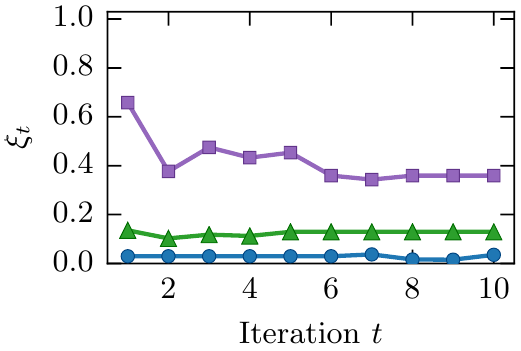} 
&
\includegraphics[width=1.425in]{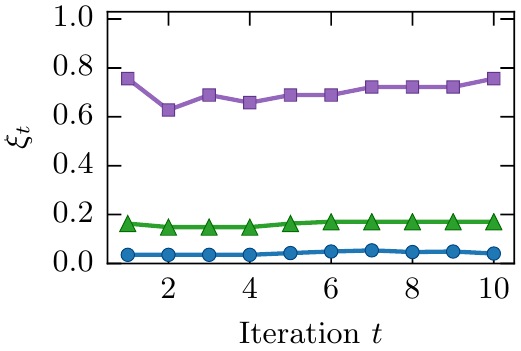} 
&
\includegraphics[width=1.425in]{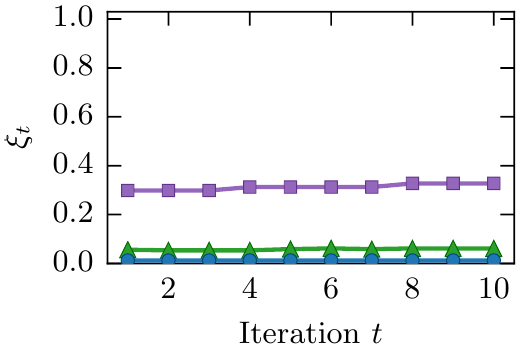} 
&
\includegraphics[width=1.425in]{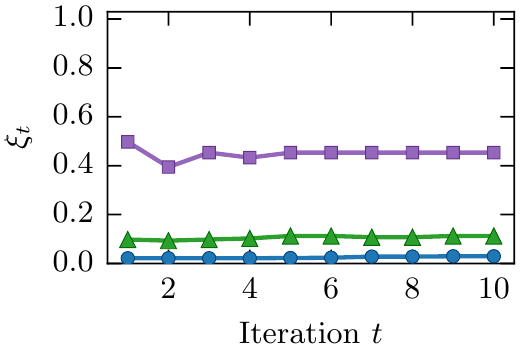} 
\\
\multicolumn{4}{c}{
\includegraphics[width=2.8in]{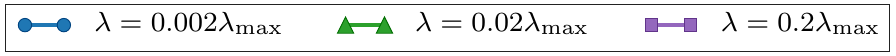} 
}
\end{tabu}
\end{small}
\caption{
\label{fig:xi_epsilon_values}
\textbf{\algname{} progress parameters for sparse logistic regression.}
Plots show $\xi_t$ values 
that \algname{} uses to produce the results in \figref{fig:logreg_comparisons}.
As regularization decreases, \algname{} adapts by decreasing $\xi_t$. 
}
\end{figure}



\subsubsection{Impact of capsule approximation}

\begin{figure}[t]
\centering
\begin{small}
\begin{tabu}{@{}c@{\hspace{0.075in}}c@{\hspace{0.075in}}c@{\hspace{0.075in}}c@{\hspace{0.075in}}c@{}}
&
\hspace{0.16in} \texttt{webspam} & 
\hspace{0.19in} \texttt{url} & 
\hspace{0.18in} \texttt{kdda} & 
\hspace{0.14in} \texttt{rcv1\_test} 
\\
\raisebox{0.545in}[0pt][0pt]{\rotatebox[origin=c]{90}{\footnotesize{$\lambda = 0.2 \lammax$}}} &
\includegraphics[width=1.350in]{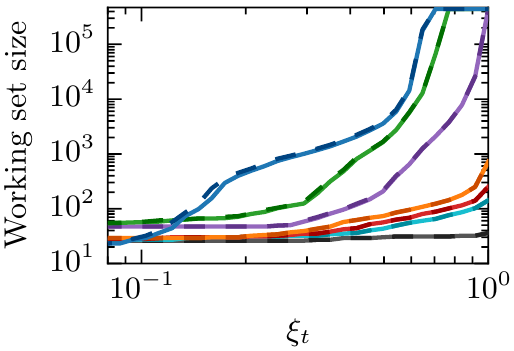} 
&
\includegraphics[width=1.350in]{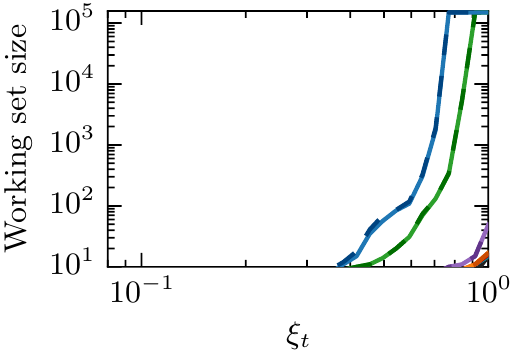} 
&
\includegraphics[width=1.350in]{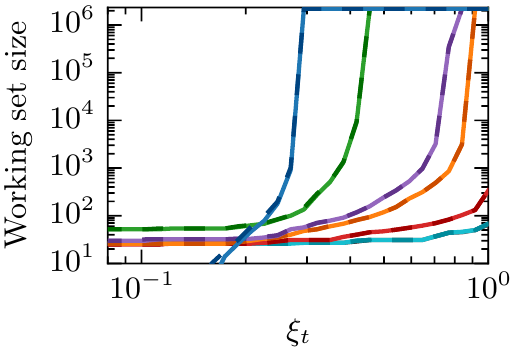} 
&
\includegraphics[width=1.350in]{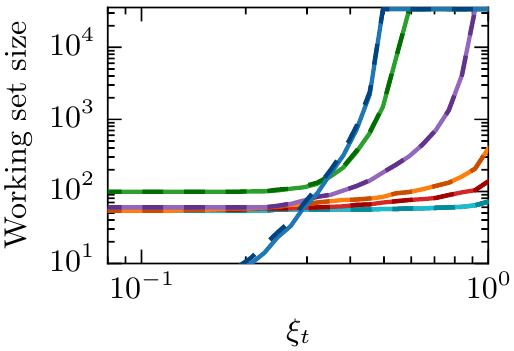} 
\\
\raisebox{0.54in}[0pt][0pt]{\rotatebox[origin=c]{90}{\footnotesize{$\lambda = 0.02 \lammax$}}} &
\includegraphics[width=1.350in]{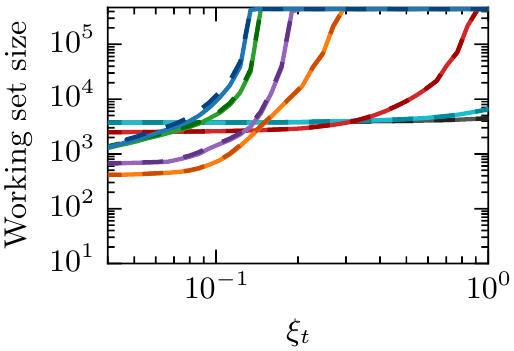} 
&
\includegraphics[width=1.350in]{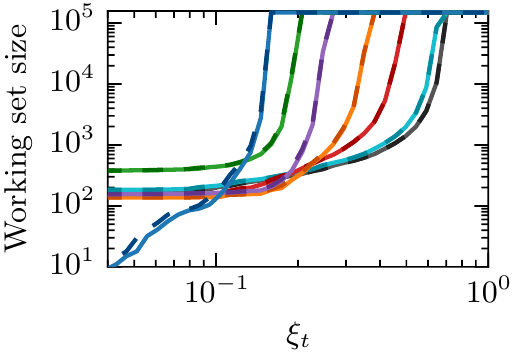} 
&
\includegraphics[width=1.350in]{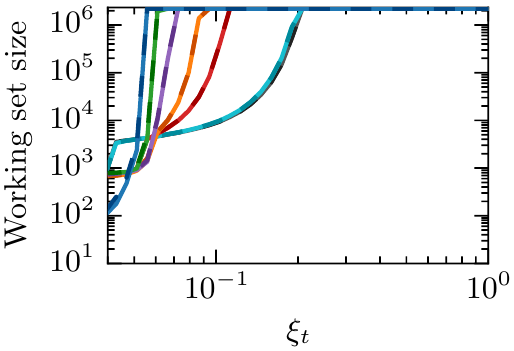} 
&
\includegraphics[width=1.350in]{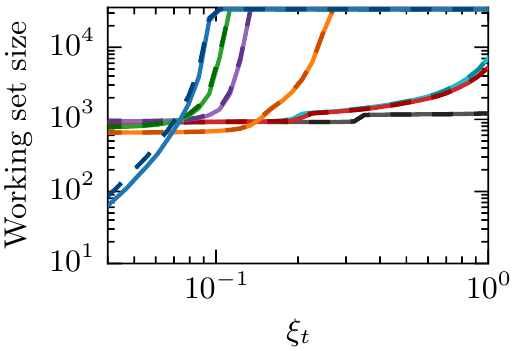} 
\\
\raisebox{0.53in}[0pt][0pt]{\rotatebox[origin=c]{90}{\footnotesize{$\lambda = 0.002 \lammax$}}} &
\includegraphics[width=1.350in]{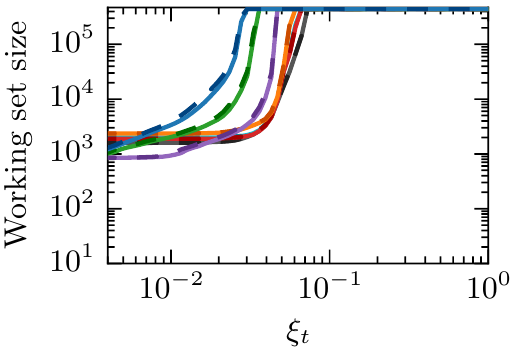} 
&
\includegraphics[width=1.350in]{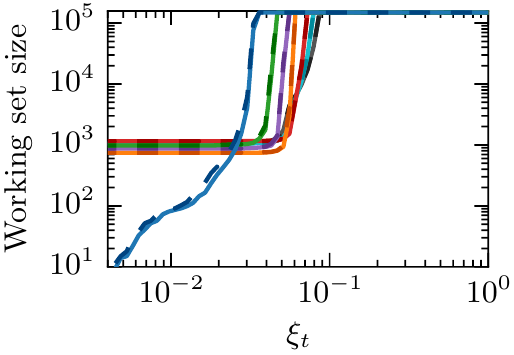} 
&
\includegraphics[width=1.350in]{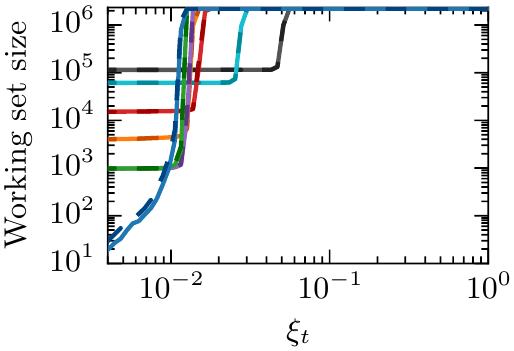} 
&
\includegraphics[width=1.350in]{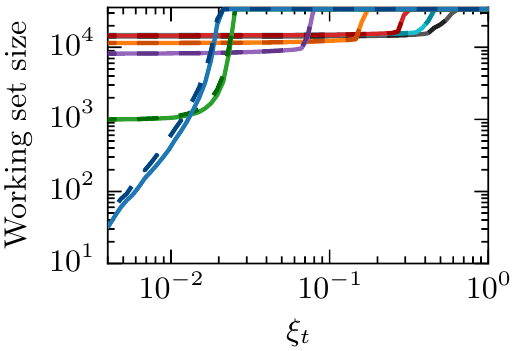} 
\\
\multicolumn{5}{c}{
\includegraphics[width=0.975\textwidth]{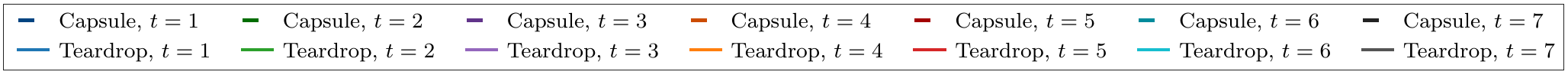} 
}
\end{tabu}
\end{small}
\caption{
\label{fig:capsule_ratio}
\textbf{Impact of \algname{}'s capsule approximation.}
We plot the working set size vs. possible choices of the $\xi_t$ progress parameter (dashed curves). 
Each of \algname{}'s first seven iterations corresponds to a different colored curve.
We also plot the working set size when using the teardrop region, $\T_\xi$, to select each working set (solid curves).
The close alignment of curves indicate the capsule approximation performs well.
}
\end{figure}


For the sparse logistic regression problems, we examine 
how \algname{}'s capsule approximation affects each working set.
We log \algname{}'s state---$\xtm$, $\ytm$, and $\Delta_{t-1}$---prior to selecting each working set.  Then offline, we compute working sets for many values of $\xi_t$.

We record working set sizes
for each problem, iteration, and $\xi_t$ value.
  In each case, we construct one working set using $\T_\xi$ and a second working set using the capsule approximation.
To calculate the working set using $\T_\xi$, we discretize the definition of this set using 200 values of $\beta$.

We measure the capsule approximation's impact by comparing the sizes of $\Wxi$ and $\Wcap$,
where the teardrop determines $\Wxi$ and the capsule determines $\Wcap$.
If $|\Wcap| \approx \abs{\Wxi}$, then \algname{}'s capsule approximation has little impact on the makeup of each working set.

\figref{fig:capsule_ratio} contains the results of this experiment.
Observe that in all cases, $|\Wcap| \approx \abs{\Wxi}$.
This suggests that $\Tcap$ is a very good approximation of $\T_\xi$.

\subsubsection{Linear SVM comparisons}

\begin{figure}
\centering
\begin{small}
\begin{tabu}{@{}c@{\hspace{0.115in}}c@{\hspace{0.115in}}c@{}}
\multicolumn{3}{c}{{\tt webspam $\quad m = 2.8\e{5},\ n \approx 4.4\e{5},\ \nnz \approx 1.0\e{9}$}} \\
\dcsvm{0.1}, \svmspar{0.06} & \dcsvm{}, \svmspar{0.06} & \dcsvm{10}, \svmspar{0.05} \\
\includegraphics[width=1.90in]{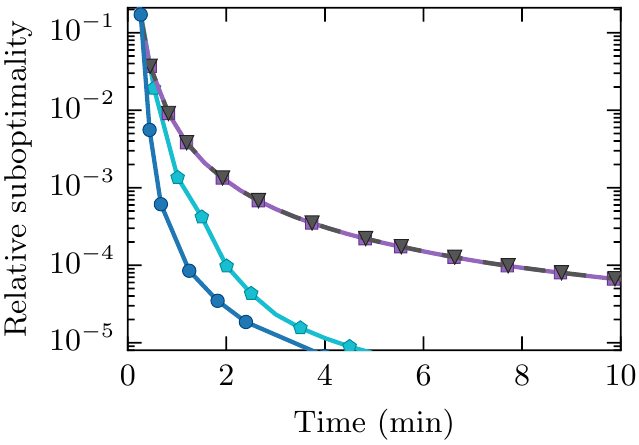} 
&
\includegraphics[width=1.90in]{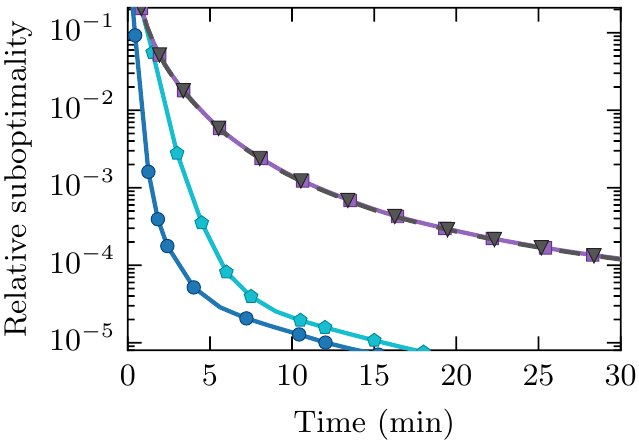} 
&
\includegraphics[width=1.90in]{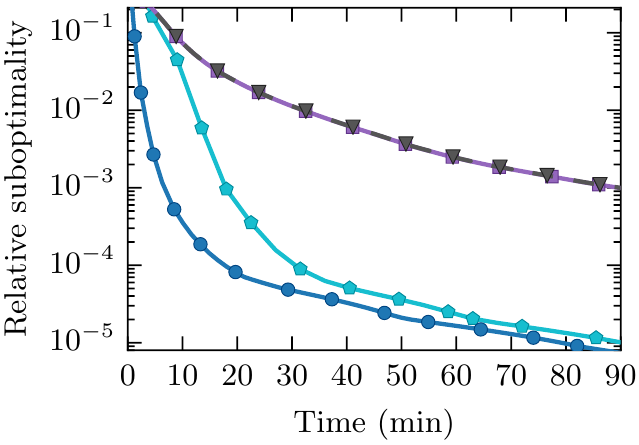} 
\\
\hline
\\[-0.99em]
\multicolumn{3}{c}{{\tt url $\quad m \approx 1.9\e{6},\ n \approx 1.5\e{5},\ \nnz \approx 2.1\e{8}$}} \\
\dcsvm{0.1}, \svmspar{0.02} & \dcsvm{}, \svmspar{0.02} & \dcsvm{10}, \svmspar{0.02} \\
\includegraphics[width=1.90in]{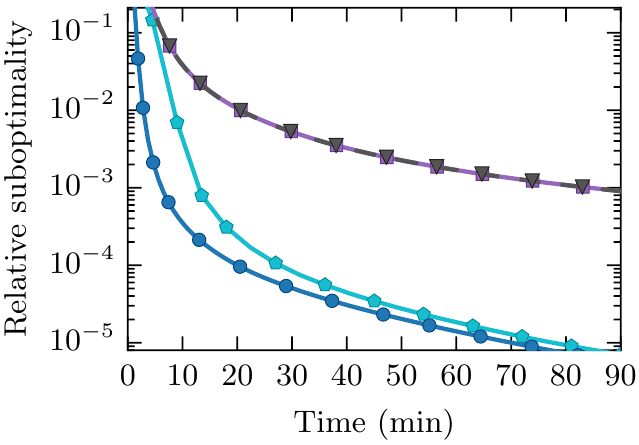} 
&
\includegraphics[width=1.90in]{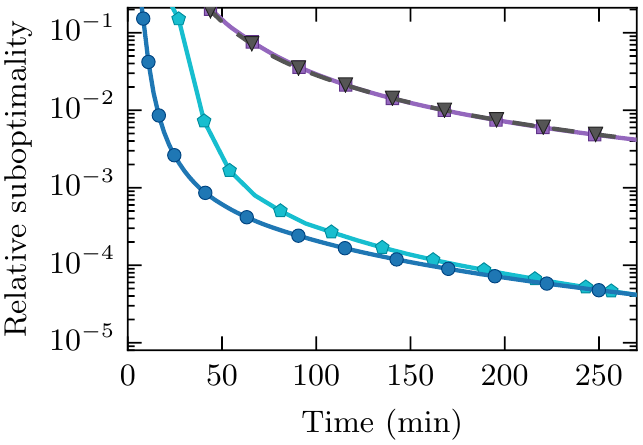} 
&
\includegraphics[width=1.90in]{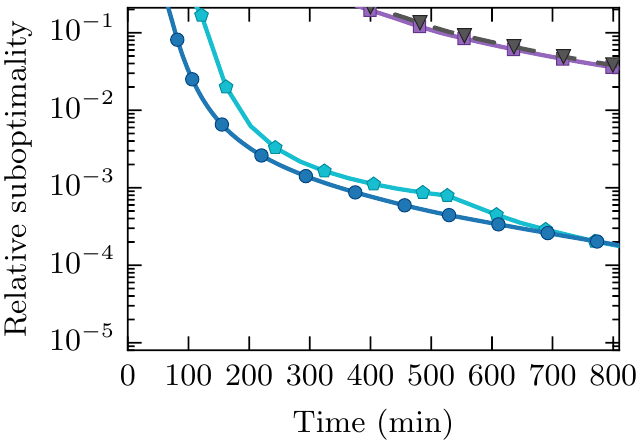} 
\\
\hline
\\[-0.99em]
\multicolumn{3}{c}{{\tt kdda $\quad m \approx 6.7\e{6},\ n \approx 2.2\e{6},\ \nnz \approx 2.2\e{8}$}} \\
\dcsvm{0.1}, \svmspar{0.1} & \dcsvm{}, \svmspar{0.1} & \dcsvm{10}, \svmspar{0.1} \\
\includegraphics[width=1.90in]{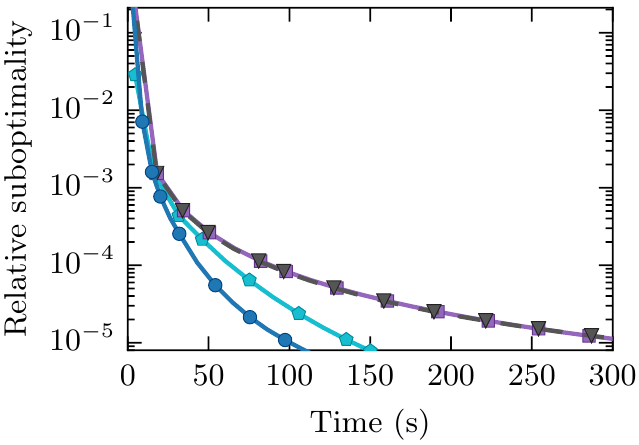} 
&
\includegraphics[width=1.90in]{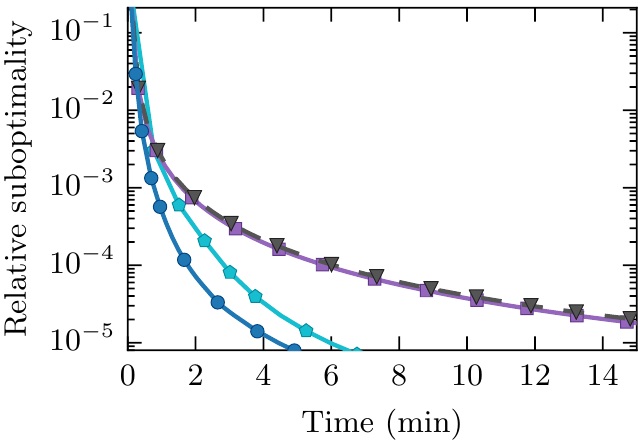} 
&
\includegraphics[width=1.90in]{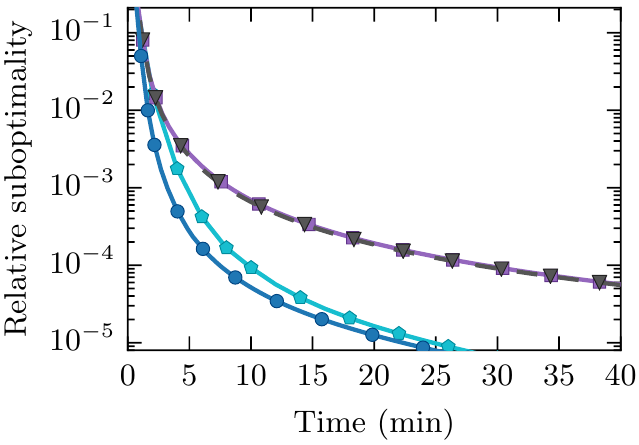} 
\\
\hline
\\[-0.99em]
\multicolumn{3}{c}{{\tt rcv1\_test} $\quad m \approx 5.4 \e{5},\ n \approx 3.3\e{4},\ \nnz \approx 4.0\e{7}$ } \\
\dcsvm{0.1}, \svmspar{0.03} & \dcsvm{}, \svmspar{0.04} & \dcsvm{10}, \svmspar{0.05} \\
\includegraphics[width=1.90in]{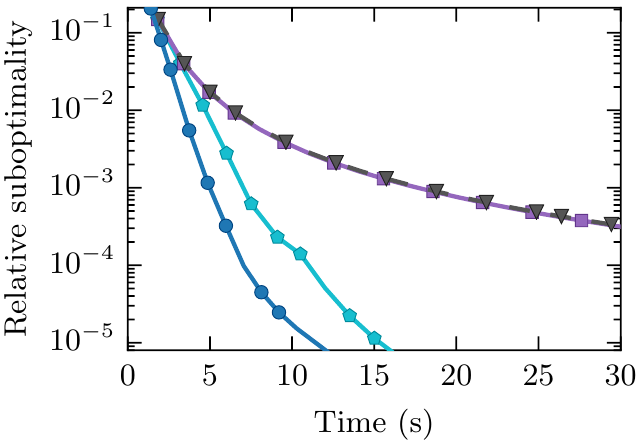} 
&
\includegraphics[width=1.90in]{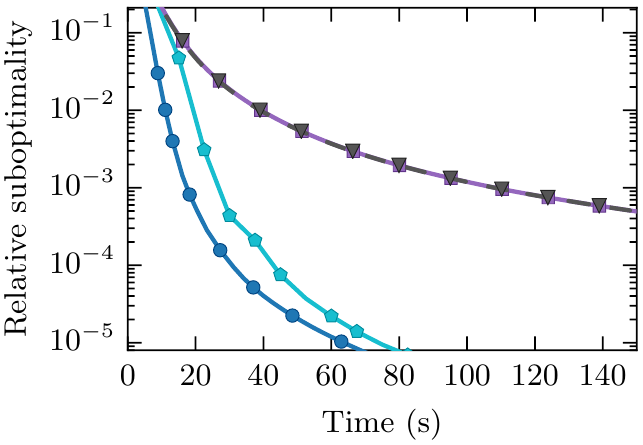} 
&
\includegraphics[width=1.90in]{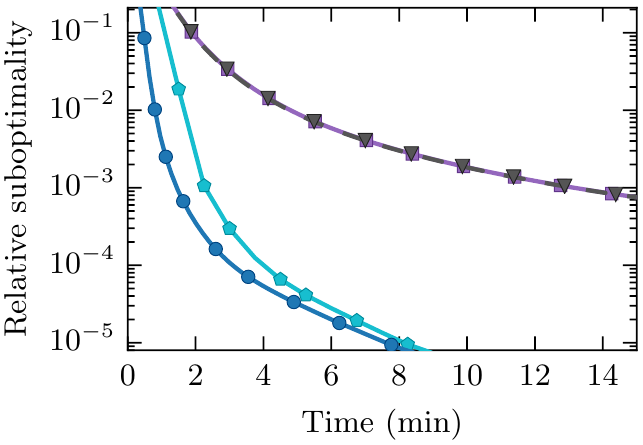} 
\\
\multicolumn{3}{c}{
\includegraphics[width=4.4in]{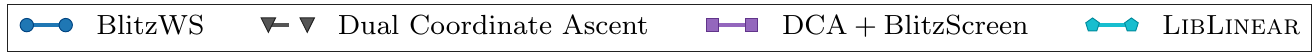} 
}
\end{tabu}
\end{small}
\caption{
\label{fig:svm_comparisons}
\textbf{Convergence comparisons for linear SVMs.}
\algname{} also leads to convergence time improvements when training linear SVMs. 
For more difficult problems, plot markers represent to multiple iterations.
}
\end{figure}

We also compare \algname{} with \liblinear{} for training linear SVMs.
The \algname{} implementation is the same as described in \secref{sec:svm_scale_results}. \liblinear{} also uses a DCA-based algorithm.

For these comparisons, we use the same data sets, compilation settings, and hardware as we used in \secref{sec:logreg_comparisons}.
For each data set, we compute a practical value of $C$ using five-fold cross validation, which we denote by $\Csvm$.
We compare using three values of $C$:
$0.1 \Csvm{}$, $\Csvm$, and $10 \Csvm$.
We also report the solution's ``sparsity,'' denoted $s^\star$, which we define as the fraction of training examples that are unbounded support vectors at the solution. 

\figref{fig:svm_comparisons} includes results from these comparisons.
\algname{} consistently provides speed-up compared to \liblinear{}, often during early iterations.


\section{Discussion \label{sec:discussion}}

We proposed \algname{}, a principled yet practical working set algorithm.
Unlike prior algorithms, \algname{} selects subproblems in a way that maximizes guaranteed progress.
We also analyzed the consequences of solving \algname{}'s subproblems approximately, and we applied this understanding to adapt algorithmic parameters as iterations progress.

In practice, \algname{} is indeed a fast algorithm.
Compared to the popular \liblinear{} library, \algname{} achieves very competitive convergence times. 
Another appealing quality of \algname{} is its capability of solving a variety of problems. This includes constrained problems, sparse problems, and piecewise loss problems.
This flexibility results from \secref{sec:formulation}'s novel piecewise problem formulation.  
We find this formulation is a useful way of thinking about sparsity and related structure in optimization.

We also proposed a state-of-the-art safe screening test called \screenname{}.
Unlike prior screening tests, \screenname{} applies to a large class of problems.
Because of its relatively small safe region, \screenname{} also simplifies the objective by a greater amount.
Unfortunately, we found that in many practical scenarios, \screenname{} had little impact on the algorithm's progress.  
While disappointing,
we think this observation is an important contribution.

Exploiting piecewise structure can lead to large optimization speed-ups.
Our analysis of \algname{} and \screenname{} provides a foundation for exploiting this structure in a principled way.  
We hope these contributions may serve as a starting point for future approaches to scalable optimization.




\section*{Acknowledgments}

Research grants that supported this work include PECASE N00014-13-1-0023, NSF IIS-1258741, and the TerraSwarm Research Center 00008169.
The Carnegie Learning and DataShop originally provided the educational performance data set ({\tt kdda}).

\bibliography{references}

\renewcommand{\theHsection}{A\arabic{section}}
\appendix

\section{Proof of \lemref{lem:blitz_lasso_lemma} \label{app:proof_blitz_lasso_lemma}}


\begin{proof}
We start with the definition of $\y_t$: 
\begin{align}
\Delta_t &= \fld(\y_t) - \fld(\x_t) \nonumber \\ 
&= \fld(\alpha_t \x_t + (1- \alpha_t) \y_{t-1}) - \fld(\x_t) \nonumber \\
&= (1 - \alpha_t) \left[ \Delta_{t-1} - \oh \alpha_t \norm{\x_t - \y_{t-1}}^2 - [\fld(\x_t) -\fld(\x_{t-1})] \right] \, . \label{eqn:xx7x}
\end{align}
Since $\fld$ is $1$-strongly convex, 
\begin{align}
& \fld(\x_t) \geq \fld(\x_{t-1}) + \ip{\nabla \fld(\x_{t-1}), \x_t - \x_{t-1}} + \oh \norm{\x_t - \x_{t-1}}^2 \nonumber \\
 \Rightarrow \ & \fld(\x_t) - \fld(\x_{t-1}) \geq \oh \norm{\x_t - \x_{t-1}}^2 \, . \label{eqn:sjasdf}
\end{align}
Above, we have used the fact that $\ip{ \nabla \fld(\x_{t-1}), \x_t - \x_{t-1}} \geq 0$, which must be true because $\fld(\xtm) \leq \fld(\xt)$.
Combining \eqnref{eqn:sjasdf} with \eqnref{eqn:xx7x}, we have
\begin{equation}
\Delta_t \leq (1 - \alpha_t) \left[ \Delta_{t-1} - \oh \alpha_t \norm{\x_t - \y_{t-1}}^2 - \oh \norm{\x_t - \x_{t-1}}^2 \right] \, . \label{eqn:m18z}
\end{equation}
Next, we use the algebraic fact
\begin{equation}
\alpha_t \norm{\x_t - \y_{t-1}}^2 + \norm{\x_t - \x_{t-1}}^2 = (1 + \alpha_t) \norm{\x_t - \tfrac{\x_{t-1} + \alpha_t \y_{t-1}}{1 + \alpha_t}}^2 + \tfrac{\alpha_t}{1 + \alpha_t} \norm{\x_{t-1} - \y_{t-1}}^2 \, .
\label{eqn:z8j3n}
\end{equation}
To simplify notation, we define $\dtm = \norm{\xtm - \ytm}$.
Applying the assumption that $\alpha_t > 0$, we can write $\x_t = \tfrac{\y_t - (1 - \alpha_t) \y_{t-1}}{\alpha_t}$.
Substituting this equality into \eqnref{eqn:z8j3n}, we have
\begin{align}
\alpha_t \norm{\x_t - \y_{t-1}}^2 + \norm{\x_t - \x_{t-1}}^2 &= (1 + \alpha_t) \norm{\tfrac{\y_t - (1 - \alpha_t) \y_{t-1}}{\alpha_t} - \tfrac{\x_{t-1} + \alpha_t \y_{t-1}}{1 + \alpha_t}}^2 + \tfrac{\alpha_t}{1 + \alpha_t} \dtm^2 \nonumber \\
&= \tfrac{1 + \alpha_t}{\alpha_t^2} \norm{\y_t - \tfrac{\alpha_t \x_{t-1} + \y_{t-1}}{1 + \alpha_t}}^2 + \tfrac{\alpha_t}{1 + \alpha_t} \dtm^2 \, . \label{eqn:x8z0s8f}
\end{align}
Inserting \eqnref{eqn:x8z0s8f} into \eqnref{eqn:m18z}, we see that
\[
\Delta_t \leq (1 - \alpha_t) \left[ \Delta_{t-1} - \tfrac{1 + \alpha_t}{\alpha_t^2} \oh \norm{\y_t - \tfrac{\alpha_t \x_{t-1} + \y_{t-1}}{1 + \alpha_t}}^2 - \tfrac{\alpha_t}{1 + \alpha_t} \oh \dtm^2\right] \, .
\]
Using the definition $\betat = \alpha_t (1 + \alpha_t)^{-1}$, we can plug in $\alpha_t = \betat (1 - \betat)^{-1}$ to complete the proof.
\end{proof}

\section{Proof of \lemref{lem:gap_progress_betat} \label{app:proof_blitz_lasso_lemma2}}


\begin{proof}
If $\beta_t = \oh$, we have $\Delta_t = 0$ by \lemref{lem:blitz_lasso_lemma}.
The bound holds in this case because $\Delta_{t-1}(1 - \xi_t) \geq 0$.
For the remainder of the proof, we assume that $\betat < \oh$, which implies that $\alpha_t < 1$.

Since $\alpha_t < 1$, there exists a constraint $\il \notin \W_t$ for which 
$\ip{\a_{\il}, \y_t} = b_i$.
Since $\il \notin \W_t$, we must have $\B_\xi(\betat) \cap \{ \x \, : \, \ip{\a_{\il}, \x} \geq b_i \} = \emptyset$.  Thus, $\y_t \notin \B_\xi(\betat)$.
Since $\B_\xi(\betat)$ is a ball with center $\betat \xtm + (1 - \betat) \ytm$ and radius $\tau_\xi(\betat)$, this implies that
\[
\norm{\y_t - \betat \x_{t-1} - (1 - \betat) \y_{t-1}} \geq \tau_\xi(\betat) \, .
\]
To simplify notation, we define $\dtm = \norm{\xtm - \ytm}$.
Combining with \lemref{lem:blitz_lasso_lemma} and plugging in the definition of $\tau_\xi(\betat)$, we have
\begin{align} \label{eqn:xn1nsadf8}
\Delta_t& \leq \tfrac{1 - 2\betat}{1 - \betat} \left[ \Delta_{t-1} - \tfrac{1 - \betat}{\betat^2} \oh \tau_\xi(\betat)^2 - \betat \oh \dtm^2 \right] \\
& = 
 \tfrac{1 - 2\betat}{1 - \betat} \left[ \Delta_{t-1} - (1 - \betat) \Delta_{t-1} \left[ 1 + \tfrac{\betat}{1 - \betat} \left(1 - \tfrac{\dtm^2}{2 \Delta_{t-1}} \right) - \tfrac{1 - \xi_t}{1 - 2 \betat} \right]_+  - \betat \oh \dtm^2 \right] \nonumber \\
 &\leq \tfrac{1 - 2 \betat}{1 - \betat} \left[ \Delta_{t-1} \tfrac{(1 - \xi_t)(1 - \betat)}{1 - 2 \betat} \right]
 \nonumber \\
 &= (1 - \xi_t) \Delta_{t-1} \nonumber \, .
\end{align}
\end{proof}

\section{Proof of \thmref{thm:blitz_lasso_theorem}} \label{app:proof_blitz_lasso_theorem}

The proof can be divided into three cases: $\betat=\toh$, $\betat \in (0, \toh)$, and $\betat = 0$.
Here we present the proof of \thmref{thm:blitz_lasso_theorem}
for only the main case that $\betat \in (0, \toh)$,
and we rely on the proof in \fullappref{app:proof_approximate} (a more general proof) for the edge cases.

\begin{proof}[Partial proof]

Assuming $\betat < \toh$ implies $\alpha_t < 1$.
This implies there exists a $\il \notin \W_t$
such that
$\ip{\a_{\il}, \y_t} = b_i$.
Since $\il \notin \W_t$, we must also have $\T_\xi \cap \{ \x \, : \, \ip{\a_{\il}, \x} \geq b_i \} = \emptyset$, which implies $\y_t \notin \T_\xi$.
Since $\y_t \notin \T_\xi$, then for all $\beta \in (0, 1/2)$, we have $\y_t \notin \B_\xi(\beta)$.  Applying the definition of $\B_\xi(\beta)$, we have
$
\norm{\y_t - \beta \x_{t-1} - (1 - \betat) \y_{t-1}} \geq \tau_\xi(\beta)
$
for all $\beta \in (0, 1/2)$.
Thus,
\begin{equation}  \label{eqn:xckvjsdf8}
\norm{\yt - \betat \xtm - (1 - \betat) \ytm} \geq \tau_\xi(\betat) \, .
\end{equation}
At this point, we can combine \eqnref{eqn:xckvjsdf8} with \lemref{lem:blitz_lasso_lemma} to achieve the desired bound.
The result follows from the same steps as \lemref{lem:gap_progress_betat}'s proof, starting at \eqnref{eqn:xn1nsadf8}.

\end{proof}

\section{Proof of \thmref{thm:quasiconcave} \label{app:proof_quasiconcave}}


\begin{proof}
Plugging in definitions of $q_s$ and $\tau_\beta$, we have
\[
q_s(\beta) = s \beta \norm{\xtm - \ytm} + \beta \sqrt{2 \Delta_{t-1}} \left[ 1 + \tfrac{\beta}{1 - \beta} \left(1 - \tfrac{\norm{\xtm - \ytm}^2}{2 \Delta_{t-1}} \right) - \tfrac{1 - \xi_t}{1 - 2 \beta} \right]_+^{1/2} \, .
\]
To simplify notation, define $\dtm = \norm{\xtm - \ytm}$.
Taking the log and substituting $\beta = \theta(1 + \theta)^{-1}$, we have
\[
\tilde{q}_s(\theta) = \log\left( \tfrac{\theta}{1 + \theta} \right) + \log \left( s \dtm + \sqrt{2 \Delta_{t-1}} \left[ 1 + \theta \left(1 - \tfrac{\dtm^2}{2 \Delta_{t-1}} \right) - (1 - \xi) \tfrac{1 + \theta}{1 - \theta} \right]_+^{1/2} \right) \, .
\]

Since $\tfrac{\theta}{1 + \theta}$ is concave, $\theta$ is concave, $-\tfrac{1+\theta}{1 - \theta}$ is concave, $[\cdot]^{1/2}$ is concave and nondecreasing, and $\log(\cdot)$ is concave and nondecreasing, we see that $\tilde{q}_s(\theta)$ is log-concave on $\{ \theta \, : \, \tilde{q}_s(\theta) > 0 \}$. (Here we have also used the facts that $1 - \tfrac{\dtm^2}{2 \Delta_{t-1}} \geq 0$ and $(1 - \xi) \geq 0$.)
Since all log-concave functions are quasiconcave and quasiconcavity is preserved under composition with the increasing function $\theta = \beta (1 - \beta)^{-1}$ (on the domain $0 < \beta \leq \toh$), it must be the case that $q_s(\beta)$ is quasiconcave.

\end{proof}

\section{Proof of \thmref{thm:blitz_lasso_theorem2} \label{app:proof_convergence_bound}}


\begin{proof}
In \secref{sec:proof_capsule_subset}, we prove that $\T_\xi \subseteq \Tcap$.
Since $\T_\xi \subseteq \Tcap$, we have 
\[
\T_\xi \cap \{ \x \, : \, \abs{\ip{\A_i, \x}} \geq \lambda \}  \ne \emptyset \ \Rightarrow  \
\Tcap \cap \{ \x \, : \, \abs{\ip{\A_i, \x}} \geq \lambda \} \ne \emptyset  \, .
\]

Thus, during iteration $t$ of \algref{alg:blitz_lasso}, condition (i) for \thmref{thm:blitz_lasso_theorem} is satisfied.
That is, for any $i \in [m]$, if $\T_\xi \cap \{ \x \, : \, \abs{\ip{\A_i, \x}} \geq \lambda \}  \ne \emptyset$, then $i \in \W_t$.
Since condition (ii) of the theorem is satisfied by our definition of \algref{alg:blitz_lasso}, we have by \thmref{thm:blitz_lasso_theorem}, $\Delta_t \leq (1 - \xi_t) \Delta_{t-1}$ for all $t \geq 1$.
The theorem then follows from induction.
\end{proof}

\section{Proof of \thmref{thm:convergence_rate} and \thmref{thm:approx_convergence_rate} \label{app:proof_approximate}}

Since \thmref{thm:convergence_rate} is a special case of \thmref{thm:approx_convergence_rate}, we prove both theorems by proving \thmref{thm:approx_convergence_rate}.
To recover \thmref{thm:convergence_rate}, we define $\epsilon_t = 0$, $\ftlb = f_t$, and $\x_t = \z_t = \mathrm{argmin}_\x\ f_t(\x)$.


\begin{proof}
We will prove that for all $t > 0$, we have
\begin{equation}
\Delta_t \leq (1 - (1 - \epsilon_t) \xi_t) \Delta_{t-1} \, . \label{eqn:itr_bound}
\end{equation}


To prove \eqnref{eqn:itr_bound} for any $t > 0$, let us define the scalar
\[
\thetat  = \mathrm{max}\, \{ \theta \in [0, 1] \, : \, \theta \z_t + (1 - \theta) \ytm \in \cl{\Tcap} \} 
\]
and point
$\ytp = \thetat \zt + (1 - \thetat) \ytm$.
Above, $\cl{\cdot}$ denotes the closure of a set.  Note $\ytm \in \cl{\Tcap}$.


Since $\yt$ minimizes $f$ along $[\ytm, \zt]$, it follows that $f(\yt) \leq f(\ytp)$.
Due to (C1), we have that $f_t(\x) = f(\x)$ for all $\x \in \Tcap$.  Since $\ytp \in \cl{\Tcap}$ and $f$ is convex lower semicontinuous, it follows that $f_t(\ytp) = f(\ytp)$.
Beginning with the definition of $\Delta_t$, we can write
\[
\Delta_t = f(\yt) - \ftlb(\xt) 
\leq f(\ytp) - \ftlb(\xt) \\
= f_t(\ytp) - \ftlb(\xt) \, . 
\]

We divide the remainder of the proof into three cases.

\paragraph{Case 1: $\thetat = 1$}

In this case, $\ytp = \zt$, and it follows that
\[
\Delta_t \leq f_t(\ytp) - \ftlb(\xt) 
= f_t(\zt) - \ftlb(\xt) 
\leq \epsilon_t \Delta_{t-1} 
\leq (1 - (1 - \epsilon_t) \xi_t) \Delta_{t-1}  \, .
\]
Above, the second-to-last step results from termination conditions for subproblem $t$, while the final step is true because $\xi_t \in (0, 1]$.

\paragraph{Case 2: $\thetat \in (0, 1)$}




Applying the definition of $\ytp$, the fact that $f_t$ is 1-strongly convex, the fact that $f_t(\x) \leq f(\x)$ for all $\x$, and the definition of $\Delta_{t-1}$, we have
\begin{align}
\Delta_t & \leq f_t(\ytp) - \ftlb(\xt) \nonumber \\
&= f_t(\thetat \zt + (1 - \thetat) \ytm) - \ftlb(\xt) \nonumber \\
&\leq \thetat f_t(\zt) + (1 - \thetat) f_t(\ytm) - \oh (1 - \thetat) \thetat \norm{\zt - \ytm}^2  - \ftlb(\xt) \nonumber \\
&\leq \thetat f_t(\zt) + (1 - \thetat) f(\ytm) - \oh (1 - \thetat) \thetat \norm{\zt - \ytm}^2 - \ftlb(\xt) \nonumber \\
&= (1 - \thetat) \Delta_{t-1} - (1 - \thetat) [\ftlb(\xt) - \ftmlb(\xtm) ] + \nonumber \nonumber \\ 
& \hskip 1.455in {\thetat} [ f_t(\zt) - \ftlb(\xt) ]- \oh (1 - \thetat) \thetat \norm{\zt - \ytm}^2  \nonumber \, .
\end{align}

From termination conditions for subproblem $t$, we have that $f_t(\zt) - \ftlb(\xt) \leq \epsilon_t \Delta_{t-1}$ and also that $\ftlb(\xt) - \ftmlb(\xtm) \geq (1 - \epsilon_t) \oh \norm{\zt - \xtm}^2$.  Thus,
\begin{align}
\Delta_t &\leq (1 - \thetat) \Delta_{t-1} - (1 - \thetat) (1 - \epsilon_t) \oh \norm{\zt - \xtm}^2 + \thetat \epsilon_t \Delta_{t-1} - \oh (1 - \thetat) \thetat \norm{\zt - \ytm}^2  \nonumber \\
&\leq \Delta_{t-1} - (1 - \epsilon_t) \left[ \thetat \Delta_{t-1} + \oh (1 - \thetat) \left( \thetat \norm{\zt - \ytm}^2 + \norm{\zt - \xtm}^2 \right) \right] \, . \label{eqn:bbz}
\end{align}
We next use the fact
\begin{equation}
\thetat \norm{\zt - \ytm}^2 +
\norm{\zt - \xtm}^2  
= (1 + \thetat) \norm{\zt - \tfrac{\xtm + \thetat \ytm}{1 + \thetat}}^2 + \tfrac{\thetat}{1 + \thetat} \norm{\xtm - \ytm}^2 \, . \label{eqn:c090}
\end{equation}
To simplify notation slightly, we define $\dtm = \norm{\xtm - \ytm}$.
Applying the assumption that $\thetat  > 0$, we can write $\zt = \tfrac{\ytp - (1 - \thetat) \ytm}{\thetat}$.  Substituting this equality into \eqnref{eqn:c090}, we have
\begin{align}
\thetat \norm{\zt - \ytm}^2 +
\norm{\zt - \xtm}^2  &= (1 + \thetat) \norm{\tfrac{\ytp - (1 - \thetat) \ytm}{\thetat} - \tfrac{\xtm + \thetat \ytm}{1 + \thetat}}^2 + \tfrac{\thetat}{1 + \thetat} \dtm^2 \nonumber \\
 &= \tfrac{1 + \thetat}{\thetat^2} \norm{\ytp - \tfrac{\thetat \xtm + \ytm}{1 + \thetat}}^2 + \tfrac{\thetat}{1 + \thetat} \dtm^2 \, . \label{eqn:hxjz}
\end{align}
Inserting \eqnref{eqn:hxjz} into \eqnref{eqn:bbz}, it follows that
\begin{equation}
\Delta_t \leq \Delta_{t-1} - (1 - \epsilon_t) \left[ \thetat \Delta_{t-1} + \tfrac{1 - \thetat^2}{\thetat^2} \oh \norm{\ytp - \tfrac{\thetat \xtm + \ytm}{1 + \thetat}}^2 + \tfrac{\thetat (1 - \thetat)}{1 + \thetat} \oh \dtm^2 \right] \, . \label{eqn:papap}
\end{equation}

Let us denote the quantity within the brackets above by $P$ (for ``progress'' toward convergence).
 Also, let us define $\betat = \thetat (1 + \thetat)^{-1}$, which implies $\thetat = \betat (1 - \betat)^{-1}$.  We see that
\begin{align}
  P &= \thetat \Delta_{t-1} + \tfrac{1 - \thetat^2}{ \thetat^2} \oh \norm{\ytp - \tfrac{\thetat \xtm + \ytm}{1 + \thetat}}^2 + \tfrac{\thetat (1 - \thetat)}{1 + \thetat} \oh \dtm^2  \nonumber \\
&=  \tfrac{\betat}{1 - \betat} \Delta_{t-1} +  \tfrac{1 - 2 \betat}{\betat^2} \oh \norm{\ytp - \betat \xtm - (1 - \betat) \ytm}^2 + \tfrac{\betat (1 - 2 \betat)}{1- \betat} \oh \dtm^2 \nonumber \\
&=  \tfrac{1 - 2 \betat}{1 - \betat} \left[ \tfrac{\betat}{1 - 2 \betat} \Delta_{t-1} +  \tfrac{1 -  \betat}{\betat^2} \oh \norm{\ytp - \betat \xtm - (1 - \betat) \ytm}^2 + \betat  \oh \dtm^2 \right] \, . \label{eqn:xuau}
\end{align}

Since $\thetat < 1$, by definition of $\thetat$ and $\ytp$, we must have $\ytp \in \bd{\Tcap}$.
Since $\Tcap$ is an open set and $\ytp \in \bd{\Tcap}$, we have $\ytp \notin \Tcap$.
Furthermore, since $\Tcap \supseteq \T_\xi \supseteq \B_\xi(\betat)$, it follows that $\ytp \notin \B_\xi(\betat)$.  By definition of $\B_\xi(\betat)$, we have
\[
\norm{\ytp - \betat \xtm - (1 - \betat) \ytm} \geq \tau_\xi(\betat) \, .
\]
Plugging in the definition of $\tau_\xi(\betat)$, it follows that
\begin{align}
\tfrac{1 - \betat}{\betat^2} \oh \norm{\ytp - \betat \xtm - (1 - \betat) \ytm}^2 &\geq (1 - \betat) \Delta_{t-1} \left[ 1 + \tfrac{\betat}{1 - \betat} \left( 1 - \tfrac{\dtm^2}{2 \Delta_{t-1}} \right) - \tfrac{1 - \xi_t}{1 - 2 \betat} \right]_+ \nonumber \\
&=  \left[ \Delta_{t-1} - \betat \oh \dtm^2 -  \tfrac{1 - \betat}{1 - 2 \betat} (1 - \xi_t) \Delta_{t-1} \right]_+ \nonumber \\
&=  \left[\tfrac{1 - \betat}{1 - 2 \betat} \xi_t \Delta_{t-1} - \tfrac{\betat}{1 - 2\betat} \Delta_{t-1} - 
\betat \oh \dtm^2 \right]_+ \nonumber
\end{align}

Plugging this result into \eqnref{eqn:xuau}, we have
\begin{align}
P &\geq \tfrac{1 - 2 \betat}{1 - \betat} \left[ \tfrac{\betat}{1 - 2 \betat} \Delta_{t-1} +
  \left[\tfrac{1 - \betat}{1 - 2 \betat} \xi_t \Delta_{t-1} - \tfrac{\betat}{1 - 2\betat} \Delta_{t-1} - 
\tfrac{\betat}2 \dtm^2 \right]_+
+  \tfrac{\betat}2 \dtm^2 \right] \nonumber \\
 &\geq \tfrac{1 - 2 \betat}{1 - \betat} \left[ \tfrac{1 - \betat}{1 - 2 \betat} \xi_t \Delta_{t-1} \right] \nonumber \\
 &= \xi_t \Delta_{t-1} \, . \label{eqn:almost}
\end{align}
By combining \eqnref{eqn:almost} with \eqnref{eqn:papap}, we obtain the desired bound.

\paragraph{Case 3: $\theta_t = 0$}
Using the definition of $\ytp$, we have
\begin{align}
\Delta_t &\leq f_t(\ytp) - \ftlb(\xt) \nonumber \\
        &= f_t(\ytm) - \ftlb(\xt)  \nonumber\\
        &\leq f(\ytm) - \ftlb(\xt) \nonumber \\
        &= \Delta_{t-1} - \left[ \ftlb(\xt) - \ftmlb(\xtm) \right]  \nonumber\\
        &\leq \Delta_{t-1} - (1 - \epsilon_t) \oh \norm{\zt - \xtm}^2 \, . \label{eqn:oof}
\end{align}
Above, the last step follows from termination conditions for subproblem $t$.

Since $\theta_t = 0$, it follows from the definition of $\thetat$ that $\beta_t \zt + (1 - \betat) \ytm \notin \cl{\Tcap}$ for all $\beta \in (0, 1/2)$.  Since $\Tcap \supseteq \T_\xi \supseteq \B_\xi(\beta)$ for all $\beta \in (0, 1/2)$, then
$\beta \zt + (1 - \beta) \ytm \notin \cl{B_\xi(\beta)}$ for all $\beta \in (0, 1/2)$.
By definition of $\B_\xi(\beta)$, this means that
\begin{align}
& \norm{\beta \zt + (1 - \beta) \ytm - \beta \xtm - (1 - \beta) \ytm} > \tau_\xi(\beta)  \nonumber \\
& \Rightarrow \quad  \norm{\zt - \xtm} > \frac{\tau_\xi(\beta)}{\beta} \, . \nonumber
\end{align}
This implies that
\begin{equation}
\norm{\zt - \xtm} \geq \lim{\beta \rightarrow 0^+} \frac{\tau_\xi(\beta)}{\beta} 
= \tfrac{d}{d \beta} \tau_\xi(\beta) \Bigr|_{\beta = 0} 
= \sqrt{2 \Delta_{t-1} \xi_t} \, . \label{eqn:wash}
\end{equation}
By combining \eqnref{eqn:wash} with \eqnref{eqn:oof}, we obtain the result.

\end{proof}

\section{Proof of \thmref{thm:screening} \label{app:proof_screening}}


\begin{proof}


We need to show that $\xstar = \mathrm{argmin}_\x\,f(\x) = \mathrm{argmin}_\x\,\fs(\x) = \xs$.
First note that since $\f_0$ is a $1$-strongly convex lower bound on $f$, and $\x_0$ minimizes $f_0$, it follows that 
\begin{equation} 
 f(\xstar)  \geq f_0(\x_0) + \oh \norm{\xstar - \x_0}^2 \, . \label{eqn:fabi}
\end{equation}
Since $f$ is 1-strongly convex, and $\xstar$ minimizes $f$, we have
\begin{equation} \label{eqn:naz}
f(\y_0) \geq f(\xstar) + \oh \norm{\y_0 - \xstar}^2 \, .
\end{equation}

Combining \eqnref{eqn:naz} with \eqnref{eqn:fabi}, we have
\begin{align}
f(\xstar) + f(\y_0) & \geq f_0(\x_0) + \oh \norm{\xstar - \x_0}^2 + f(\xstar) + \oh \norm{\y_0 - \xstar}^2 \nonumber \\
\Rightarrow \quad & \Delta_0  \geq \norm{\xstar - \oh (\x_0 + \y_0)}^2 + \tfrac{1}4 \norm{\x_0 - \y_0}^2 
\label{eqn:x8gj}
\\
\Rightarrow \quad & \xstar  \in \cl{\T_1} \, .
\nonumber 
\end{align}

By construction, $\fs(\x) = f(\x)$ for all $\x \in \T_1$.
Since $\T_1$ is an open set, if $\xstar \in \T_1$, then
\[
\partial \fs(\xstar)  = \partial f(\xstar) 
\ \Rightarrow \  \m{0}  \in \partial \fs(\xstar) 
\ \Rightarrow \  \xstar  = \xs \, .
\]

For the remainder of the proof, we consider the case that $\xstar \in \bd{\T_1}$.
In this case, \eqnref{eqn:naz} holds with equality (since \eqnref{eqn:x8gj} holds with equality), meaning
\begin{equation}
f(\y_0) = f(\xstar) + \oh \norm{\y_0 - \xstar}^2 \, . \label{eqn:blueb}
\end{equation}
Define $\z = \oh (\y_0 + \xstar)$.  Note $\z \in \T_1$, since $\xstar \in \bd{\T_1}$, $\y_0 \in \cl{\T_1}$, $\xstar \ne \y_0$, and $\T_1$ is an open ball.
Also, since $\z$ lies on the segment $[\xstar, \y_0]$, and $f$ is 1-strongly convex,  \eqnref{eqn:blueb} implies that
\[
f(\z) = f(\xstar) + \oh \norm{\z - \xstar}^2 \, . \label{eqn:bleed}
\]
This implies that $\z - \xstar \in \partial f(\z)$,
since $f(\xstar) + \oh \norm{\x - \xstar}^2 \leq f(\x)$ for all $\x$.
Because $\z \in {\T_1}$,  it follows that
$\z - \xstar \in \partial \fs(\z)$.
Since $\fs$ is 1-strongly convex, then for all $\x$,  we have
\begin{align}
\fs(\x) & \geq f(\z) + \ip{\z - \xstar, \x - \z} + \oh \norm{\x - \z}^2\nonumber  \\
& = f(\xstar) + \oh \norm{\z - \xstar}^2 + \ip{\z - \xstar, \x - \z} + \oh \norm{\x - \z}^2 \nonumber \\
& \geq f(\xstar) \, . \label{eqn:b7b7b}
\end{align}

At the same time, since $\fs(\x) = f(\x)$ for all $\x \in \T_1$, and $\fs$ is lower semicontinuous, we must have
$
\fs(\xstar) = f(\xstar)
$.
Combined with \eqnref{eqn:b7b7b}, it follows that $\xstar$ minimizes $\fs$.




\end{proof}

\section{Proof of \thmref{thm:S1} \label{app:proof_S1}}


\begin{proof}
Consider any $\x_{\T_1} \in \T_1$.
For some $\delta > 0$, we have
\[
 \norm{\x_{\T_1} - \oh (\x_0 + \y_0)} = \sqrt{\Delta_0 - \tfrac{1}4 \norm{\x_0 - \y_0}^2}  - \delta
 \]
From the definition of $\T_\xi$, we know $\T_\xi \supseteq \B_\xi(\beta)$ for all $\beta \in (0, \toh)$.
Recall that $\B_\xi(\beta)$ is a ball with center $\beta \x_0 + (1 - \beta) \y_0$ and radius $\tau_\xi(\beta)$.
We have
\begin{align}
\lim{\beta \rightarrow \toh^-} [\tau_\xi(\beta) - \norm{\x_{\T_1} - \beta \x_0 - (1- \beta)\y_0}] &= \sqrt{\Delta_0 - \tfrac{1}4 \norm{\x_0 - \y_0}^2} - \norm{\x_{\T_1} - \oh (\x_0 + \y_0)} \nonumber \\
&= \delta 
> 0 \,. \nonumber
\end{align}
Thus, for some $\beta \in (0,\toh)$, we have $\x_{\T_1}\in \B_\xi(\beta)$, implying $\x_{\T_1} \in \T_\xi$.
We have shown $\T_1 \subseteq \T_\xi$.

To show that $\T_\xi \subseteq \T_1$, consider any $\x_{\T_\xi} \in \T_\xi$.  
Since $\xi_t = 1$, for all $\beta \in (0, \toh)$ we have
\[
\tau_\xi(\beta) = \beta \sqrt{ 2 \Delta_0 \left[ 1 + \tfrac{\beta}{1 - \beta} \left(1 - \tfrac{ \norm{\x_0 - \y_0}^2}{2 \Delta_0} \right)\right] } 
= 2 \beta \sqrt{ \Delta_0 - \tfrac{1}4 \norm{\x_0 - \y_0}^2} \, .
\]
Above we used the fact $\tfrac{\norm{\x_0 - \y_0}^2}{2 \Delta_0} \leq 1$, which follows from $f(\y_0) \geq f(\x_0) + \oh \norm{\y_0 - \x_0}^2$. 

From the definition of $\T_\xi$, there exists a $\beta \in (0, 1/2)$ such that \mbox{$\x_{\T_\xi} \in \B_\xi(\beta)$}.
From this, we see
\begin{align}
 &
\norm{\x_{\T_\xi} - \beta \x_0 - (1 - \beta) \y_0} < \tau_\xi(\beta) \nonumber \\
\Rightarrow &
\norm{\x_{\T_\xi} - \beta \x_0 - (1 - \beta) \y_0} < 2 \beta \sqrt{ \Delta_0 - \tfrac{1}4 \norm{\x_0 - \y_0}^2} \nonumber \\
\Rightarrow &
\norm{\x_{\T_\xi} - \oh(\x_0 + \y_0) + (\oh - \beta) (\x_0 - \y_0)} < 2 \beta \sqrt{ \Delta_0 - \tfrac{1}4 \norm{\x_0 - \y_0}^2}\nonumber  \\
\Rightarrow &
\norm{\x_{\T_\xi} - \oh(\x_0 + \y_0)} < (1 - 2 \beta) \oh \norm{\x_0 - \y_0} + 2 \beta \sqrt{ \Delta_0 - \tfrac{1}4 \norm{\x_0 - \y_0}^2} \nonumber \\
\Rightarrow &
\norm{\x_{\T_\xi} - \oh(\x_0 + \y_0)} < \sqrt{ \Delta_0 - \tfrac{1}4 \norm{\x_0 - \y_0}^2} \nonumber \\
\Rightarrow & \
\x_{\T_\xi} \in \T_1 \, . \nonumber
\end{align}
Note we again used the fact $\tfrac{ \norm{\x_0 - \y_0}^2}{2 \Delta_0} \leq 1$, which implies $\oh \norm{\x_0 - \y_0} \leq \sqrt{ \Delta_0 - \tfrac{1}4 \norm{\x_0 - \y_0}^2}$.
\end{proof}

\section{Miscellaneous proofs \label{app:misc_proofs}}

\subsection{Proof that teardrop equivalence region is a subset of capsule equivalence region \label{sec:proof_capsule_subset}}

\begin{thm}
Define $\T_\xi$ and $\Tcap$ as in \secref{sec:teardrop} and \secref{sec:capsule}.  Then $\T_\xi \subseteq \Tcap$.
\end{thm}
\begin{proof}
Recall $\Tcap$ is the set of points within a distance $\rcap$
from the segment $[\ccap_1, \ccap_2]$, where
\[
\ccap_1 = \beta_1 \xtm + (1 - \beta_1) \ytm \,, \quad \quad
\ccap_2 = \beta_2 \xtm + (1 - \beta_2) \ytm \,,
\]
\vspace{-1.2em}
\[
\beta_1 = \tfrac{\dminc + \rcap}{\norm{\xtm - \ytm}} \,, \quad \quad
\beta_2 = \tfrac{\dmaxc - \rcap}{\norm{\xtm - \ytm}} \,. 
\]

Consider any $\x' \in \T_\xi$.  By definition of $\T_\xi$, there exists a scalar $\beta' \in (0, \toh)$ such that
\[
\norm{\x' - (\beta' \xtm + (1 - \beta') \ytm)} < \tau_\xi(\beta') \, .
\]

In the case that $\beta_1 \leq \beta' \leq \beta_2$, then $\beta' \xtm + (1 - \beta') \ytm$ falls on the segment $[\ccap_1, \ccap_2]$.  This implies that $\x' \in \T_\xi^{\mathrm{Cap}}(\T_\xi)$, since
\[
\norm{\x' - (\beta' \xtm + (1 - \beta') \ytm)} < \tau_\xi(\beta') 
\leq  \rcap \, .
\]

In the case that $\beta' \leq \beta_1$, we have
\begin{align}
\norm{\x' - \ccap_1} &\leq \norm{\x' - (\beta' \xtm + (1 - \beta') \ytm)} + \norm{\beta' \xtm + (1 - \beta') \ytm - \ccap_1} \nonumber \\
&< \tau_\xi(\beta') +  (\beta_1 - \beta') \norm{\xtm - \ytm} \nonumber \\
&= \left[ \tau_\xi(\beta') - \beta' \norm{\xtm - \ytm} \right] + \beta_1 \norm{\xtm - \ytm}  \nonumber \\
&\leq -\dminc + \left[ \dminc + \rcap \right] \nonumber \\
&= \rcap \, . \nonumber
\end{align}
Thus, $\x' \in \Tcap$ if $\beta' \leq \beta_1$.
A similar argument implies $\x' \in \Tcap$ when $\beta' \geq \beta_2$.

Thus, for all $\beta'$, we have $\x' \in \T_\xi$, which implies $\T_\xi \subseteq \Tcap$.
\end{proof}

\subsection{Proof that dual progress termination condition can be satisfied}
\label{app:condition_2_always}

\begin{thm}
For the \algname{} algorithm with approximate subproblem solutions described in \secref{sec:approximate}, 
if subproblem $t$ is solved exactly, then it is always the case that
  \[
  \ftlb(\xt) - \ftmlb(\xtm) \geq (1 - \epsilon_t) \oh \norm{\zt - \xtm}^2 \, .
  \]
\end{thm}
\begin{proof}
If subproblem $t$ is solved exactly, then $f_t(\zt) = \ftlb(\xt)$, since $\xt = \zt$.
Due to condition (C3) in \secref{sec:approximate}, we have $f_t(\x) \geq \ftmlb(\x)$ for all $\x$.
Thus,
\begin{align}
& f_t(\x)  \geq \ftmlb(\xtm) + \oh \norm{\x - \xtm}^2 \nonumber \\
 \Rightarrow \quad & f_t(\zt)  \geq \ftmlb(\xtm) + \oh \norm{\zt - \xtm}^2 \nonumber \nonumber \\
 \Rightarrow \quad  &\ftlb(\xt) - \ftmlb(\xtm)  \geq \oh \norm{\zt - \xtm}^2 \nonumber \nonumber \\
 \Rightarrow \quad & \ftlb(\xt) - \ftmlb(\xtm)  \geq (1 - \epsilon_t) \oh \norm{\zt - \xtm}^2 \,  . \nonumber
\end{align}
\end{proof}

\subsection{Proof that $f_0$ lower bounds $\flod$ in \secref{sec:screen_l1} \label{app:screen_l1}}

\begin{thm}
For any $\omegab_0 \in \reals^m$, define $f_0$, $\flod$, and $\x_0$ as in \secref{sec:screen_l1}.  
  Then 
$
f_0(\x)  \leq 
 \flod(\x)
$
$\forall \x$.
\end{thm}
\begin{proof}
Let $[\x_0]_j$ denote the $j$the entry of $\x_0$.
For all $x_j$, the Fenchel-Young inequality implies 
\[
L_j^*(x_j) - x_j \ip{\a_j, \omegab_0} \geq -L_j(\ip{\a_j, \omegab_0}) \, .
\]
When $x_j = L_j'(\ip{\a_j, \omegab_0})$, this inequality holds with equality, implying $L_j^*(x_j) - x_j \ip{\a_j, \omegab_0}$ is minimized when $x_j = [\x_0]_j$.
By assuming that $L_j$  is 1-smooth, $L_j^*$ is 1-strongly convex. 
Thus,
\begin{align}
\textstyle \sum_{j=1}^n \left[ L_j^*(x_j) - x_j \ip{\a_j, \omegab_0} \right] &\geq \oh \norm{\x - \x_0}^2 + \textstyle\sum_{j=1}^n \left[ L_j^*([\x_0]_j) - [\x_0]_j \ip{\a_j, \omegab_0} \right] \nonumber \\
&= \oh \norm{\x - \x_0}^2 - \textstyle\sum_{j=1}^n  L_j(\ip{\a_j, \omegab_0}) \, \nonumber .
\end{align}
Applying this result, we have
\begin{align}
f_0(\x) \leq \flod(\x) \ \Leftrightarrow\ &  
\oh \norm{\x - \x_0}^2 - \flo(\omegab_0) 
\leq \textstyle\sum_{j=1}^n L_j^*(x_j) + \textstyle\sum_{i=1}^m \phi_i(\x) \nonumber  
\\
\Leftrightarrow\ &  -\textstyle\sum_{j=1}^n x_j \ip{\a_j, \omegab_0} - \lambda \norm{\omegab_0}_1 \leq \textstyle\sum_{i=1}^m \phi_i(\x) \nonumber \\
\Leftrightarrow\ & - \ip{\A \omegab_0, \x} - \lambda \norm{\omegab_0}_1 \leq \textstyle\sum_{i=1}^m \phi_i(\x) \, . \label{eqn:0zm10jj}
\end{align}

Thus, it remains to prove \eqnref{eqn:0zm10jj}.
For each $i$, note $\phi_i(\x) = +\infty$ if $\abs{\ip{\A_i, \x}} > \lambda$.
Thus, we must only consider the case $\abs{\ip{\A_i, \x}} \leq \lambda$, which implies $\phi_i(\x) =0$.
Assuming $\abs{\ip{\A_i, \x}} \leq \lambda$, we have
\begin{align}
-[\omegab_0]_i \ip{\A_i, \x} - \lambda \abs{[\omegab_0]_i} &\leq \abs{[\omegab_0]_i} \abs{\ip{\A_i, \x}} - \lambda \abs{[\omegab_0]_i} \nonumber  \\
&= \abs{[\omegab_0]_i} \left( \abs{\ip{\A_i, \x}} - \lambda \right) \nonumber \\
&\leq 0 \nonumber \\
&= \phi_i(\x) \, . \nonumber
\end{align}
Summing over $i \in [m]$ proves \eqnref{eqn:0zm10jj}.
\end{proof}

\end{document}